\documentclass[11pt]{article}

\usepackage[utf8]{inputenc} 
\usepackage[T1]{fontenc}    
\usepackage{hyperref}       
\usepackage{url}            
\usepackage{amsfonts}       
\usepackage{microtype}      

\usepackage{subcaption}
\usepackage{algorithm}
\usepackage{algorithmic}

\usepackage{amsmath,amsfonts,bm,amssymb}
\usepackage{amsthm}
\usepackage{mathtools}
\usepackage{xcolor}
\usepackage{nicefrac}       

\usepackage[noabbrev,capitalize,nameinlink]{cleveref}

\newcommand{\tr}{\textrm{tr}}
\newcommand{\half}{\frac{1}{2}}

\allowdisplaybreaks

\newtheorem{theorem}{Theorem}[section]

\newtheorem{corollary}{Corollary}[section]

\newtheorem{lemma}[theorem]{Lemma}
{}

\newtheorem{definition}{Definition}[section]

\newtheorem{example}{Example}

\makeatletter
\newtheorem*{rep@theorem}{\rep@title}
\newcommand{\newreptheorem}[2]{%
\newenvironment{rep#1}[1]{%
 \def\rep@title{#2 \ref{##1}}%
 \begin{rep@theorem}}%
 {\end{rep@theorem}}}
\makeatother

\newreptheorem{theorem}{Theorem}
\newreptheorem{lemma}{Lemma}
\newreptheorem{prop}{Proposition}
\newreptheorem{corollary}{Corollary}
\newreptheorem{definition}{Definition}

\newcommand{\ex}{\mathop{\mathbb{E}}\limits}

\newenvironment{itemize*}%
{\begin{itemize}[leftmargin=*,topsep=0pt]%
		\setlength{\itemsep}{0pt}%
		\setlength{\parskip}{0pt}}%
	{\end{itemize}}
\newenvironment{enumerate*}%
{\begin{enumerate}[leftmargin=*,topsep=0pt]%
		\setlength{\itemsep}{0pt}%
		\setlength{\parskip}{0pt}}%
	{\end{enumerate}}

\def\eqref#1{equation~\ref{#1}}

\def\1{\bm{1}}

\DeclareMathAlphabet{\mathsfit}{\encodingdefault}{\sfdefault}{m}{sl}
\SetMathAlphabet{\mathsfit}{bold}{\encodingdefault}{\sfdefault}{bx}{n}

\def\gA{{\mathcal{A}}}

\def\gD{{\mathcal{D}}}

\def\gF{{\mathcal{F}}}

\def\gO{{\mathcal{O}}}

\def\gT{{\mathcal{T}}}
\def\gU{{\mathcal{U}}}

\def\gX{{\mathcal{X}}}

\newcommand{\E}{\mathbb{E}}

\newcommand{\R}{\mathbb{R}}

\DeclareMathOperator*{\argmax}{arg\,max}

\DeclareMathOperator{\sign}{sign}

\usepackage{tabularx}
\usepackage{booktabs}
\usepackage{multirow}
\usepackage{enumitem}
\usepackage{bbm}
\usepackage{enumitem}
\usepackage{manfnt}
\usepackage{comment}

\newcommand{\RR}{\mathbb{R}}

\newcommand{\ns}[1]{{\color{orange}NS: #1}}

\newcommand{\dipendra}[1]{\textcolor{blue}{[DM: #1]}}

\newcommand{\highlight}[1]{{\color{teal}#1}}

\newcommand{\Xorig}{\bar{\gX}}
\newcommand{\Xaug}{{\gX}}

\newcommand{\norm}{\circ}

\newcommand{\BE}{\text{Bayes-error}}

\newcommand{\A}{{A}}
\newcommand{\An}{{\A_{\norm}}}
\newcommand{\Ln}{{L_{\norm}}}
\newcommand{\Abar}{\bar{\A}}
\newcommand{\Abarn}{{\bar{\A}_{\norm}}}
\newcommand{\gstar}{{\bar{y}^{\star}}}
\newcommand{\gstaraug}{{y^{\star}}}
\newcommand{\ystar}{{\bar{y}^{\star}}}

\newcommand{\wstar}{{w^{\star}}}
\newcommand{\gstarn}{\gstar_{\norm}}
\newcommand{\fn}{f_{\norm}}
\newcommand{\Fn}{F_{\norm}}
\newcommand{\gn}{g_{\norm}}

\newcommand{\Phin}{\Phi_{\norm}}
\newcommand{\gDbar}{\bar{\gD}}
\newcommand{\Dbar}{\bar{D}}

\newcommand{\phin}{{\phi_{\norm}}}

\newcommand{\Dsim}{\gD_{\text{sim}}}
\newcommand{\Dneg}{\gD_{\text{neg}}}
\newcommand{\Daug}{\gD_{\Xaug}}
\newcommand{\Dorig}{\gD_{\Xorig}}
\newcommand{\dd}{d}
\newcommand{\DD}{D}

\newcommand{\diag}{\textrm{diag}}

\newcommand{\aug}[1]{#1_{\gA}}

\newcommand{\Lclf}{L_{\textrm{clf}}}

\newcommand{\Lreg}{L_{\text{reg}}}

\newcommand{\Lspec}{L_{\text{spec}}}
\newcommand{\Lcont}{L_{\text{cont}}}
\newcommand{\Lsimclr}{L_{\textrm{SimCLR}}}
\newcommand{\Lcons}{\Delta_{\gA}}

\newcommand{\vit}{{\tt ViT}}
\newcommand{\resnet}{{\tt ResNet}}
\newcommand{\resnete}{{\tt ResNet-18}}
\newcommand{\mlpmixer}{{\tt MLP-Mixer}}
\newcommand{\bow}{{\tt BoW}}
\newcommand{\gru}{{\tt GRU}}
\newcommand{\transformers}{{\tt Transformer}}

\usepackage{hyperref}
\hypersetup{colorlinks=true,citecolor=blue,linkcolor=blue}
\usepackage[parfill]{parskip}
\usepackage{natbib}
\usepackage[margin=0.9in]{geometry}
\usepackage{authblk}

\newcommand\blfootnote[1]{%
  \begingroup
  \renewcommand\thefootnote{}\footnote{#1}%
  \addtocounter{footnote}{-1}%
  \endgroup
}

\newboolean{arxiv}
\setboolean{arxiv}{true}

\begin{document}

\title{Understanding Contrastive Learning Requires \\ Incorporating Inductive Biases}
\date{}
\author{Nikunj Saunshi$^{1*}$ \qquad Jordan T. Ash$^2$ \qquad Surbhi Goel$^2$ \qquad Dipendra Misra$^2$ \qquad Cyril Zhang$^2$ \\ \vspace{-3mm}
Sanjeev Arora$^1$ \qquad Sham Kakade$^{2\,3}$ \qquad Akshay Krishnamurthy$^2$ \\ \vspace{5mm}
  \normalsize{$^1$Department of Computer Science, Princeton University\\$^2$Microsoft Research, New York City\\ $^3$Departments of  Computer Science \& Statistics, Harvard University}
}

\maketitle

\begin{abstract}

Contrastive learning is a popular form of self-supervised learning that encourages augmentations (views) of the same input to have more similar representations compared to augmentations of different inputs. Recent attempts to theoretically explain the success of contrastive learning on downstream classification tasks prove guarantees depending on properties of {\em augmentations} and the value of {\em contrastive loss} of representations. We demonstrate that such analyses, that ignore {\em inductive biases} of the function class and training algorithm, cannot adequately explain the success of contrastive learning, even {\em provably} leading to vacuous guarantees in some settings. Extensive experiments on image and text domains highlight the ubiquity of this problem -- different function classes and algorithms behave very differently on downstream tasks, despite having the same augmentations and contrastive losses. Theoretical analysis is presented for the class of linear representations, where incorporating inductive biases of the function class allows contrastive learning to work with less stringent conditions compared to prior analyses.

\end{abstract}

\blfootnote{\!\!\!$^*$Corresponding author <\texttt{nsaunshi@cs.princeton.edu}>. Work started as an intern at Microsoft Research NYC.}


\section{Introduction}


Recently, representation functions learned via contrastive learning have transformed machine learning.
Using unlabeled data, a representation function is learnt by generating simple augmentations of each datapoint and by enforcing, via a suitable loss function, that (1) augmentations of a single datapoint tend to be clustered (2) augmentations of different datapoints tend to be far apart.  Such representations give competitive classification performance ---via even a linear classifier--- on a host of downstream tasks, bringing us closer to the old dream of machine learners capable of generalization across different data distributions and tasks.

We lack a conceptual framework for understanding such wondrous phenomena --- which is unsurprising, since good quantitative understanding of generalization is lacking even for single task and single data distribution. However, deriving even partial conceptual understanding could help push the field forward, and researchers have begun to grapple with this task \citep{arora2019theoretical,tosh2020contrastive,haochen2021provable}. The current paper seeks to provide guidance for further development of this nascent theory\footnote{Our title is a clear allusion to  \citet{zhang2017understanding}, which highlighted a gap between deep learning phenomena and classical ML theory, motivating development of better theoretical understanding.} using simple experiments and  theoretical analysis. 
A common thread in these existing theories is the following components: (1) Quantifying how data augmentations {\em implicitly} encode downstream class labels. (2) Demonstrating how representations with small contrastive loss can uncover this implicit structure and do well on downstream tasks.

Recent works formalize (1) via assumptions that end up implying that the augmentation distributions of inputs from the same class  have significant overlap, but there is little overlap for inputs from different classes.
For example, distributions of augmentations of different dog images tend to be similar to each other, but their union has little overlap with distributions of augmentations of cat images.
\citet{arora2019theoretical} ---which predates the recent wave of methods---assume that points in the same class share the same augmentation distribution, and use this to show that the contrastive loss is a surrogate to the downstream performance. Since  methods like SimCLR \citep{chen2020simple} do not appear to satisfy such assumptions, recently \citet{haochen2021provable} gave a more refined analysis under milder assumptions, that require only some overlap in augmentation distributions, such that the resultant graph of connections due to overlaps within a class is dense.
Again, it can be shown that a low-dimensional representation that is near-optimal in the contrastive loss is guaranteed to linearly separate the downstream classes.

Note that properties of the class of representation functions --- VGG16, ResNet18 etc. --- or the training algorithm --- SGD, Adam etc. --- make no appearance in the above analyses;  but only properties of the {\em augmentation distributions} and {\em value of contrastive loss} of representations.
This is understandable since currently theory is unable to pinpoint why different real-life architectures differ in their capabilities, or to pinpoint the implicit bias of the training algorithms.
Nevertheless it raises interesting questions: {\em Is the contrastive loss indeed a good indicator of downstream performance?}
{\em Do augmentations overlap sufficiently enough in practice to explain the success of contrastive learning?} 
{\em Can contrastive learning succeed even when there is little to no overlap?} 
In a nutshell, the current paper suggests via experiments and simple theory that the answers are, respectively: {\em No, No, Yes.}
In particular, ignoring the architecture and the training algorithm can make the current theoretical analyses of contrastive learning vacuous. 
We present three key phenomena with regards to this:

\ifthenelse{\boolean{arxiv}}{
\begin{figure}[t!]
    \centering
    \includegraphics[width=0.4\linewidth]{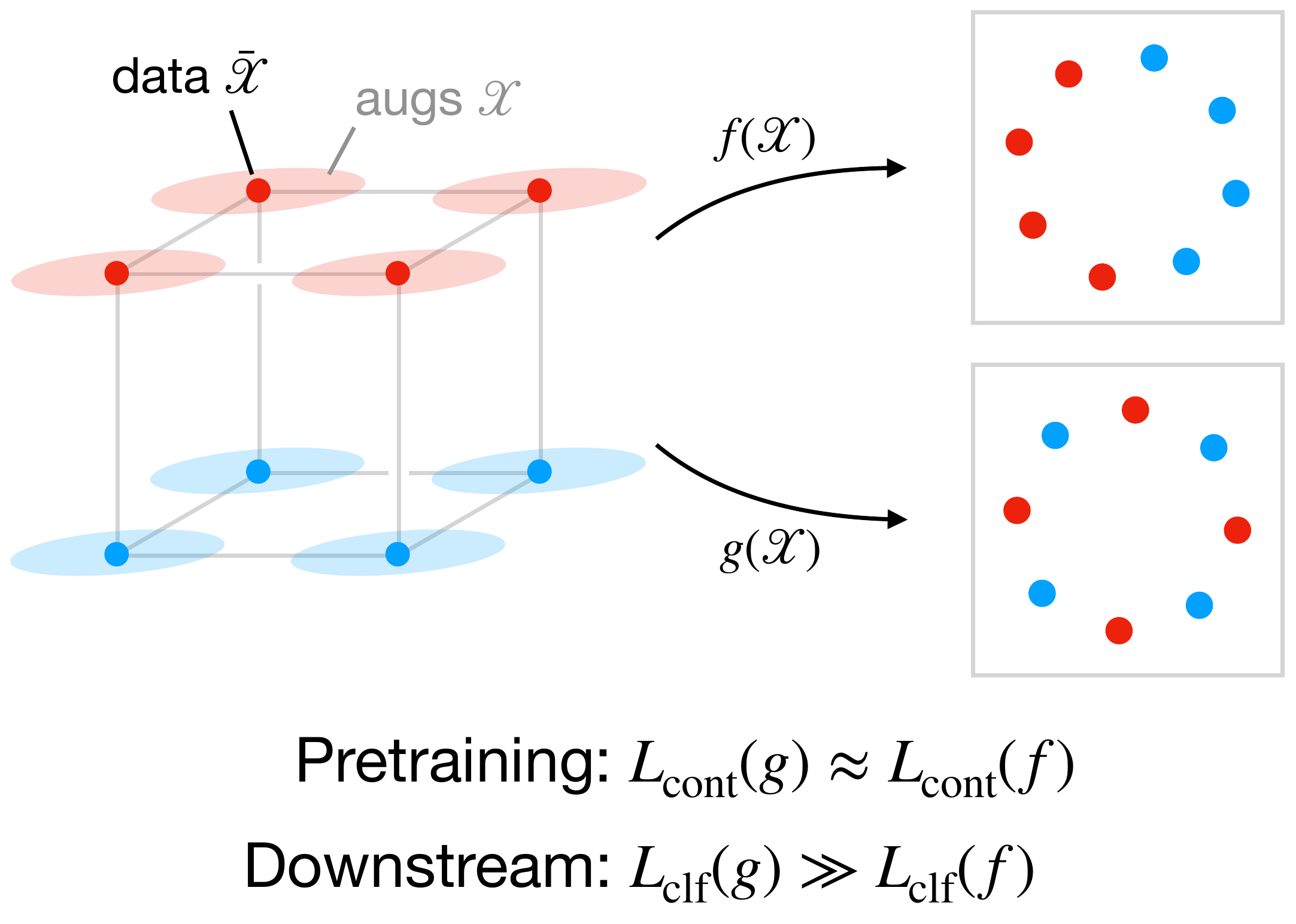}
    \caption{Cartoon of our theoretical example. The downstream labels (and thus classification loss $\Lclf$) are determined by a few relevant attributes (e.g. cat or dog?), and the augmentations perturb irrelevant attributes (e.g. grayscale, random crop). Without restricting the function class for the contrastive pretraining task, there exist perfect ($f$) and spurious ($g$) {\em augmentation-invariant} representations which both minimize the contrastive loss $\Lcont$. However, minimizing using a linear representation class is always guaranteed to succeed with these augmentations (\Cref{sec:hypercube}).}
    \label{fig:hypercube-thought-experiment}
\end{figure}
}{
\begin{figure}[t!]
    \centering
    \includegraphics[width=0.8\linewidth]{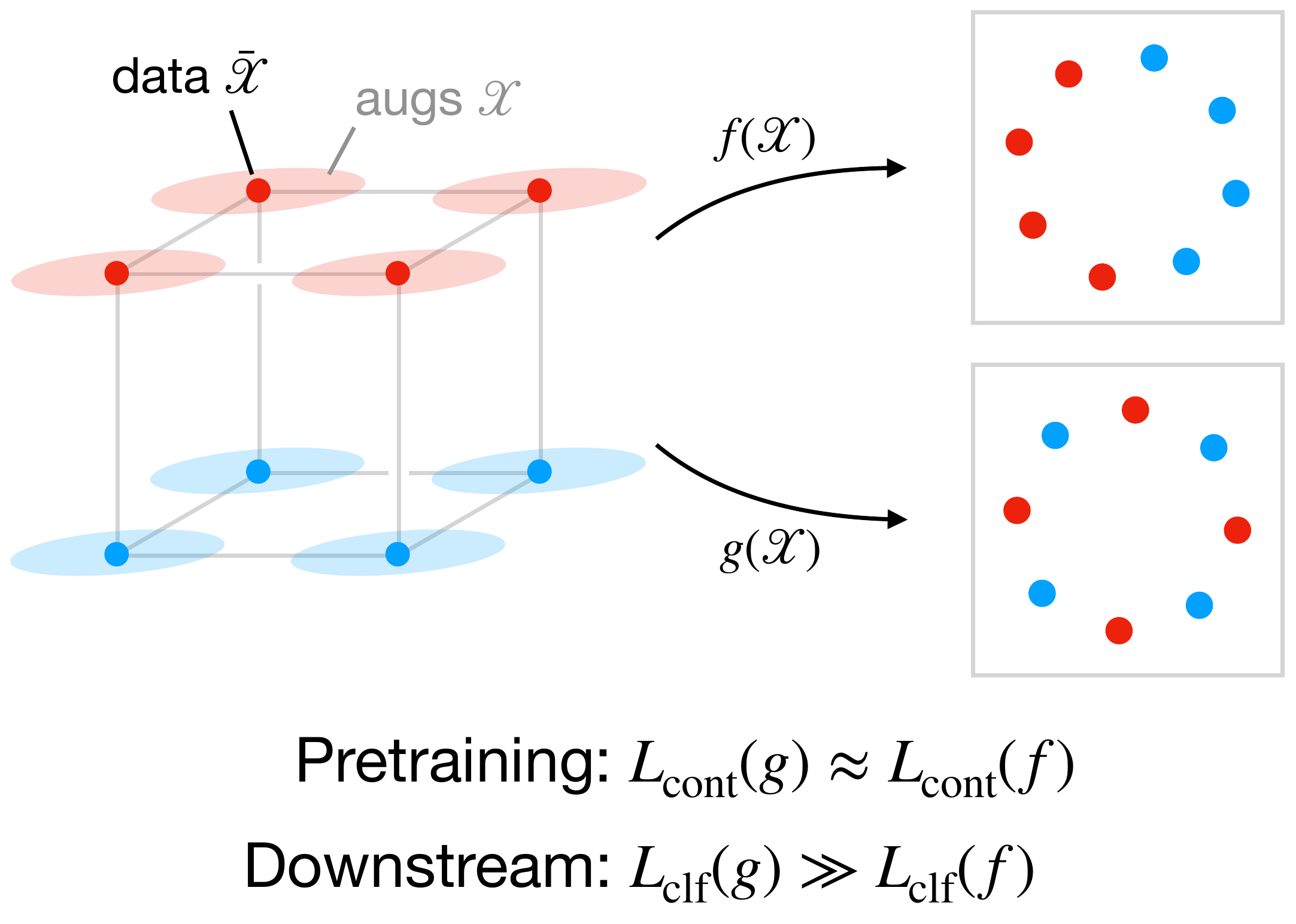}
    \caption{Cartoon of our theoretical example. The downstream labels (and thus classification loss $\Lclf$) are determined by a few relevant attributes (e.g. cat or dog?), and the augmentations perturb irrelevant attributes (e.g. grayscale, random crop). Without restricting the function class for the contrastive pretraining task, there exist perfect ($f$) and spurious ($g$) {\em augmentation-invariant} representations which both minimize the contrastive loss $\Lcont$. However, minimizing using a linear representation class is always guaranteed to succeed with these augmentations (\Cref{sec:hypercube}).}
    \label{fig:hypercube-thought-experiment}
\end{figure}
}

\begin{itemize}[leftmargin=*,noitemsep]
    \item {\bf Function class sensitivity.} Downstream performance of a representation depends not just on its contrastive loss, but it is also sensitive to the function class (architecture) and training procedure used to learn it.
    \item {\bf Brittleness of transfer.} Minimizing the contrastive loss to optimality can sometimes have a non-monotonic, deleterious effect on downstream performance, despite the augmentations being effective for some function classes.
    \item {\bf The disjoint augmentations regime.} When augmentation distributions for inputs do not overlap with each other, it can be shown that any function-class-agnostic analysis (including those from prior work) provably leads to vacuous guarantees.
    That said, non-overlapping augmentations can sometimes still be informative,
     and contrastive learning with appropriate function classes can succeed, a phenomenon that is not captured by existing theory.
\end{itemize}


\noindent\textbf{Organization.} We define the contrastive losses and downstream performance in \Cref{sec:preliminaries}, and summarize prior theoretical results and how they ignore inductive biases.
In \Cref{sec:hypercube} we describe a simple synthetic setting that elucidates all of the aforementioned phenomena.
A pictorial depiction in \Cref{fig:hypercube-thought-experiment} demonstrates the existence of bad contrastive solutions, despite the augmentations satisfying intuitive properties.
These ideas are grounded through theoretical results in \Cref{sec:theory}, which includes lower bounds for function class agnostic analyses and upper bounds that are sensitive to the function class of linear representations.
Finally we describe various experimental setups in \Cref{sec:experiments}.









\vspace{-1ex}\subsection{Related work}
\label{sec:related}


Contrastive learning has been very successful at solving downstream tasks by learning representations from similar pairs of data obtained using temporal information \citep{wang2015unsupervised,logeswaran2018efficient} or different views or augmentations of inputs \citep{dosovitskiy2014discriminative,hjelm2018learning,wu2018unsupervised,bachman2019learning,tian2019contrastive,chen2020simple,chen2021exploring,gao2021simcse}.
Given its empirical success, there has been significant interest in the theory of contrastive learning, from various perspectives.
Most relevant to us are learning theoretic analyses \citep{arora2019theoretical,tosh2020contrastive0,tosh2020contrastive,haochen2021provable,wang2022chaos} and their follow ups~\citep{nozawa2021understanding,ash2021investigating}.
These study the downstream linear classification performance of learned representation, by making assumptions about the data and augmentation distributions; we discuss these in more detail in \Cref{sec:prior_analyses}.

Contrastive learning has also been studied (1) from a mutual information maximization view \citep{oord2018representation,hjelm2018learning,bachman2019learning}; \citep{tschannen2019mutual} points out certain issues with this view, 
(2) using an information theoretic framework \citet{tsai2021selfsupervised}; fails to explain downstream success via simple linear classifiers, (3) through properties like alignment and uniformity on the sphere \citet{wang2020understanding}, (4) under certain latent variable data generative processes \citep{zimmermann2021contrastive,von2021selfsupervised}, and (5) through a causality perspective \citep{mitrovic2021representation}.
On the optimization front, \citep{wen2021toward} study the feature learning process of contrastive learning with gradient dynamics on a two layer network, under a sparse coding model. The theory of noise contrastive estimation \citep{gutmann2010noise} has been a useful motivation for negative sampling based objectives.
On the empirical side, there are studies on identifying useful augmentation properties \citep{tian2020what}. \looseness=-1


Non-contrastive methods, with no negative samples, \citep{chen2021exploring,grill2020bootstrap} rely on tricks like stop-grad to avoid representation collapse.
Dimension collapse of representations has also been studied \citep{jing2021understanding}.
Unlike these works, the brittleness of transfer we study is neither due to ill-designed objectives nor due to training degeneracies.
It is fundamental to data distributions and arises out of existence of spurious solutions.
A related idea of feature suppression \citep{chen2021intriguing} and shortcut solutions found by contrastive learning was recently studied in \citet{robinson2021can} in certain stylized settings, with a proposed fix through better augmentations strategies. 
We instead study the role of inductive bias of function classes in avoiding such shortcut solutions.
\citep{abnar2022exploring} analyze upstream to downstream transfer for supervised pre-training, complementing our experiments for unsupervised pre-training, whereas \citet{wu2020Understanding} studies negative transfer for multi-task learning.
Finally, there are theoretical works for other types of self-supervised learning \citep{bansal2021for}, including methods like context reconstruction \citep{lee2020predicting} and language model \citep{saunshi2021mathematical}, studying their benefits on downstream tasks.

\section{Preliminaries}
\label{sec:preliminaries}

Here we formalize the problem of learning useful representations via contrastive learning for downstream classification. 

\noindent\textbf{Notation.} We use $[n]$ for the set $\{1,\dots,n\}$. $\gU(S)$ denotes uniform distribution over a set $S$. For a vector $v\in\R^{n}$, we denote $v_{:i}\in\R^{i}$ and $v_{i:}\in\R^{n-i}$ to be the sub-vector of first $i\in[n]$ and last $i$ coordinates respectively.
For sets $P, Q$, we use $P^{Q}$ to denote the set of functions from $Q$ to $P$.

\noindent\textbf{Augmentations.}
We use $\Xorig$ to denote the set of all (unaugmented) samples and denote their marginal distribution as $\Dorig$.
$\Xaug$ denotes the set of all augmented data.
For an input $\bar{x} \in \Xorig$, we define the corresponding augmentation distribution over $\Xaug$ as $\gA(\cdot \mid \bar{x})$.
For instance, augmentations for an image $\bar{x}$ can correspond to applying a sequence of random transformations such as random cropping, Gaussian blur, and color jitter.
The distributions $\Dorig$ and $\gA$ together induce a marginal distribution $\Daug$ over augmentations.

\noindent\textbf{Contrastive self-supervised learning.}
\label{sec:contrastive_learning}
The goal is to learn a representation function $f: \Xaug \rightarrow \mathbb{R}^{\dd}$ that maps augmentations to $\dd$-dimensional vectors by encouraging representations of ``similar pairs'' of augmentations to be closer to each other, compared to representations of random pairs.
A common strategy to pick a similar pair $(x,x^{+})$ is to pick two augmentations of the same input.
Formally we define this distribution of similar pairs $\Dsim$ as follows
\begin{align*}
    (x,x^{+}) \sim \Dsim ~\equiv~ \bar{x} \sim \Dorig;~ x,x^{+}\sim_{\text{i.i.d.}}\gA(\cdot \mid \bar{x})
\end{align*}
The negative sampling distribution, denoted by $\Dneg$, is picked to be the augmentation marginal distribution $\Daug$.
There are several variants of the contrastive loss, a popular one being the \textit{SimCLR loss} \citep{chen2020simple}.
\ifthenelse{\boolean{arxiv}}{
    \begin{align}
        \Lsimclr(f) =
        &\ex_{\substack{(x,x^{+})\sim\Dsim, x^{-}_{1:n}\sim\Dneg^{n}}}\left[-\log\left(\frac{e^{f(x)^\top f(x^+)}}{e^{f(x)^\top f(x^+)} + \sum_{i=1}^n e^{f(x)^\top f(x^-_i)}}\right)\right]\label{eqn:simclr_loss}
    \end{align}
}{
    \begin{align}
        &\Lsimclr(f) =\label{eqn:simclr_loss}\\
        &\ex_{\substack{(x,x^{+})\sim\Dsim,\\ x^{-}_{1:n}\sim\Dneg^{n}}}\left[-\log\left(\frac{e^{f(x)^\top f(x^+)}}{e^{f(x)^\top f(x^+)} + \sum_{i=1}^n e^{f(x)^\top f(x^-_i)}}\right)\right]\nonumber
    \end{align}
}
Intuitively the contrastive loss aims to make $f(x)^{\top} f(x^{+})$ larger compared to $f(x)^{\top} f(x^{-}_{i})$.
Another variant proposed in \citet{haochen2021provable} is the spectral contrastive loss:
\ifthenelse{\boolean{arxiv}}{
    \begin{align}
        \Lspec(f) =
        \ex_{\substack{(x,x^{+})\sim\Dsim}}\left[-2 f(x)^{\top}f(x^{+})\right] + \ex_{\substack{x,x^{-}\sim\Dneg^{2}}}\left[ \left(f(x)^{\top}f(x^{-})\right)^{2}\right]\label{eqn:spectral_loss}
    \end{align}
}{
    \begin{align}
        &\Lspec(f)= \label{eqn:spectral_loss}\\
        &\ex_{\substack{(x,x^{+})\\\sim\Dsim}}\left[-2 f(x)^{\top}f(x^{+})\right] + \ex_{\substack{x,x^{-}\\\sim\Dneg^{2}}}\left[ \left(f(x)^{\top}f(x^{-})\right)^{2}\right]\nonumber
    \end{align}
}

We will use $\Lcont$ to refer to a generic contrastive loss, either $\Lsimclr$ or $\Lspec$ or something else.

\noindent\textbf{Downstream task.}
We assume these involve binary classification\footnote{We consider binary tasks mostly for simplicity. Extensions of our results (lower bounds for function class agnostic analyses and upper bound guarantees for linear representations) to more than two classes are not difficult.}.
If the ground-truth labeling function is $\ystar: \Xorig \rightarrow \{\pm 1\}$, the quality of
representation $\bar{f}:\Xorig\rightarrow\R^{d}$ is captured by how well it allows {\em linear classification}:
\begin{align}
    \Lclf(\bar{f}; \ystar)
    &= \inf_{w\in\R^{\dd}} \E_{\bar{x}} \left[\mathbbm{1}\left\{\ystar(\bar{x}) \left({\bar{f}(\bar{x})}^{\top} w\right) < 0\right\}\right]
    \label{eqn:clf_loss}
\end{align}
Since the representation function is trained to map augmentations to vectors, its behavior on unaugmented inputs can be undefined.
We evaluate downstream performance on original inputs $\Xorig$ by using the average augmentation representation $\aug{f}: \Xorig \rightarrow \R^{\dd}$, defined as:
\begin{align}
\label{eqn:aug_mean_rep}
    &\aug{f}(\bar{x})
    = \E_{x \sim \gA(\cdot | \bar{x})}[f(x)],~
    \Lclf(f; \ystar)
    \coloneqq \Lclf(\aug{f}; \ystar)
\end{align}
Experimentally such an average gives better performance than the standard un-averaged approach.

\noindent\textbf{Transfer Bounds.}
\label{sec:prior_analyses}
We introduce an abstraction of {\em transfer function} $\gT$ to capture prior analyses \citep{arora2019theoretical,tosh2020contrastive,haochen2021provable}. It translates performance on contrastive loss to performance on the downstream task as follows:   
\begin{align}
    \Lclf(f; \ystar)
    &\le \gT(\Gamma, \Lcont(f), \dd),~\text{where } \Gamma=(\Dorig, \gA, \ystar, \Lcont)\label{eqn:prior_bound_general}
\end{align}
These guarantees only depend on (1) problem dependent quantities like input marginals $\Dorig$, properties of augmentations $\gA$, downstream label $\ystar$, form of contrastive loss $\Lcont$, (2) contrastive loss $\Lcont(f)$ of the representation $f$ and (3) its dimensionality $d$.
Typically $\gT$ is monotone non-decreasing function of $\Lcont(f)$ and so the above bound justifies minimizing the contrastive loss.
A common property of augmentations and labels that transfer bounds assume is \textit{overlap} between augmentations of images from the same class.
For instance, \cite{arora2019theoretical} effectively assume full overlap, that is, all images from the same class have identical augmentation distribution, and these distributions are different for different classes. \cite{haochen2021provable} relax this requirement to a spectral quantity that depends on the ratio of overlap between augmentation distributions of the same class and those of different classes, leading to a bound $\Lclf(f) \le \alpha (\Lcont(f) - \min_{f^{\star}} \Lcont(f^{\star})) + \beta$, where $f^{\star}$ is a minimizer of the contrastive loss, and $\alpha,\beta$ quantify the overlap in augmentations.
These bounds place a premium on the value of contrastive loss of $f$, but are agnostic to any other properties of $f$, like the representation function class $\gF$ it belongs to or how it was trained.
We are interested in transfer bounds that also incorporate these effects.
A simple abstraction that incorporates the function class is
\begin{align}
    \Lclf(f; \ystar)
    &\le \gT(\Gamma, \Lcont(f), \gF),~\text{where } \Gamma=(\Dorig, \gA, \ystar, \Lcont)\label{eqn:fn_class_dep_bound}
\end{align}
A bound like this, unlike the one in \Cref{eqn:prior_bound_general}, reflects that the downstream performance at a particular value of contrastive loss depends also on the representation function class, which we also find to be true in many experiments.

\noindent\textbf{Not about generalization.}
The above bounds only deal with upstream and downstream {\em population losses}.
Thus the role of function class bias is not for guaranteeing good generalization properties, as in supervised learning.
It is more subtle, as will become evident in the following sections.

\section{Warm-up: contrastive learning on hyper\mancube s}
\label{sec:hypercube}

\ifthenelse{\boolean{arxiv}}{
\begin{figure}
    \centering
    \includegraphics[width=.5\linewidth]{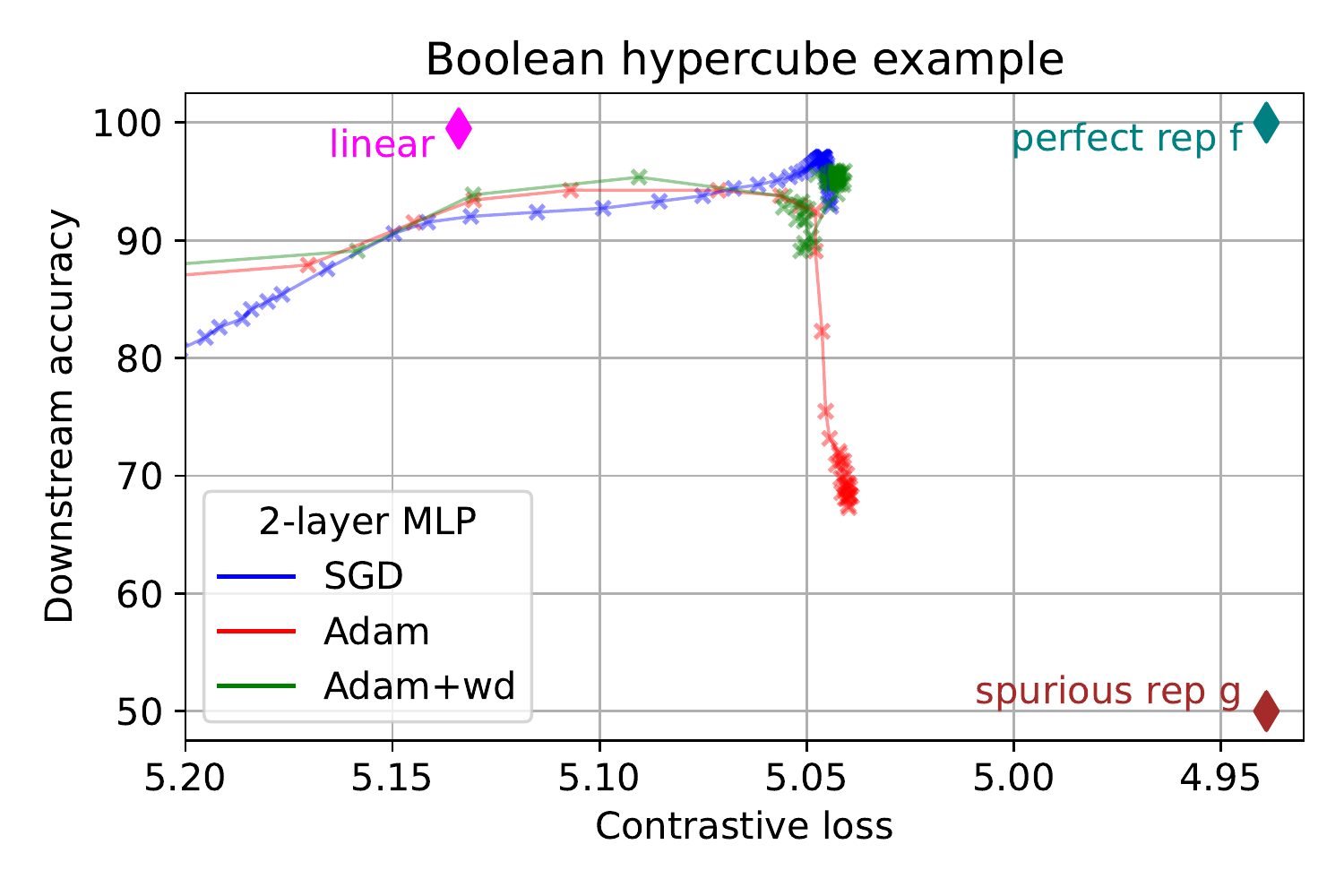}
    \caption{Contrastive loss $\rightarrow$ accuracy transfer plots for the Boolean hypercube example. There exist global minimizers of $\Lcont$ with perfect (top right) and worst possible (bottom right) downstream classification error $\Lclf$. The representations learned by two-layer neural networks are very sensitive to training configuration. With a smaller (linear) function class, the contrastive loss minimizer gives a nearly-perfect downstream classifier.}
    \label{fig:synth_scatter_plot_small}
\end{figure}
}{
\begin{figure}
    \centering
    \includegraphics[width=.9\linewidth]{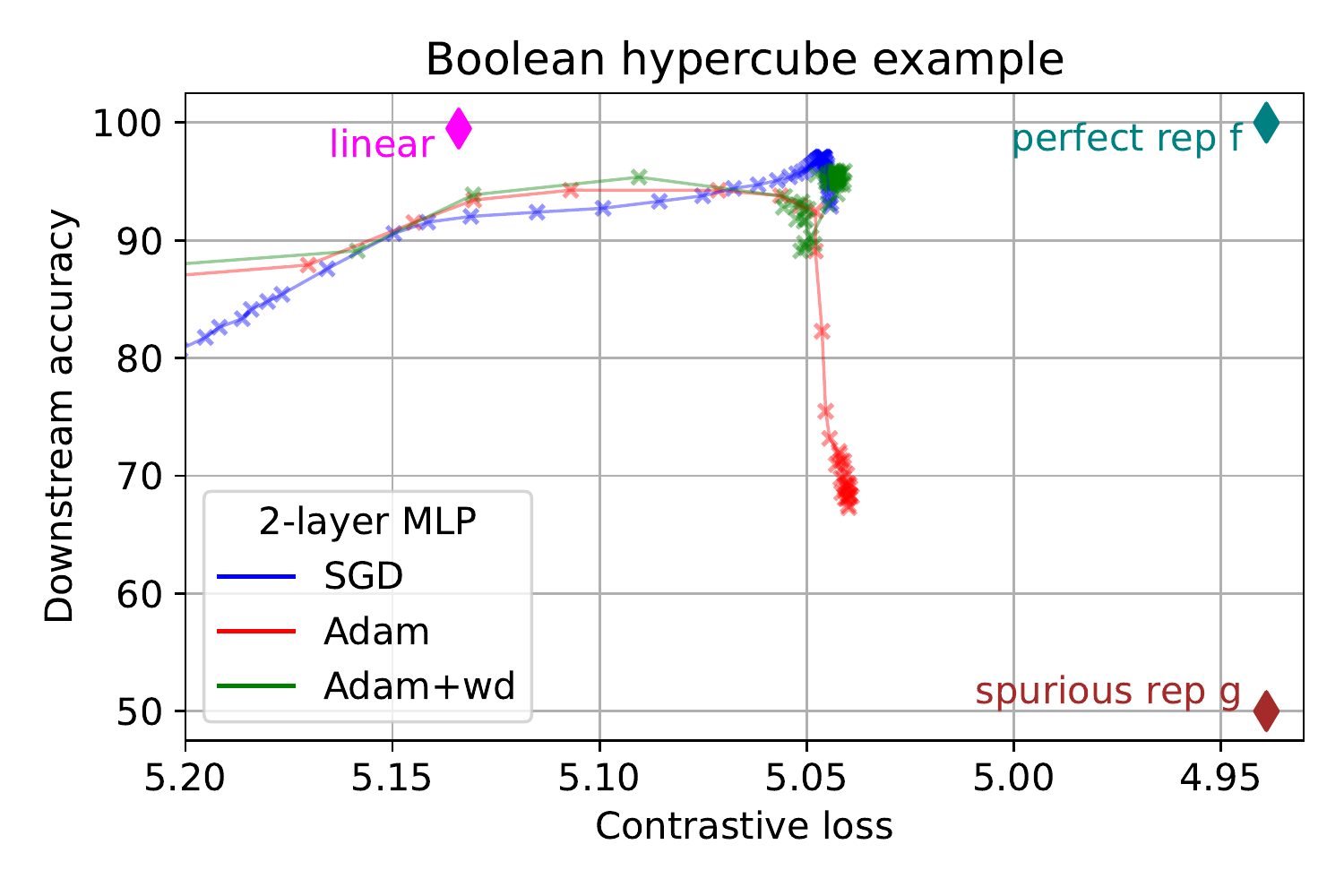}
    \caption{Contrastive loss $\rightarrow$ accuracy transfer plots for the Boolean hypercube example. There exist global minimizers of $\Lcont$ with perfect (top right) and worst possible (bottom right) downstream classification error $\Lclf$. The representations learned by two-layer neural networks are very sensitive to training configuration. With a smaller (linear) function class, the contrastive loss minimizer gives a nearly-perfect downstream classifier.}
    \label{fig:synth_scatter_plot_small}
\end{figure}
}

In this section, we present a simple but illustrative example that succinctly highlights brittleness of transfer and the importance of incorporating inductive biases in transfer bounds.
Since contrastive learning tries to make representations invariant to augmentation transformations, ideal augmentations are those that retain parts of the input that can predict the downstream label, but modify parts that are less important for the label.
We now describe a simple example on the Boolean hypercube that captures these intuitions.

\begin{example}
\label{eg:hypercube}
\normalfont
    The input set is $\Xorig=\{\pm1\}^{D}$, the augmentation set is $\Xaug=\R^{D}$.
    Downstream label $\ystar$ is linear in the first $k \ll D$ coordinates.
    \begin{align*}
        \ystar(\bar{x}) = \sign({\wstar}^{\top} \bar{x}_{:k}), ~\wstar\in\R^{k}
    \end{align*}
    Augmentation distribution $\gA(\cdot \mid \bar{x})$ for input $\bar{x}\in\Xorig$ randomly scales down the last $k$ coordinates while keeps the first $k$ coordinates unchanged\footnote{Can be generalized to only downscaling random subsets of $\bar{x}_{k:}$, analogous to downscaling different aspect of an image like grayness, sharpness}.
    Formally it is defined as
    \begin{align*}
        x \sim \gA(\cdot \mid \bar{x}) ~\equiv~ \tau \sim \gU((0,1]), ~x_{:k} = \bar{x}_{:k}, ~x_{k:} = \tau \bar{x}_{k:}
    \end{align*}
    where $\gU((0,1])$ is the uniform distribution over $(0,1]$.
\end{example}

We experimentally study this example using two function classes to minimize the contrastive objective: MLP, linear.
Results from \Cref{table:hypercube} and \Cref{fig:synth_scatter_plot_small} are summarized below.


\noindent\textbf{Transfer is sensitive to function class and algorithm:}
Firstly, we notice that despite having much worse (higher) contrastive loss compared to MLP, linear representation has significantly better downstream performance.
Secondly, \Cref{fig:synth_scatter_plot_small} suggests that even for the same MLP architecture, the training algorithm (Adam v/s SGD, weight decay or not) can drastically affect the downstream accuracy.


\noindent\textbf{Brittleness of transfer and disjoint augmentations.}
Although the augmentations in the example seem intuitively helpful, there exists a spurious representation that has much smaller contrastive loss than all architectures, but random guessing performance downstream. 
This, as we show in later sections, is a consequence of having disjoint augmentation distributions, since the original input can be recovered from an augmentation by simply performing $\bar{x} = \sign(x)$.
The existence of a bad minimizer of the contrastive loss also gives us a concrete case where contrastive learning can succeed with an appropriate function class, but the success cannot be explained by any function class agnostic analysis.
With this backdrop, we present our theoretical results next.



\begin{table}[t!]
\begin{center}
\caption{Contrastive loss and downstream accuracy of various representation classes and training procedures used to train on the hypercube example. There exist minimizers of the contrastive loss which transfer to perfect downstream classifiers, and ones which are no better than random guessing. Function class matters (MLP vs. linear vs. any representation), as does the training algorithm.}
	\label{table:hypercube}
\begin{tabular}{ |c|c|c| } 
 \hline
    Representation & Contrastive loss & Accuracy (\%) \\
 \hline
    $\exists f$ (perfect) & 4.939 & 100\\
    $\exists g$ (spurious) & 4.939 & 50\\
    MLP + Adam & 5.039 $\pm$ 0.001 & 74.1 $\pm$ 4.3\\
    MLP + Adam + wd & 5.040 $\pm$ 0.002 & 89.5 $\pm$ 4.9\\
    Linear & 5.134 $\pm$ 0.002 & 99.5 $\pm$ 0.1\\
 \hline
\end{tabular}
\end{center}
\end{table}

\section{Lower bounds and improved analysis}
\label{sec:theory}

In this section we discuss the role of overlap in augmentations and function class in theoretical guarantees.
We first show in the disjoint augmentation regime (augmentation distributions do not overlap), that {\em any} function class independent analyses will lead to vacuous bounds, which includes many previous analysis.
Delving deeper into the most recent results from \citet{haochen2021provable}, we discuss reasons for failure, even in approximately disjoint augmentations.
Finally we present guarantees for contrastive learning with a linear representation function class that is sensitive to the function class and allows for weaker assumptions on augmentations.
We instantiate this bound for the hypercube example, provably explaining the good performance of linear representations on disjoint augmentations. 

\subsection{Lower bound for disjoint augmentations}
\label{sec:lower_bound}

In this section, we prove that brittle transfer exists much more generically whenever the augmentation distributions for different inputs do not overlap, generalizing our observations from the hypercube example in the previous section.

%
\begin{definition}[Disjoint augmentations]
\label{defn:non_overlapping}
    We say the augmentation distributions are disjoint if for all distinct inputs $\bar{x}_1, \bar{x}_2\in\Xorig$, augmentation distributions  $\gA(\cdot \mid \bar{x}_1)$ and $\gA(\cdot \mid \bar{x}_2)$ have disjoint supports.
\end{definition}
Disjoint augmentations can be problematic because the contrastive loss only encourages separating individual instances, but does not encourage making classes linearly separable.
We formalize this argument in the next two lemmas by showing that any representation $f$ can be transformed --- by shuffling identities of examples --- to a new representation $\tilde{f}$ that has lower (or equal) contrastive loss but near-trivial downstream performance.
An immediate consequence is that any function class agnostic analysis (including all previous analyses) will necessarily leads to vacuous downstream guarantees. 
%
We establish this in two settings where the exact choice of contrastive objective is not critical; results hold for both $\Lsimclr$ and $\Lspec$ and we abbreviate these by $\Lcont$ below.
First we consider unconstrained representations.
\begin{lemma}
\label{lem:unnormalized_lb}
Let $|\Xorig|=N$ and $\dd = \gO\left(N/\log_2(N)\right)$.
Suppose the labeling function is balanced, i.e. $\sum_i y_i^\star = 0$, and let $\Daug$ be uniform over $\Xorig$. If the augmentation distribution is disjoint, then for any $f^\star: \Xaug \to \RR^{\dd}$ there exists a $\hat{f}: \Xaug \to \mathbb{R}^{\dd}$ such that:
\begin{align*}
    &\Lcont(\hat{f}) \leq \Lcont(f^\star), \textrm{ \& }
    \Lclf(\hat{f}) 
    \geq \frac{1}{2} - O\left(\sqrt{\frac{d\log(N)}{N}}\right).
\end{align*}
\end{lemma}

Since it is common to use normalized representations in practice (e.g., to have Euclidean norm 1), we also establish a similar result for this case.
\begin{lemma}
\label{lem:normalized_lb}
In the setup of~\cref{lem:unnormalized_lb}, suppose further that representations are constrained (to any given set) and that the augmentation distributions satisfy the following: There exists a fixed source of randomness $W$ and a deterministic map $T: (\bar{x},w) \mapsto x$ that is invertible in $w$ for any $\bar{x}$ such that $x \sim \mathcal{A}(\cdot \mid \bar{x}) \equiv w \sim W, x = T(\bar{x},w)$. 
Then the conclusion of~\cref{lem:unnormalized_lb} holds. 
\end{lemma}

We prove both statements jointly in \cref{sec:lower_bound_proof}.
Both lemmas show that when the representation dimension is small relative to the size of the input space (as is typical) and the augmentations are disjoint, there exists a \emph{global minimizer} of the contrastive loss with vacuous transfer to downstream.
The extra assumption in \cref{lem:normalized_lb} is that the augmentation generation protocol uses a common source of randomness, which is actually satisfied in many practical scenarios.
For instance, the same sequence of transformations like random cropping, color jittering etc. are applied to all images to generate augmentations.
The other assumptions, e.g., that the labeling function $\ystar$ is balanced and that $\Daug$ is uniform, are technical in nature and can be potentially relaxed. 

Note that Proposition 1 in \citet{robinson2021can} discusses a similar lower bound when augmentations are disjoint, arguing that contrastive learning can find ``shortcut solutions'' that can lead to feature suppression.
While the motivation is similar to ours, those results are shown specifically for contrastive loss with normalized representations, in the regime of large number of negative samples and with a specific uniform over sphere assumption on latent variables generating the data.
The above results are shown in much more general settings.

A corollary of these results is that any transfer learning bound that only depends on the value of the contrastive loss cannot be meaningful in the disjoint augmentations setting. 
\begin{corollary}
\label{cor:bias_agnostic_vacuous}
    In the setup of \Cref{lem:unnormalized_lb} or \Cref{lem:normalized_lb}, consider a transfer function $\gT$ bounding the downstream performance as $\Lclf(f; \ystar) \leq \gT(\Gamma,\Lcont(f),\dd)$ as in \Cref{eqn:prior_bound_general}, where $\Gamma = (\Dorig, \gA, \ystar, \Lcont)$ are problem dependent but function class independent quantities.
    Suppose $\gT$ is monotonic in its second argument, then for all $f: \Xaug \rightarrow \R^{\dd}$:
    \begin{align*}
       \gT\left(\Gamma, \Lcont(f),\dd\right) \ge \nicefrac{1}{2} - \tilde{\gO}\left(\sqrt{\nicefrac{d}{|\Xorig|}}\right)
    \end{align*}
\end{corollary}

%
\vspace{-0.2in}
\paragraph{Takeaways.}
The above lower bounds suggest that previous analyses for contrastive learning are vacuous in the disjoint augmentation setting, due to existence of bad minimizers of the contrastive loss.
The brittleness of transfer for disjoint augmentations is also observable in practice, as in the first row of \Cref{table:hypercube} for the hypercube example.
Vision and NLP experiments in \Cref{sec:experiments} also demonstrate this phenomenon, for more expressive function classes.

\vspace{-1.5ex}\subsection{Prior theoretical results and failure modes}
\label{sec:discussion}

We briefly discuss the results from \citet{haochen2021provable} and delve deeper into how their function class agnostic nature leads to poor guarantees even for {\em approximately} disjoint augmentations.
Their analysis considers the spectral loss $\Lspec(f)$ from \Cref{eqn:spectral_loss}.
A key component of their analysis is an {\em augmentation graph} constructed using $\gA$, whose spectral properties characterize how much overlap there is in augmentations.
This is a weighted graph on augmentations $\Xaug$ with adjacency matrix $\A\in\R^{\Xaug \times \Xaug}$ with entries $\A[x,x'] = \Dsim(x, x')$, i.e. similar augmentations have edges. 
The normalized adjacency matrix, a central object in spectral graph theory, is defined as $\An\in\R^{\Xaug \times \Xaug}$ with entries $\An[x, x'] = \frac{\Dsim(x,x')}{\sqrt{\Daug(x) \Daug(x')}}$.

Canonical results in spectral graph theory connect the eigenvalues $\lambda_{1} \le \dots \le \lambda_{|\Xaug|}$ of the normalized Laplacian $\Ln = I - \An$ to density of edges in the graph: denser graphs have larger eigenvalues.
For representation dimension $\dd$, \citet{haochen2021provable} roughly make two key assumptions: (1) any partition of the graph into $\gO(\dd)$ partitions is dense i.e. $\lambda_{\dd+1}$ is high, (2) the partition of downstream classes is sparse.
The condition (2) is the same as saying augmentations of different classes do not overlap much.
Under these assumptions, they show the following transfer bound:
\begin{theorem}[Theorem 4.2 from \citet{haochen2021provable}]
\label{thm:haochen_result}
    If $\lambda_{1} \le \dots \le \lambda_{|\Xaug|}$ are the eigenvalues of the normalized Laplacian $\Ln = I - \An$ for the augmentation distribution $\gA$, and if the augmentations can predict the original input labels with probability $1-\alpha$, then for any $\dd'\in[\dd]$ and representation $f$ we have
    \begin{align*}
        \Lclf(f; \ystar) \lesssim c_1 \frac{\alpha}{\lambda_{\dd'+1}} + c_2 \frac{\left(\Lspec(f) - \inf_{f^{\star}}\Lspec(f^{\star})\right) \dd'}{(\lambda_{\dd+1} - \lambda_{\dd'})^{2}}
    \end{align*}
    where $\Lspec(f) - \inf_{f^{\star}}\Lspec(f^{\star})$ is the sub-optimality of $f$.
\end{theorem}
Firstly we note that the above bound is function class independent and fits the abstraction from \Cref{eqn:prior_bound_general}.
If augmentations are disjoint, then augmentations of an image $\bar{x}$ will be connected to each other in the augmentation graph, but disconnected from all other input augmentations.
Thus the graph $\A$ will have $|\Xorig|$ connected components, implying that the first $|\Xorig|$ eigenvalues of the Laplacian $\Ln$ are 0, i.e. $\lambda_{i} = 0$ for $i\in[|\Xorig|]$.\footnote{Standard results in spectral graph theory connect the number of connected components to the multiplicity of the eigenvalue 0 of the Laplacian.}
So any representation dimension $\dd < |\Xorig|$ leads to vacuous bounds in \Cref{thm:haochen_result}.
This again happens because the global minimizer of $\Lspec$ is not unique, and some of those could be terrible on downstream, as in our proof for \Cref{lem:unnormalized_lb}.

\noindent\textbf{Approximately disjoint augmentations.}
We show that the above bound does not scale well even when there is very little overlap in the augmentation distributions.
To quantify approximate disjointness, we consider the problem of predicting the original input $\bar{x}$ that could have generated an augmentation $x$, as a classification problem.
\begin{definition}
\label{defn:bayes_error_}
    We say an augmentation distribution $\gA$ is $1-\tau$ disjoint when the minimum error achievable in the input identification task, i.e. predicting the input $\bar{x}$ that could have generated an augmentation $x$, is at most $\tau$.
    Formally this means
    \begin{align}
    \label{eqn:bayes_error_}
        \inf_{g: \Xaug \rightarrow \Xorig} ~\ex_{\bar{x}} \left[\ex_{x\sim\gA(\cdot \mid \bar{x})} \left[\mathbbm{1}\left\{g(x) \neq \bar{x}\right\}\right]\right] \le \tau
    \end{align}
\end{definition}
The augmentation distributions are disjoint if and only if one can perfectly predict $\bar{x}$ from $x$, i.e. under the disjoint augmentation setting from \Cref{defn:non_overlapping}, it is easy to see that $\gA$ is 1-disjoint.
The following result shows that the eigenvalues and eigen-gaps in \Cref{thm:haochen_result} will be small if the augmentation classification accuracy is high.
\begin{lemma}
\label{lem:aug_clf}
    Suppose again that $|\Xorig|=N$ and $\Dorig = \gU(\Xorig)$. If the augmentations are $1-\tau$ disjoint (defined in \Cref{defn:bayes_error_}), i.e. average accuracy of predicting $\bar{x}$ from $x$ is $1-\tau$, then for $\dd'\in[\dd]$,
    \begin{align*}
        \lambda_{\dd+1} - \lambda_{\dd'} \le \lambda_{\dd+1} \le \frac{2\tau}{(1-\nicefrac{\dd}{N})}
    \end{align*}
\end{lemma}
Thus for a small representation dimension $\dd \ll N$, the guarantees from \Cref{thm:haochen_result} are non-vacuous only when $\Lcont(f) \le \inf_{f^{\star}} \Lcont(f^{\star}) + \gO(\tau^{2})$, which is a stringent condition to satisfy.
The proof of this is presented in \Cref{sec:apx_apx_disjoint_lemma}.
We evaluate this augmentation classification metric on standard augmentations on images, and find that the accuracy achievable is almost 100\%, suggesting that we might be closer to the disjoint augmentation setting than we think, but contrastive learning still succeeds.
Given that prior analysis fail, we now proceed to show function class dependent guarantees that can show tighter bounds.



\vspace{-1.5ex}\subsection{Function class dependent transfer guarantees}
\label{sec:upper_bound}

We present guarantees for a representation that incorporates the function class in addition to the contrastive loss and augmentations.
Results in this section are for the spectral contrastive loss defined in \Cref{eqn:spectral_loss}.
For simplicity we assume that the input and augmentation sets are finite.

We consider a representation class that is linear in fixed features $\phi : \Xaug \rightarrow \R^{\DD}$, defined as
\begin{align}
    \gF_{\phi} = \left\{f(\cdot) = W^{\top} \phi(\cdot) \mid W \in \R^{\DD \times \dd}\right\}\label{eqn:fn_class_phi}
\end{align}
A crucial property of the function class $\gF_{\phi}$ is that it is expressive enough to solve the downstream task on {\em augmentations} well, even if not sample efficiently.
To formalize this, we define the following metrics
\begin{definition}[Expressivity]
For any augmentation representations $h:\Xaug\rightarrow\R^{d}$ on augmentation labels $g: \Xaug\rightarrow \{\pm1\}$, the regression loss is defined as
\label{defn:reg_loss}
    \begin{align*}
        \Lreg(h; g)
        &= \inf_{w\in\R^{\dd}}~ \ex_{x\sim\Xaug} \left[\left(w^{\top} h(x) - g(x)\right)^{2}\right]
    \end{align*}
\end{definition}
\begin{definition}[Inconsistency]
\label{defn:inconsitent_pred}
    We define inconsistency of a labeling function $g \in \{\pm 1\}^{\Xaug}$ on augmentations w.r.t. ground truth labeling $\gstar\in \{\pm 1\}^{\Xorig}$ on original inputs as
    \begin{align}
        \Lcons(g, \gstar) = \ex_{\bar{x}} \left[ \ex_{x\sim \gA(\cdot \mid \bar{x})} \left[\mathbbm{1}\{g(x) \neq \gstar(\bar{x})\}\right]\right]
    \end{align}
\end{definition}
%
Denote the augmentation mean features as $\aug{\phi} = \ex_{x \sim \gA(\bar{x})} \left[\phi(x)\right]$ and covariance as $\Sigma(\phi) = \ex_{x} \left[\phi(x)\phi(x)^{\top}\right]$.
We now present the upper bound result.
%
\begin{theorem}
\label{thm:low_rank_result}
    Let $\lambda_{1},\cdots,\lambda_{\DD}$ be the eigenvalues of $I_{\DD} - \Sigma(\phi)^{-\half}\Sigma(\aug{\phi})\Sigma(\phi)^{-\half}$ in increasing order.
    Then for every $\dd'\in[\dd]$, a representation $f\in\gF_{\phi}$ will satisfy
    \ifthenelse{\boolean{arxiv}}{
        \begin{align*}
            \Lclf(f; \gstar)
            \le \frac{\min\limits_{g\in\{\pm 1\}^{\Xaug}} 4\left(2\Lcons(g, \gstar) + \sqrt{\Lreg(\phi; g)}\right)}{\lambda_{\dd'+1}}
            + \frac{2\dd'(\Lspec(f) - \inf\limits_{f^{\star}\in\gF_{\phi}}\Lspec(f^{\star}))}{(1-\lambda_{\dd'}) (\lambda_{\dd+1}-\lambda_{\dd'})^{2}}.
        \end{align*}
    }{
        \begin{align*}
            \Lclf(f; \gstar)
            \le& \frac{\min\limits_{g\in\{\pm 1\}^{\Xaug}} 4\left(2\Lcons(g, \gstar) + \sqrt{\Lreg(\phi; g)}\right)}{\lambda_{\dd'+1}}\\
            &+ \frac{2\dd'(\Lspec(f) - \inf\limits_{f^{\star}\in\gF_{\phi}}\Lspec(f^{\star}))}{(1-\lambda_{\dd'}) (\lambda_{\dd+1}-\lambda_{\dd'})^{2}}.
        \end{align*}
    }
\end{theorem}
Firstly note that the this transfer bound is indeed of the form $\Lclf(f;\ystar) \le \gT(\Gamma, \Lspec(f), \gF_{\phi})$ as in \Cref{eqn:fn_class_dep_bound}, since all the eigenvalues, and the inconsistency and regression metrics in the above bound depend on the features $\phi$ that defines the function class.
This is unlike the guarantee from \citet{haochen2021provable} in \Cref{thm:haochen_result}, where the eigenvalues depend only on the data distributions.
We discuss the result in more detail in \Cref{sec:apx_discussion} and present its proof in \Cref{sec:apx_low_rank}.

This result can in fact recover \Cref{thm:haochen_result}, in the special case of $\phi$ being full rank, i.e. $\DD = |\Xaug|$.
In this case we have
\begin{itemize}
    \item $\Lreg(\phi;g)=0$, since a full rank $\phi$ can express any function in $\R^{\Xaug}$.
    \item $\inf_{f^{\star}\in\gF_{\phi}}\Lspec(f)=\inf_{f^{\star}}\Lspec(f)$ since $\gF_{\phi}$ can express all $d$-dimensional representations.
    \item $\min_{g} \Lcons(g, \gstar) = \gO(\alpha)$, where $\alpha$ is defined in \Cref{thm:haochen_result} as the minimum error in predicting labels from augmentations. This can be seen by setting $g$ to be an optimal augmentations to label predictor, and plugging into \Cref{defn:inconsitent_pred}.
    \item Finally, the matrix $I_{\DD} - \Sigma(\phi)^{-\half}\Sigma(\aug{\phi})\Sigma(\phi)^{-\half}$ is precisely the normalized Laplacian from \Cref{sec:discussion}. Proof of this is presented in \Cref{lem:matrix_to_laplacian}.
\end{itemize}
However when $\phi$ is not full rank, we get a function class dependent bound with non-vacuous guarantees under weaker assumptions, as will be evident in the next part.
\vspace{-0.1in}

\paragraph{Revisiting hypercube setting.}
\label{sec:hypercube_special_case}

We provide theoretical explanations for some of the observations from \Cref{sec:hypercube} by instantiating our lower and upper bounds for the hypercube example.
\begin{corollary}
\label{prop:hypercube_results}
    Consider the setting from \Cref{eg:hypercube}.
    Suppose the classifier is $\wstar = e_{1} \in \R^{k}$, so the downstream label is $\ystar(\bar{x}) = \bar{x}_{1}$.
    Furthermore, let the feature map $\phi$ be an identity mapping, i.e. $\phi(x) = x$.
    In this setting, the following statements are true:
    
        (a) All function class-agnostic transfer guarantees are vacuous. \\
        (b) For any $f\in\gF_{\phi}$, we have $\Lclf(f; \ystar) \le 32 k \left(\Lspec(f) - \inf\limits_{f^{\star}\in\gF_{\phi}}\Lspec(f^{\star})\right)$.
\end{corollary}
Result $(b)$ suggests that finding the minimizer (or an approximate minimizer) of the contrastive loss, within the class $\gF_{\phi}$, is sufficient to guarantee good downstream performance; this explains the good performance of linear representation in \Cref{table:hypercube}.
The proof of this part is presented in \Cref{sec:apx_proof_hypercube}.
Result $(a)$ explains the presence of spurious representation in the same table and also why prior analyses fail on this example, and follows from \Cref{cor:bias_agnostic_vacuous}.

\vspace{-0.1in}

\section{Experiments}
\label{sec:experiments}

\ifthenelse{\boolean{arxiv}}{
\begin{figure}[t!]
    \centering
        \includegraphics[width=.65\linewidth]{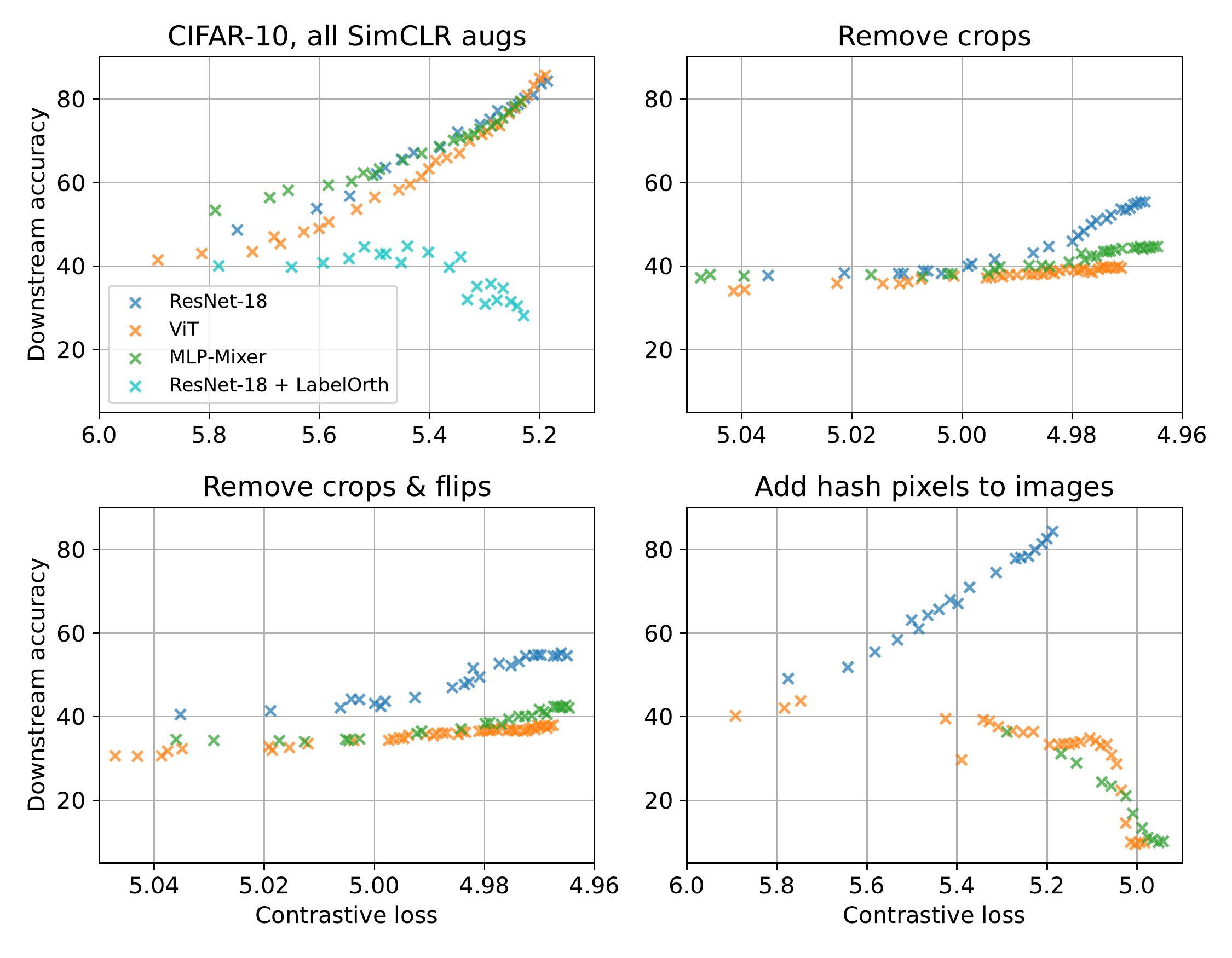}
        \vspace{-0.1in}
    \caption{Contrastive loss $\rightarrow$ accuracy transfer plots for CIFAR-10 with {\resnete}, {\vit}, and {\mlpmixer} architectures for different augmentations. {\bf TL:} Full pipeline of augmentations from SimCLR \citep{chen2020simple}. {\bf TR:} Remove random cropping. {\bf BL:} Remove random cropping and horizontal flip. {\bf BR:} Add ``hash pixels'' to each image, as described in \Cref{sec:vision_exp} to ensure that there is no overlap in augmentations. Here, we observe transfer collapse for the {\vit} and {\mlpmixer} architectures, as they overfit to these uninformative features; {\resnete} ignores these pixels.
    }
    \label{fig:vision_scatter_plot}
\end{figure}
}{
\begin{figure}[t!]
    \centering
        \includegraphics[width=.9\linewidth]{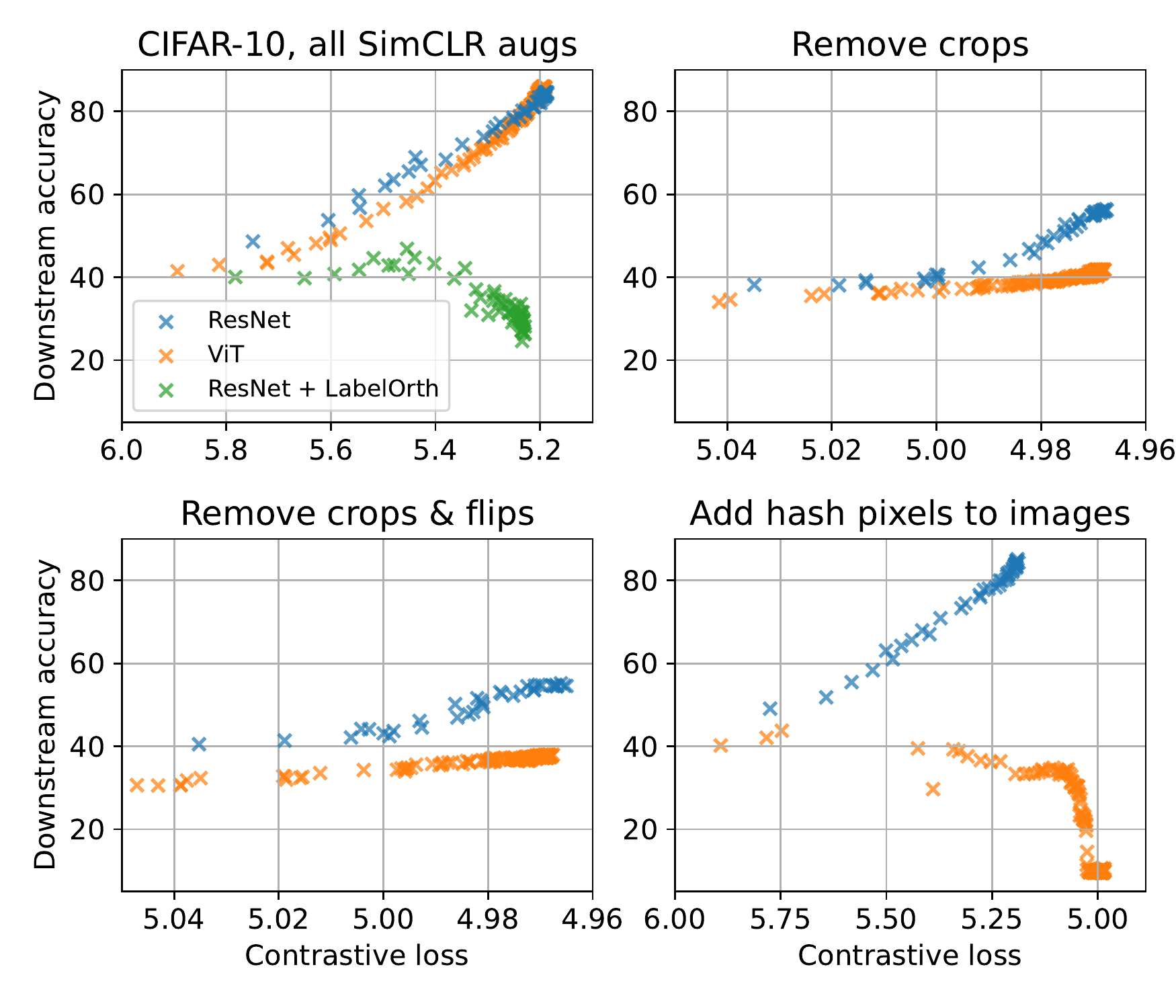}
        \vspace{-0.1in}
    \caption{Contrastive loss $\rightarrow$ accuracy transfer plots for CIFAR-10 with {\resnete} and {\vit} architectures for different augmentations. {\bf TL:} Full pipeline of augmentations from SimCLR \citep{chen2020simple}. {\bf TR:} Remove random cropping. {\bf BL:} Remove random cropping and horizontal flip. {\bf BR:} Add hash pixels to the image as described in \Cref{sec:vision_exp} to ensure that there is no overlapping in augmentations. Here we observe brittleness of transfer for {\vit}, despite {\resnete} working perfectly fine.
    }
    \label{fig:vision_scatter_plot}
\end{figure}
}

Our theoretical examples and analysis show that the prior transfer bounds (that ignore function class biases) can be near vacuous, particularly in the regime where augmentation distributions are disjoint (\Cref{cor:bias_agnostic_vacuous}) or near disjoint (\Cref{lem:aug_clf}).
Furthermore they suggest that meaningful downstream guarantees for contrastive learning would need to depend not only on the contrastive loss but also on the representation function class and possibly training algorithm.
In this section, we ask, in the context of modern contrastive learning pipelines: \emph{$(a)$ how sensitive to the function class is the contrastive loss $\rightarrow$ downstream accuracy transfer in practice?}, \emph{$(b)$ do augmentations sufficiently overlap in standard settings, as required by prior analyses} and \emph{$(c)$ can contrastive learning work when there is little to no overlap?}




\subsection{CIFAR-10 + SimCLR experiments}
\label{sec:vision_exp}

We consider the setting of CIFAR-10 image classification, where the augmentation distribution for contrastive learning is derived from the popular SimCLR protocol \citep{chen2020simple}.
An augmentation is generated by applying a series of transformations (each with some probability) to an image, like random cropping, horizontal flipping, color jittering, grayscaling and Gaussian blurring (details in \Cref{sec:apx_vision_exps}).

We run contrastive learning with standard function classes (architectures): residual convolutional networks ({\resnet}) \citep{he2016deep}, Vision Transformers ({\vit}) \citep{dosovitskiy2021an} and {\mlpmixer} \citep{tolstikhin2021mlp}.
Like in the hypercube example, we compare the transfer performance of different function class and algorithmic choices, as contrastive pre-training proceeds, by plotting the trajectories through $(\Lcont(f), 1-\Lclf(f))$ space of different setups.
Figure~\ref{fig:vision_scatter_plot} summarizes our findings at a glance and we list the key observations below:
\begin{itemize}[leftmargin=*]
    \item \textbf{Effect of function class.} While standard training using the full pipeline of SimCLR augmentations (top left) displays very similar behavior for different architectures, removal of certain transformations like random cropping (top right) and horizontal flipping (bottom left), that make the augmentations ``weaker'', can accentuate the difference in transfer performances between different architectures.

    \item \textbf{Label-orthogonal training.} All architectures behaving similarly for the full SimCLR pipeline (top left) might superficially suggest that the role of inductive biases is not that significant, and that function class agnostic guarantees are good enough to explain the practical success of contrastive learning for these augmentations.
    However for the same augmentations, we can find pathological representations that have small contrastive loss but poor downstream performance, by introducing an adversarial modification to the training algorithm and minor tweak to ResNet architecture that is detailed in \Cref{sec:label_orth}.
    This suggests that guarantees depending only on the contrastive loss, but not the function class or algorithm, cannot explain the effectiveness of contrastive learning with standard architectures and augmentations.

    \item \textbf{Hash pixels.} The difference in architectures is even more prominent in the hash pixels setting (bottom right). 
    Here we non-destructively\footnote{Only add a small number of pseudorandom pixels in random locations of a 2D image; this kind of noise can be easily removed and do not visually change images by much.} modify images and augmentations in order to force the augmentation to be in the disjoint augmentation regime, as defined in \Cref{defn:non_overlapping}.
    In this case, {\vit} and {\mlpmixer} representations make the contrastive loss much smaller than {\resnet}, but have close to random guessing downstream performance.
    {\resnet} training however is unaffected by this hash pixel modification, and it does well on the downstream task, despite being far from minimizing the contrastive loss.
    This experiment not only highlights the difference in function classes, but also concretely demonstrate a case where {\em contrastive learning can succeed despite the augmentation distributions being disjoint}.
    Details on the hash pixel augmentation are in \Cref{sec:hash_exp}.
\end{itemize}

\begin{figure}[t!]
    \centering
    \begin{minipage}[b]{0.5\linewidth}
         \includegraphics[width=\linewidth]{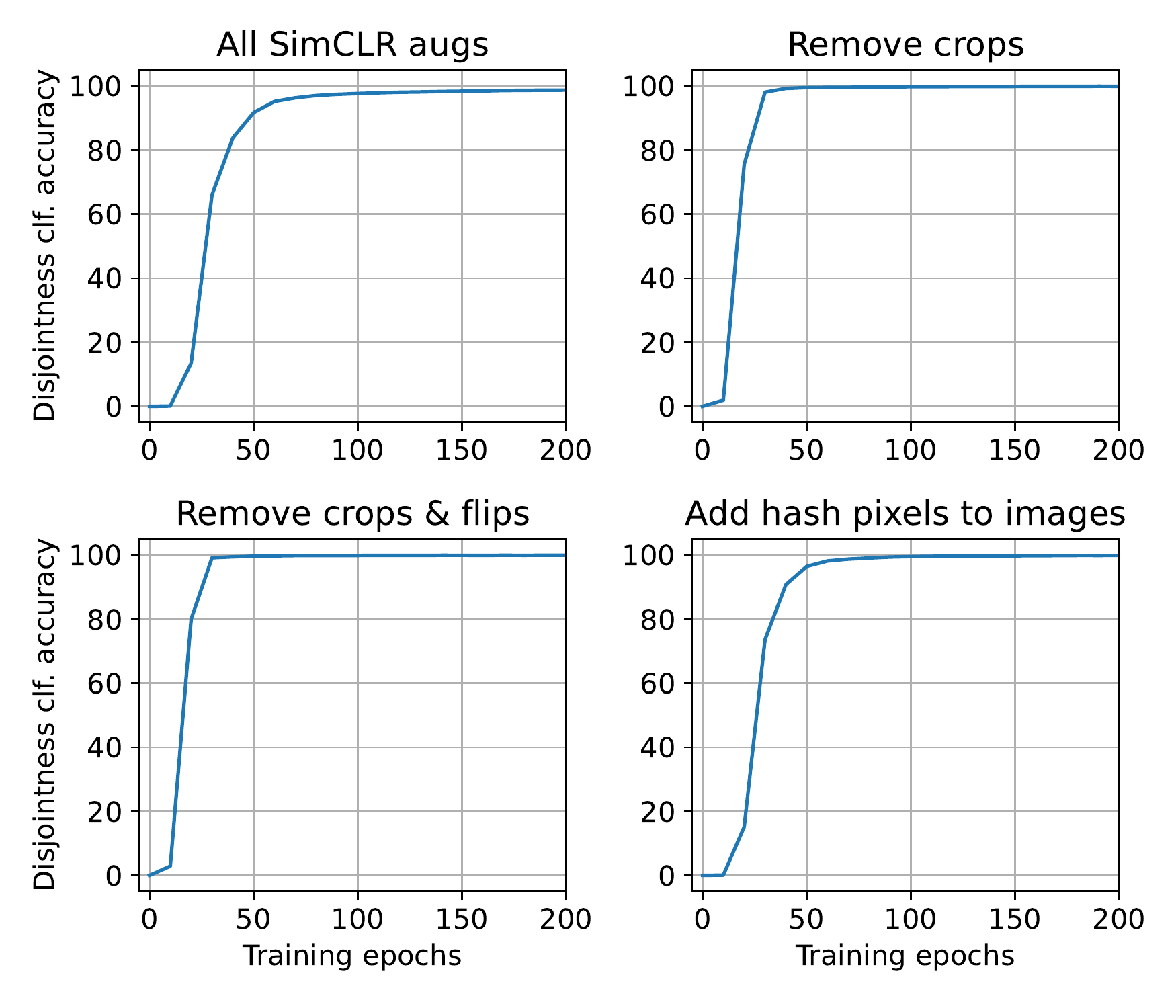}
    \end{minipage}
    \begin{minipage}[b]{0.4\linewidth}
        \begin{tabular}[b]{ |c|c|c| } 
         \hline
            Aug. distribution & Classifier acc. (\%) \\
         \hline
            CIFAR-10 + SimCLR & 99.623 $\pm$ 0.061 \\
            Remove crops & 99.981 $\pm$ 0.024 \\
            Remove crops \& flips & 99.974 $\pm$ 0.038\\
            Add hash pixels & 99.935 $\pm$ 0.033\\
         \hline
        \end{tabular}
        \vspace{20mm}
    \end{minipage}
    \caption{Demonstration of augmentation disjointness for CIFAR-10 with SimCLR augmentations.
    As described in \Cref{sec:disjoint_regime} we train classifiers to distinguish between 5000 same-class examples, for each class.
    These classifiers reach $\approx 100\%$ accuracy (averaged over 10 classes) in the 5000-way classification task, in all 4 settings from Figure~\ref{fig:vision_scatter_plot} (standard deviations shown over $10$ random epoch picked close to end of training).
    This is evidence that these distributions are close to the disjoint regime, despite contrastive learning leading to good downstream accuracy.}
    \label{fig:vision_disjoint_plot}
\end{figure}

An important point to note is that the contrastive losses and downstream accuracies in \Cref{fig:vision_scatter_plot} are measured on unseen data and are thus reflective of the population versions of these metrics; thus the difference in transfer performance is not an issue of generalization.
Further details of experimental setups and hyperparameters are in \Cref{sec:apx_vision_exps}.
We hope that experiments like these, which directly visualize the contrastive loss $\rightarrow$ downstream performance, will be adopted by the community in analyzing and mitigating the brittleness of representation learning using deep networks.
The next question we tackle is understanding how much overlap there is in augmentations for standard settings.

\ifthenelse{\boolean{arxiv}}{
\begin{figure}[t!]
    \centering
        \includegraphics[width=.65\linewidth]{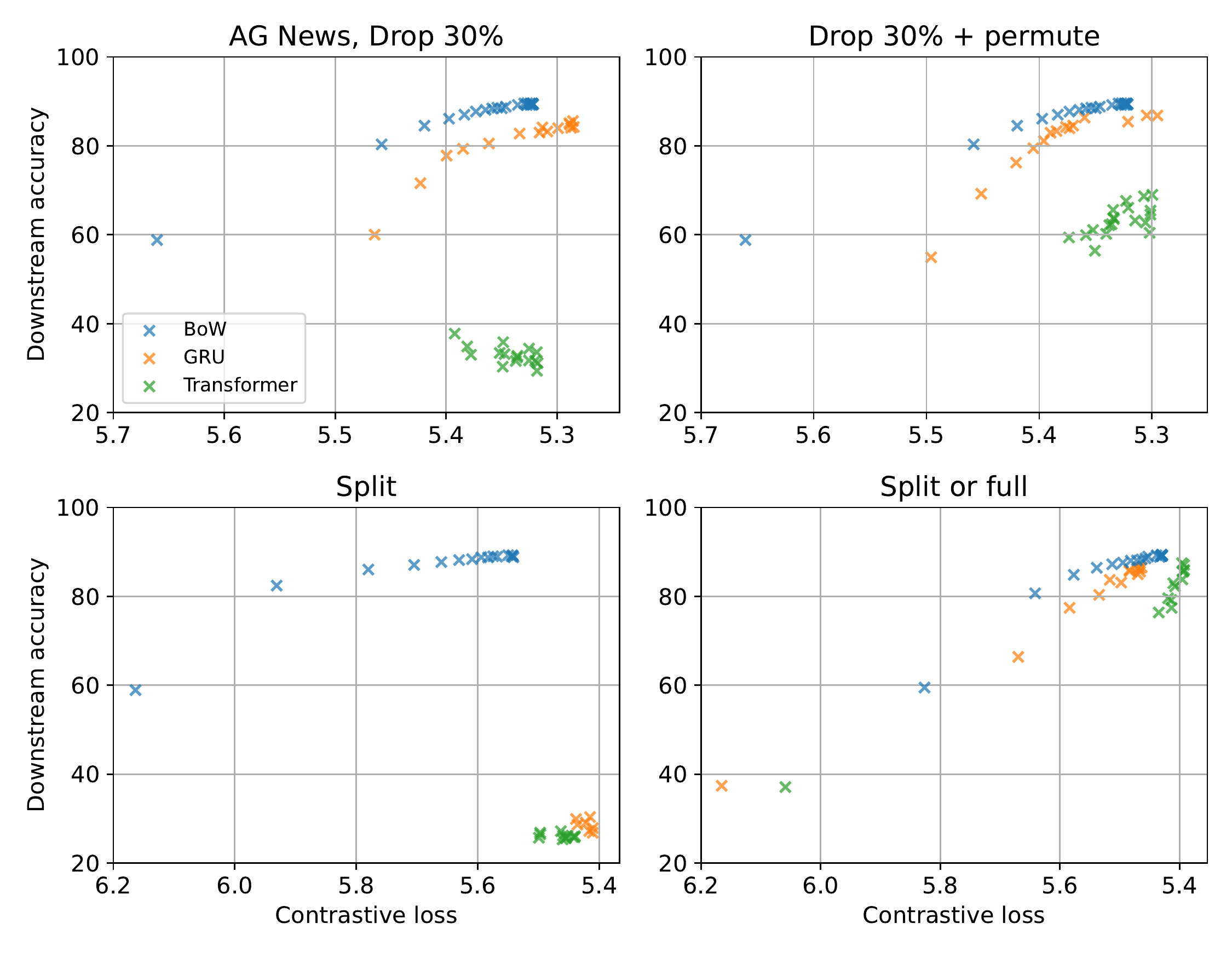}
        \vspace{-0.1in}
    \caption{Contrastive loss $\rightarrow$ accuracy transfer plots for AG News with bag-of-words ({\bow}), {\gru} and {\transformers} architectures with representation dimensionality $\dd=128$. Augmentations in each case are as follows: {\bf TL:} Drop random 30\% of tokens. {\bf TR:} Drop random 30\% of tokens and randomly permute the rest. {\bf BL:} Either the first half or second half of the input. {\bf BR:} Either the first half, second half or the full input. In all cases {\bow} representation makes the contrastive loss reasonably small and does quite well downstream ($\sim 90$\%), but either {\transformers} or both {\gru} and {\transformers} demonstrate brittleness of transfer for different augmentations.
    }
    \label{fig:agnews_scatter_plot}
\end{figure}
}{
\begin{figure}[t!]
    \centering
        \includegraphics[width=.9\linewidth]{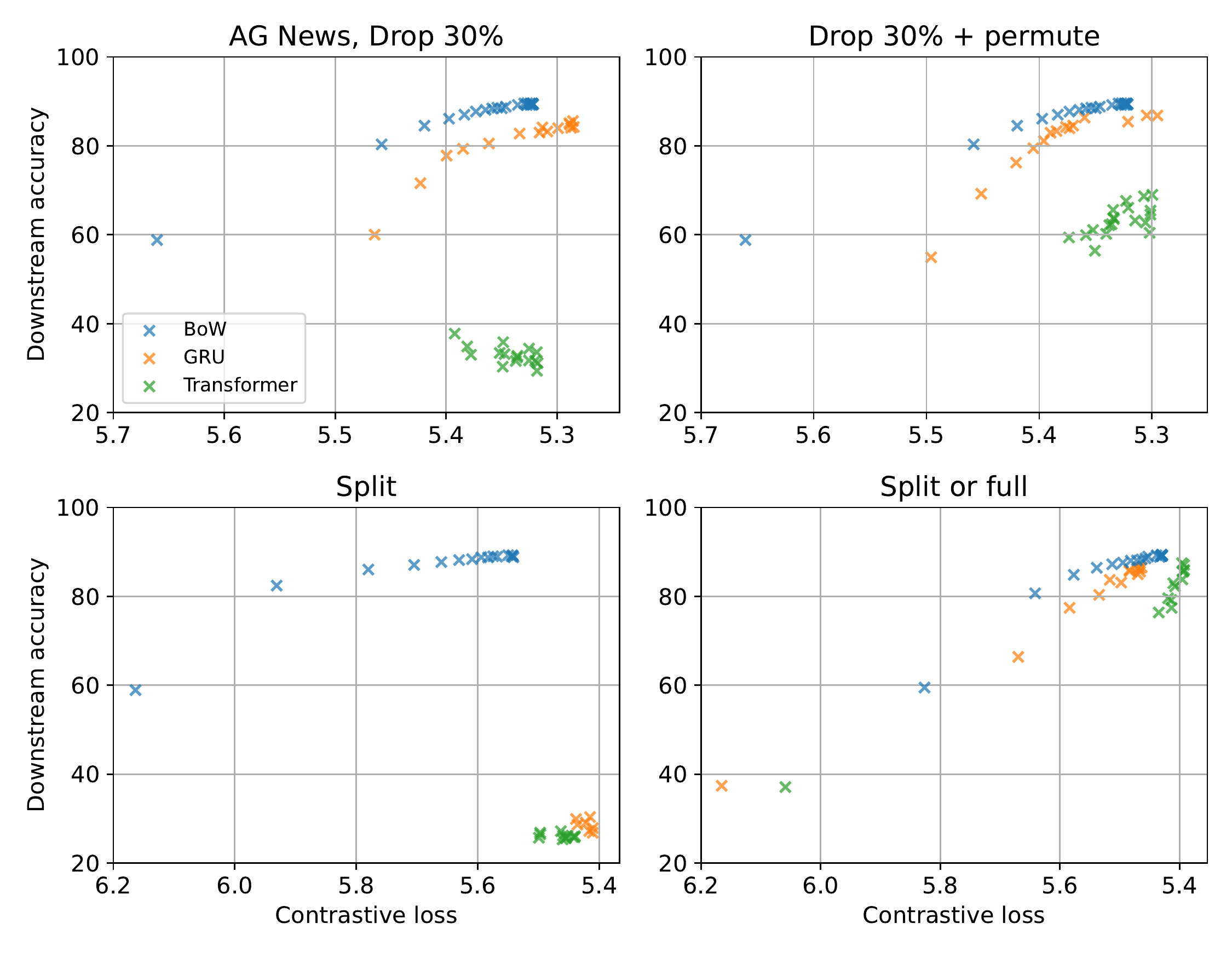}
        \vspace{-0.1in}
    \caption{Contrastive loss $\rightarrow$ accuracy transfer plots for AG News with bag-of-words ({\bow}), {\gru} and {\transformers} architectures. Augmentations in each case are as follows: {\bf TL:} Drop random 30\% of tokens. {\bf TR:} Drop random 30\% of tokens and randomly permute the rest. {\bf BL:} Either the first half or second half of the input. {\bf BR:} Either the first half, second half or the full input. In all cases {\bow} representation makes the contrastive loss reasonably small and does quite well downstream ($\sim 90$\%), but either {\transformers} or both {\gru} and {\transformers} demonstrate brittleness of transfer for different augmentations.
    }
    \label{fig:agnews_scatter_plot}
\end{figure}
}


\subsection{Are we in the disjoint augmentation regime?}
\label{sec:disjoint_regime}

Central to previous theory is the assumption that there exists overlap between augmentations distributions of data within a class.
We test this by setting up a classification task of predicting the image $\bar{x}$ that could have generated an augmentation $x$, similar to \Cref{defn:bayes_error_}.
The standard {\resnet}-18 architecture is modified to have $5000$ output classes, one for each image in a CIFAR-10 class.
We train the model to take augmentations from a fixed CIFAR-10 class and predict index of the original image generating it.
Performance is measured on {\em unseen} augmented data from inputs from the same class, by evaluating the accuracy of predicting the original input.
The results of these experiments in Figure~\ref{fig:vision_disjoint_plot} (averaged over the 10 classes) suggest that with extremely high accuracy, the trained model is able to identify the image given an augmentation, with accuracies higher than 99.5\% for different augmentation types.
This suggests that we may be closer to the disjoint augmentation setting than we think. 


\subsection{Experiments on a Text Domain}
\label{sec:nlp_exp}

In order to understand if our findings apply beyond images, we evaluate the contrastive pipeline on text domain.
We use the AG News classification dataset\footnote{We use the PyTorch torchtext library: \url{https://pytorch.org/text/stable/index.html}} \citep{zhang2015character}, where inputs are new articles and the 4 classes correspond to topics of the articles.
Compared to vision, there is relatively less study of augmentations for text.
Inspired by simple strategies like word/span deletion and word reordering \citep{wu2020clear,giorgi2021declutr,meng2021coco,yan2021consert}, we consider four simple augmentations strategies for our study: (i) \emph{Drop}: randomly drop 30\% tokens but keep the order of remaining tokens, (ii) \emph{Drop+Permute}: randomly drop 30\% tokens and randomly permute the remaining tokens, (iii) \emph{Split}: randomly return either the left half or right half of the text, and (iv) \emph{Split+Full}: randomly return from the full text, its left half, or its right half.

We run contrastive learning with three models on this dataset.
The first is a simple Bag-of-Word ($\bow$) model that learns a single word embedding matrix and returns the average word embedding of tokens in the text.
The second model is Gated Recurrent Unit ($\gru$), which is a recurrent neural network~\cite{chung2014empirical}.
The last is a $\transformers$~\citep{vaswani2017attention}, which is the base model for many state-of-the-art neural networks in NLP.
Both $\gru$ and $\transformers$ are selected to be unidirectional, and they map text to a sequence of hidden representations; we pick the representation for the final token as the text representation for contrastive learning and downstream evaluation.
We follow the SimCLR-like contrastive objective for training and linear classification for evaluating these models.
All models are trained from scratch to minimize just the contrastive loss, without the auxiliary MLM objective employed in some prior works \citep{wu2020clear,giorgi2021declutr,meng2021coco}.
See Appendix~\ref{sec:apx_exps} for details on the experimental setup and hyperparameters.

Figure~\ref{fig:agnews_scatter_plot} visualizes the training trajectories through the $(\Lcont(f), 1-\Lclf(f))$ space, i.e. the contrastive learning $\rightarrow$ downstream accuracy transfer plots.
We observe that for all augmentations, $\bow$ performs the best on downstream classification task, despite doing somewhat worse on the contrastive learning task.
For the drop augmentation (top left), the {\bow} and {\gru} plots might suggest that the augmentation is good; however the {\transformers} model leads to brittle transfer, i.e. it fails to solve the downstream task despite achieving very low contrastive learning loss.
This kind of difference in transfer performance is unexplained by existing function class agnostic theoretical guarantees.
Since the {\bow} representation is order invariant, we also test the augmentation that permutes tokens after dropping 30\% of them (top right).
This change does help the downstream accuracy of {\transformers}, however it does not completely bridge the gap.
While the split augmentation (bottom left) works for {\bow}, both {\gru} and {\transformers} display brittle transfer.
However a simple change of including the original text as an augmentation leads to both {\gru} and {\transformers} doing well downstream.
This is particularly surprising, since including the identity augmentation only decreases the probability of overlap between augmentations, a desirable property based on our current understanding of contrastive learning.
In \Cref{sec:robust_eval} we verify that this difference in performance is not just due to distribution shift (augmentations in contrastive learning v/s unaugmented inputs in downstream evaluation).

In \Cref{fig:agnews_visualize_reps_drop} we visualize two dimensional representations learned using contrastive learning, where it is evident that while the {\transformers} makes the representations invariant to augmentations, representations of augmentations from different classes look very similar to each in distribution and are thus not linear separable.
This phenomenon aligns with our lower bound \Cref{lem:unnormalized_lb,lem:normalized_lb}, whose proofs reveal how such spurious representations can be constructed.
The main takeaway is the for various augmentations, a weaker (less expressive) function class can succeed with weaker augmentations, while more expressive ones like {\gru} and {\transformers} might require stronger augmentations to transfer well to downstream tasks.
This phenomenon is not well understood by current theory and deserves more exploration.

\section{Conclusion}
\label{sec:conclusion}

Contrastive learning has emerged as a unifying paradigm for building flexible learners that can adapt to many tasks.
It is imperative to understand it better at a conceptual and mathematical level. The current paper lays out simple experiments and theoretical examples which suggest gaps in our current understanding.
Filling these gaps will require incorporating the inductive bias of the deep nets being used, which has primarily been studied in simplistic architectures (e.g., depth $2$ or $3$) so far.
The hypercube example from \Cref{sec:hypercube} and the behavior of simple architectures like MLPs is already an open problem.
Incorporating function class bias into transfer bounds is quite non-trivial and our results show how this can be done for linear representations.
Extending these results to more complex function classes, and incorporating training procedures could potentially give us new insights.
Our study in this paper has been diagnostic in nature: identifying gaps in our understanding. Converting these insights into algorithmic approaches is a very promising direction. 
We also hope that visualizations of contrastive loss $\rightarrow$ downstream performance can aid selection of more robust augmentations.

\clearpage
\bibliography{references}

\begin{thebibliography}{49}
\providecommand{\natexlab}[1]{#1}
\providecommand{\url}[1]{\texttt{#1}}
\expandafter\ifx\csname urlstyle\endcsname\relax
  \providecommand{\doi}[1]{doi: #1}\else
  \providecommand{\doi}{doi: \begingroup \urlstyle{rm}\Url}\fi

\bibitem[Abnar et~al.(2022)Abnar, Dehghani, Neyshabur, and
  Sedghi]{abnar2022exploring}
Samira Abnar, Mostafa Dehghani, Behnam Neyshabur, and Hanie Sedghi.
\newblock Exploring the limits of large scale pre-training.
\newblock In \emph{International Conference on Learning Representations}, 2022.

\bibitem[Arora et~al.(2019)Arora, Khandeparkar, Khodak, Plevrakis, and
  Saunshi]{arora2019theoretical}
Sanjeev Arora, Hrishikesh Khandeparkar, Mikhail Khodak, Orestis Plevrakis, and
  Nikunj Saunshi.
\newblock A theoretical analysis of contrastive unsupervised representation
  learning.
\newblock In \emph{Proceedings of the 36th International Conference on Machine
  Learning}, 2019.

\bibitem[Ash et~al.(2021)Ash, Goel, Krishnamurthy, and
  Misra]{ash2021investigating}
Jordan~T Ash, Surbhi Goel, Akshay Krishnamurthy, and Dipendra Misra.
\newblock Investigating the role of negatives in contrastive representation
  learning.
\newblock \emph{arXiv preprint arXiv:2106.09943}, 2021.

\bibitem[Bachman et~al.(2019)Bachman, Hjelm, and
  Buchwalter]{bachman2019learning}
Philip Bachman, R~Devon Hjelm, and William Buchwalter.
\newblock Learning representations by maximizing mutual information across
  views.
\newblock In \emph{Advances in Neural Information Processing Systems}, 2019.

\bibitem[Bansal et~al.(2021)Bansal, Kaplun, and Barak]{bansal2021for}
Yamini Bansal, Gal Kaplun, and Boaz Barak.
\newblock For self-supervised learning, rationality implies generalization,
  provably.
\newblock In \emph{International Conference on Learning Representations}, 2021.

\bibitem[Chen et~al.(2020)Chen, Kornblith, Norouzi, and Hinton]{chen2020simple}
Ting Chen, Simon Kornblith, Mohammad Norouzi, and Geoffrey Hinton.
\newblock A simple framework for contrastive learning of visual
  representations.
\newblock In \emph{International conference on machine learning}, 2020.

\bibitem[Chen et~al.(2021)Chen, Luo, and Li]{chen2021intriguing}
Ting Chen, Calvin Luo, and Lala Li.
\newblock Intriguing properties of contrastive losses.
\newblock \emph{Advances in Neural Information Processing Systems}, 2021.

\bibitem[Chen and He(2021)]{chen2021exploring}
Xinlei Chen and Kaiming He.
\newblock Exploring simple siamese representation learning.
\newblock In \emph{Proceedings of the IEEE/CVF Conference on Computer Vision
  and Pattern Recognition}, 2021.

\bibitem[Chung et~al.(2014)Chung, Gulcehre, Cho, and
  Bengio]{chung2014empirical}
Junyoung Chung, Caglar Gulcehre, KyungHyun Cho, and Yoshua Bengio.
\newblock Empirical evaluation of gated recurrent neural networks on sequence
  modeling.
\newblock \emph{arXiv preprint arXiv:1412.3555}, 2014.

\bibitem[Cover(1999)]{cover1999elements}
Thomas~M Cover.
\newblock \emph{Elements of information theory}.
\newblock John Wiley \& Sons, 1999.

\bibitem[Dosovitskiy et~al.(2014)Dosovitskiy, Springenberg, Riedmiller, and
  Brox]{dosovitskiy2014discriminative}
Alexey Dosovitskiy, Jost~Tobias Springenberg, Martin Riedmiller, and Thomas
  Brox.
\newblock Discriminative unsupervised feature learning with convolutional
  neural networks.
\newblock \emph{Advances in neural information processing systems}, 27, 2014.

\bibitem[Dosovitskiy et~al.(2021)Dosovitskiy, Beyer, Kolesnikov, Weissenborn,
  Zhai, Unterthiner, Dehghani, Minderer, Heigold, Gelly, Uszkoreit, and
  Houlsby]{dosovitskiy2021an}
Alexey Dosovitskiy, Lucas Beyer, Alexander Kolesnikov, Dirk Weissenborn,
  Xiaohua Zhai, Thomas Unterthiner, Mostafa Dehghani, Matthias Minderer, Georg
  Heigold, Sylvain Gelly, Jakob Uszkoreit, and Neil Houlsby.
\newblock An image is worth 16x16 words: Transformers for image recognition at
  scale.
\newblock In \emph{International Conference on Learning Representations}, 2021.

\bibitem[Gao et~al.(2021)Gao, Yao, and Chen]{gao2021simcse}
Tianyu Gao, Xingcheng Yao, and Danqi Chen.
\newblock {S}im{CSE}: Simple contrastive learning of sentence embeddings.
\newblock In \emph{Proceedings of the 2021 Conference on Empirical Methods in
  Natural Language Processing}, 2021.

\bibitem[Giorgi et~al.(2021)Giorgi, Nitski, Wang, and Bader]{giorgi2021declutr}
John Giorgi, Osvald Nitski, Bo~Wang, and Gary Bader.
\newblock {D}e{CLUTR}: Deep contrastive learning for unsupervised textual
  representations.
\newblock In \emph{Proceedings of the 59th Annual Meeting of the Association
  for Computational Linguistics and the 11th International Joint Conference on
  Natural Language Processing (Volume 1: Long Papers)}, 2021.

\bibitem[Grill et~al.(2020)Grill, Strub, Altch{\'e}, Tallec, Richemond,
  Buchatskaya, Doersch, Pires, Guo, Azar, et~al.]{grill2020bootstrap}
Jean-Bastien Grill, Florian Strub, Florent Altch{\'e}, Corentin Tallec,
  Pierre~H Richemond, Elena Buchatskaya, Carl Doersch, Bernardo~Avila Pires,
  Zhaohan~Daniel Guo, Mohammad~Gheshlaghi Azar, et~al.
\newblock Bootstrap your own latent: A new approach to self-supervised
  learning.
\newblock 2020.

\bibitem[Gutmann and Hyv{\"a}rinen(2010)]{gutmann2010noise}
Michael Gutmann and Aapo Hyv{\"a}rinen.
\newblock Noise-contrastive estimation: A new estimation principle for
  unnormalized statistical models.
\newblock In \emph{Proceedings of the Thirteenth International Conference on
  Artificial Intelligence and Statistics}, 2010.

\bibitem[HaoChen et~al.(2021)HaoChen, Wei, Gaidon, and Ma]{haochen2021provable}
Jeff~Z HaoChen, Colin Wei, Adrien Gaidon, and Tengyu Ma.
\newblock Provable guarantees for self-supervised deep learning with spectral
  contrastive loss.
\newblock \emph{arXiv preprint arXiv:2106.04156}, 2021.

\bibitem[He et~al.(2016)He, Zhang, Ren, and Sun]{he2016deep}
Kaiming He, Xiangyu Zhang, Shaoqing Ren, and Jian Sun.
\newblock Deep residual learning for image recognition.
\newblock In \emph{Proceedings of the IEEE conference on computer vision and
  pattern recognition}, 2016.

\bibitem[Hjelm et~al.(2018)Hjelm, Fedorov, Lavoie-Marchildon, Grewal, Bachman,
  Trischler, and Bengio]{hjelm2018learning}
R~Devon Hjelm, Alex Fedorov, Samuel Lavoie-Marchildon, Karan Grewal, Phil
  Bachman, Adam Trischler, and Yoshua Bengio.
\newblock Learning deep representations by mutual information estimation and
  maximization.
\newblock \emph{arXiv preprint arXiv:1808.06670}, 2018.

\bibitem[Jin et~al.(2017)Jin, Ge, Netrapalli, Kakade, and
  Jordan]{jin2017escape}
Chi Jin, Rong Ge, Praneeth Netrapalli, Sham~M Kakade, and Michael~I Jordan.
\newblock How to escape saddle points efficiently.
\newblock In \emph{International Conference on Machine Learning}, 2017.

\bibitem[Jing et~al.(2021)Jing, Vincent, LeCun, and
  Tian]{jing2021understanding}
Li~Jing, Pascal Vincent, Yann LeCun, and Yuandong Tian.
\newblock Understanding dimensional collapse in contrastive self-supervised
  learning.
\newblock \emph{arXiv preprint arXiv:2110.09348}, 2021.

\bibitem[Lee et~al.(2021)Lee, Lei, Saunshi, and Zhuo]{lee2020predicting}
Jason~D. Lee, Qi~Lei, Nikunj Saunshi, and Jiacheng Zhuo.
\newblock Predicting what you already know helps: provable self-supervised
  learning.
\newblock \emph{Advances in Neural Information Processing Systems}, 2021.

\bibitem[Logeswaran and Lee(2018)]{logeswaran2018efficient}
Lajanugen Logeswaran and Honglak Lee.
\newblock An efficient framework for learning sentence representations.
\newblock In \emph{Proceedings of the International Conference on Learning
  Representations}, 2018.

\bibitem[Meng et~al.(2021)Meng, Xiong, Bajaj, Bennett, Han, Song,
  et~al.]{meng2021coco}
Yu~Meng, Chenyan Xiong, Payal Bajaj, Paul Bennett, Jiawei Han, Xia Song, et~al.
\newblock Coco-lm: Correcting and contrasting text sequences for language model
  pretraining.
\newblock \emph{Advances in Neural Information Processing Systems}, 2021.

\bibitem[Mitrovic et~al.(2021)Mitrovic, McWilliams, Walker, Buesing, and
  Blundell]{mitrovic2021representation}
Jovana Mitrovic, Brian McWilliams, Jacob Walker, Lars Buesing, and Charles
  Blundell.
\newblock Representation learning via invariant causal mechanisms.
\newblock In \emph{International Conference on Learning Representations}, 2021.

\bibitem[Nozawa and Sato(2021)]{nozawa2021understanding}
Kento Nozawa and Issei Sato.
\newblock Understanding negative samples in instance discriminative
  self-supervised representation learning.
\newblock \emph{arXiv preprint arXiv:2102.06866}, 2021.

\bibitem[Oord et~al.(2018)Oord, Li, and Vinyals]{oord2018representation}
Aaron van~den Oord, Yazhe Li, and Oriol Vinyals.
\newblock Representation learning with contrastive predictive coding.
\newblock \emph{arXiv preprint arXiv:1807.03748}, 2018.

\bibitem[Robinson et~al.(2021)Robinson, Sun, Yu, Batmanghelich, Jegelka, and
  Sra]{robinson2021can}
Joshua Robinson, Li~Sun, Ke~Yu, Kayhan Batmanghelich, Stefanie Jegelka, and
  Suvrit Sra.
\newblock Can contrastive learning avoid shortcut solutions?
\newblock In \emph{Advances in Neural Information Processing Systems}, 2021.

\bibitem[Saunshi et~al.(2021)Saunshi, Malladi, and
  Arora]{saunshi2021mathematical}
Nikunj Saunshi, Sadhika Malladi, and Sanjeev Arora.
\newblock A mathematical exploration of why language models help solve
  downstream tasks.
\newblock In \emph{International Conference on Learning Representations}, 2021.

\bibitem[Tian et~al.(2019)Tian, Krishnan, and Isola]{tian2019contrastive}
Yonglong Tian, Dilip Krishnan, and Phillip Isola.
\newblock Contrastive multiview coding.
\newblock \emph{arXiv preprint arXiv:1906.05849}, 2019.

\bibitem[Tian et~al.(2020)Tian, Sun, Poole, Krishnan, Schmid, and
  Isola]{tian2020what}
Yonglong Tian, Chen Sun, Ben Poole, Dilip Krishnan, Cordelia Schmid, and
  Phillip Isola.
\newblock What makes for good views for contrastive learning?
\newblock In \emph{Advances in Neural Information Processing Systems}, 2020.

\bibitem[Tolstikhin et~al.(2021)Tolstikhin, Houlsby, Kolesnikov, Beyer, Zhai,
  Unterthiner, Yung, Steiner, Keysers, Uszkoreit, et~al.]{tolstikhin2021mlp}
Ilya~O Tolstikhin, Neil Houlsby, Alexander Kolesnikov, Lucas Beyer, Xiaohua
  Zhai, Thomas Unterthiner, Jessica Yung, Andreas Steiner, Daniel Keysers,
  Jakob Uszkoreit, et~al.
\newblock Mlp-mixer: An all-mlp architecture for vision.
\newblock \emph{Advances in Neural Information Processing Systems}, 2021.

\bibitem[Tosh et~al.(2021{\natexlab{a}})Tosh, Krishnamurthy, and
  Hsu]{tosh2020contrastive}
Christopher Tosh, Akshay Krishnamurthy, and Daniel Hsu.
\newblock Contrastive learning, multi-view redundancy, and linear models.
\newblock 2021{\natexlab{a}}.

\bibitem[Tosh et~al.(2021{\natexlab{b}})Tosh, Krishnamurthy, and
  Hsu]{tosh2020contrastive0}
Christopher Tosh, Akshay Krishnamurthy, and Daniel Hsu.
\newblock Contrastive estimation reveals topic posterior information to linear
  models.
\newblock \emph{Journal of Machine Learning Research}, 2021{\natexlab{b}}.

\bibitem[Tsai et~al.(2021)Tsai, Wu, Salakhutdinov, and
  Morency]{tsai2021selfsupervised}
Yao-Hung~Hubert Tsai, Yue Wu, Ruslan Salakhutdinov, and Louis-Philippe Morency.
\newblock Self-supervised learning from a multi-view perspective.
\newblock In \emph{International Conference on Learning Representations}, 2021.

\bibitem[Tschannen et~al.(2019)Tschannen, Djolonga, Rubenstein, Gelly, and
  Lucic]{tschannen2019mutual}
Michael Tschannen, Josip Djolonga, Paul~K Rubenstein, Sylvain Gelly, and Mario
  Lucic.
\newblock On mutual information maximization for representation learning.
\newblock \emph{arXiv preprint arXiv:1907.13625}, 2019.

\bibitem[Vaswani et~al.(2017)Vaswani, Shazeer, Parmar, Uszkoreit, Jones, Gomez,
  Kaiser, and Polosukhin]{vaswani2017attention}
Ashish Vaswani, Noam Shazeer, Niki Parmar, Jakob Uszkoreit, Llion Jones,
  Aidan~N Gomez, \L~ukasz Kaiser, and Illia Polosukhin.
\newblock Attention is all you need.
\newblock In \emph{Advances in Neural Information Processing Systems}, 2017.

\bibitem[Von~K{\"u}gelgen et~al.(2021)Von~K{\"u}gelgen, Sharma, Gresele,
  Brendel, Sch{\"o}lkopf, Besserve, and Locatello]{von2021selfsupervised}
Julius Von~K{\"u}gelgen, Yash Sharma, Luigi Gresele, Wieland Brendel, Bernhard
  Sch{\"o}lkopf, Michel Besserve, and Francesco Locatello.
\newblock Self-supervised learning with data augmentations provably isolates
  content from style.
\newblock In \emph{Advances in Neural Information Processing Systems}, 2021.

\bibitem[Wang and Isola(2020)]{wang2020understanding}
Tongzhou Wang and Phillip Isola.
\newblock Understanding contrastive representation learning through alignment
  and uniformity on the hypersphere.
\newblock \emph{arXiv preprint arXiv:2005.10242}, 2020.

\bibitem[Wang and Gupta(2015)]{wang2015unsupervised}
Xiaolong Wang and Abhinav Gupta.
\newblock Unsupervised learning of visual representations using videos.
\newblock In \emph{Proceedings of the IEEE International Conference on Computer
  Vision}, 2015.

\bibitem[Wang et~al.(2022)Wang, Zhang, Wang, Yang, and Lin]{wang2022chaos}
Yifei Wang, Qi~Zhang, Yisen Wang, Jiansheng Yang, and Zhouchen Lin.
\newblock Chaos is a ladder: A new understanding of contrastive learning.
\newblock In \emph{International Conference on Learning Representations}, 2022.

\bibitem[Wen and Li(2021)]{wen2021toward}
Zixin Wen and Yuanzhi Li.
\newblock Toward understanding the feature learning process of self-supervised
  contrastive learning.
\newblock 2021.

\bibitem[Wu et~al.(2020{\natexlab{a}})Wu, Zhang, and
  R{\'e}]{wu2020Understanding}
Sen Wu, Hongyang Zhang, and Christopher R{\'e}.
\newblock Understanding and improving information transfer in multi-task
  learning.
\newblock In \emph{International Conference on Learning Representations},
  2020{\natexlab{a}}.

\bibitem[Wu et~al.(2018)Wu, Xiong, Yu, and Lin]{wu2018unsupervised}
Zhirong Wu, Yuanjun Xiong, Stella~X Yu, and Dahua Lin.
\newblock Unsupervised feature learning via non-parametric instance
  discrimination.
\newblock In \emph{Proceedings of the IEEE conference on computer vision and
  pattern recognition}, pages 3733--3742, 2018.

\bibitem[Wu et~al.(2020{\natexlab{b}})Wu, Wang, Gu, Khabsa, Sun, and
  Ma]{wu2020clear}
Zhuofeng Wu, Sinong Wang, Jiatao Gu, Madian Khabsa, Fei Sun, and Hao Ma.
\newblock Clear: Contrastive learning for sentence representation.
\newblock \emph{arXiv preprint arXiv:2012.15466}, 2020{\natexlab{b}}.

\bibitem[Yan et~al.(2021)Yan, Li, Wang, Zhang, Wu, and Xu]{yan2021consert}
Yuanmeng Yan, Rumei Li, Sirui Wang, Fuzheng Zhang, Wei Wu, and Weiran Xu.
\newblock {C}on{SERT}: A contrastive framework for self-supervised sentence
  representation transfer.
\newblock In \emph{Proceedings of the 59th Annual Meeting of the Association
  for Computational Linguistics and the 11th International Joint Conference on
  Natural Language Processing (Volume 1: Long Papers)}, 2021.

\bibitem[Zhang et~al.(2017)Zhang, Bengio, Hardt, Recht, and
  Vinyals]{zhang2017understanding}
Chiyuan Zhang, Samy Bengio, Moritz Hardt, Benjamin Recht, and Oriol Vinyals.
\newblock Understanding deep learning requires rethinking generalization.
\newblock In \emph{International Conference on Learning Representations}, 2017.

\bibitem[Zhang et~al.(2015)Zhang, Zhao, and LeCun]{zhang2015character}
Xiang Zhang, Junbo Zhao, and Yann LeCun.
\newblock Character-level convolutional networks for text classification.
\newblock \emph{Advances in neural information processing systems}, 2015.

\bibitem[Zimmermann et~al.(2021)Zimmermann, Sharma, Schneider, Bethge, and
  Brendel]{zimmermann2021contrastive}
Roland~S Zimmermann, Yash Sharma, Steffen Schneider, Matthias Bethge, and
  Wieland Brendel.
\newblock Contrastive learning inverts the data generating process.
\newblock 2021.

\end{thebibliography}
\bibliographystyle{plainnat}

\clearpage
\appendix
\onecolumn



\section{Omitted Proofs}
\subsection{Proof of Proposition \ref{prop:hypercube_results}}
\label{sec:apx_proof_hypercube}
The proof of $(a)$ follows directly from \Cref{lem:unnormalized_lb}, since all the conditions are satisfied and the augmentations are disjoint.

For the eigenvalues, we can compute the covariances by using $\tau\sim\gU((0,1])$
\begin{align*}
    \Sigma(\phi) 
    &= \diag\left(\bm{1}_{k}, \ex_{\tau}[\tau^{2}] \bm{1}_{\DD-k}\right) =\diag(\bm{1}_{k}, \nicefrac{1}{3} \bm{1}_{\DD-k})\\
    \Sigma(\aug{\phi}) 
    &= \diag\left(\bm{1}_{k}, (\ex_{\tau}[\tau])^{2} \bm{1}_{\DD-k}\right) =\diag(\bm{1}_{k}, \nicefrac{1}{4} \bm{1}_{\DD-k})
\end{align*}

Thus the matrix of interest is \begin{align*}
    I_{\DD} - \Sigma(\phi)^{-\half}\Sigma(\aug{\phi})\Sigma(\phi)^{-\half}
    &= \diag\left(\bm{0}_{k}, \nicefrac{1}{4}\bm{1}_{\DD-k}\right)
\end{align*}
giving us $\lambda_{i}=0$ for $i\le k$ and $\lambda_{i}=\nicefrac{1}{4}$ for $i>k$.
Plugging into \Cref{thm:low_rank_result} for $\dd'=k$ finishes the proof.
\section{Low-rank Linear Representation Proof}
\label{sec:apx_low_rank}

\begin{table}
\begin{center}
    \caption{Notations\vspace{0.1in}}
    \label{table:notation}
    \begin{tabular}{c c l}
    \toprule
      Notation & Definition & Description \\
    \toprule
     & \underline{Distributions} & \\
      $\Xorig,\Xaug$ &  & Set of inputs and augmentations\\
      $\gA$ & $x\sim \gA(\cdot \mid \bar{x})$ & Augmentation distribution \\
      $\Dorig$ & $\bar{x}\sim\Dorig$ & Marginal distribution on inputs $\Xorig$ \\
      $\Daug$ & $\ex_{\bar{x}}\left[\gA(\cdot\mid \bar{x})\right]$ & Marginal distribution on augmentations $\Xaug$ \\
      $\Dbar\in\R^{\Xorig \times \Xorig}$ & $\Dbar[\bar{x}, \bar{x}] = \Dorig(\bar{x})$ & Matrix of marginal distributions on $\Xorig$\\
      $D\in\R^{\Xaug \times \Xaug}$ & $D[x, x] = \Daug(x)$ & Matrix of marginal distributions on $\Xaug$\\
      $\Abar\in\R^{\Xorig \times \Xaug}$ & $\Abar[\bar{x},x] = \Dorig(\bar{x}) \gA(x \mid \bar{x})$ & Input augmentation distribution \\
      $\Abarn\in\R^{\Xorig \times \Xaug}$ & $\Dbar^{-\half} \Abar D^{-\half}$ & Normalized matrix version of $\Abar$\\
      $\A\in\R^{\Xaug \times \Xaug}$ & $\A[x,x] = \Dsim(x, x)$ & Matrix of joint distribution of augmentations\\
      $\An\in\R^{\Xaug \times \Xaug}$ & $D^{-\half} \A D^{-\half} = \Abarn^{\top} \Abarn$ & Normalized matrix version of $\A$\\
    \toprule
     & \underline{Fixed features} & \\
      $\phi:\Xaug\rightarrow\R^{\DD}$ & & Fixed feature map for augmentations\\
      $\aug{\phi}:\Xorig\rightarrow\R^{\DD}$ & $\ex_{x \sim \gA(x \mid \cdot)}\left[\phi(x)\right]$ & Augmentation averaged feature\\
      $\Sigma(\phi) \in \R^{\DD\times \DD}$ & $\ex_{x}\left[\phi(x)\phi(x)^{\top}\right]$ & Covariance of feature map $\phi$\\
      $\Phi\in\R^{\Xaug\times \DD}$  & $\Phi[x] = \phi(x)$ & Matrix version of feature map $\phi$ \\
      $\Phin\in\R^{\Xaug\times \DD}$  & $D^{\half} \Phi$ & Normalized version of $\Phi$ \\
    \toprule
     & \underline{Representation} & \\
      $f: \Xaug \rightarrow \R^{\dd}$ & & Representation function \\
      $\fn: \Xaug \rightarrow \R^{\dd}$ & $\sqrt{\Daug(\cdot)} f(\cdot)$ & Normalized version of $f$ \\
      $F\in \R^{\Xaug \times \dd}$ & $F[x, i] = f(x)_{i}$ & Matrix version of $f$ \\
      $\Fn\in \R^{\Xaug \times \dd}$ & $D^{\half} F$ & Normalized version of $F$ \\
      %
    \toprule
     & \underline{Function classes} & \\
      $\gF\subseteq \{f: \Xaug \rightarrow \R^{\dd}\}$ & & Representation function class \\
      $\gF_{\phi} \subseteq  \{f: \Xaug \rightarrow \R^{\dd}\}$ & $\left\{W^{\top} \phi(\cdot) \mid W \in \R^{\DD \times \dd}\right\}$ & Linear representation class \\
      $\gF_{\Phi} \subseteq  \R^{\Xaug \times \dd}$ & $\left\{\Phi W \mid W \in \R^{\DD \times \dd}\right\}$ & Linear representation class (matrix version) \\
    \bottomrule
    \end{tabular}
\end{center}
\end{table}


Firstly we set up some notation. For sets $P$ and set $Q$, where $P$ is finite, we denote $Q^{P}$ to denote the set of all functions from $P \rightarrow Q$.
We abuse notation and also denote $Q^{P}$ to be a subset of $Q^{|P|}$, where an element $r \in Q^{P}$ is a vector of $|P|$ dimensions and coordinates are indexed by elements of $P$.
For instance, when $Q = \{\pm1\}$ and $P$ is finite, $Q^{P} = \{\pm 1\}^{P}$ denotes all functions mapping elements in $P$ to either $1$ or $-1$.
Furthermore, $r \in \{\pm 1\}^{P}$ denotes a vector in $\{\pm 1\}^{|P|}$ that looks like $(r(p))_{p\in P}$.
Similarly we denote $Q^{P \times R}$ to denote a matrix in $Q^{|P| \times |R|}$.
For a matrix $Q \in \R^{m \times n}$, $Q_{:\dd}\in\R^{}$

We now prove function class dependent guarantees for the class of linear representations.
As in \Cref{sec:upper_bound}, for a feature map $\phi : \Xaug \rightarrow \R^{\DD}$, we define the linear representation class $\gF_{\phi} = \left\{f(\cdot) = W^{\top} \phi(\cdot) \mid W \in \R^{\DD \times \DD}\right\}$.
We wish to show downstream guarantees for contrastive learning that depend not only on the contrastive loss of a representation $f:\Xaug \rightarrow \R^{\dd}$ but also uses the fact that it belongs to the class $\gF_{\phi}$.
In particular, we desire a bound that looks like $\Lclf(f; \gstar) \le \gT(\gA, \gstar, \Lcont(f), \gF_{\phi})$, as described in \Cref{eqn:fn_class_dep_bound}.

We employ the strategy from \citet{haochen2021provable} and show guarantees for the spectral contrastive loss, defined in \Cref{eqn:spectral_loss} as
\begin{align}
    \Lspec(f)= -2 \ex_{\substack{(x,x^{+})\sim\Dsim}}\left[f(x)^{\top}f(x^{+})\right] + \ex_{\substack{x,x^{-}\sim\Dneg^{2}}}\left[ \left(f(x)^{\top}f(x^{-})\right)^{2}\right]
\end{align}
We first provide a sketch of their proof in our notation and highlight the main steps.
Our result is similar in spirit to theirs, but deviates at crucial junctions due to incorporation of the function class.
\paragraph{1. Rewrite as matrix factorization.}
Lemma 3.2 from \citet{haochen2021provable} shows that this objective can be rewritten as matrix factorization.
For any two augmentations $x,x'\in\Xaug$, define $w_{x,x'} = \Dsim(x,x') = \ex_{\bar{x}}\left[\gA(x\mid \bar{x}) \gA(x\mid \bar{x})\right]$ to be the probability that $x$ and $x'$ appear as a similar pair, i.e. two augmentations of the same input.
Let $w_{x} = \sum_{x'\in\Xaug} w_{x,x'} = \Dneg(x)$ be the marginal probability.
Then the objective can be rewritten as follows:
\begin{align*}
    \Lspec(f)
    &= \ex_{\substack{(x,x^{+})\sim\Dsim}}\left[f(x)^{\top}f(x^{+})\right] + \ex_{\substack{x,x^{-}\sim\Dneg^{2}}}\left[ \left(f(x)^{\top}f(x^{-})\right)^{2}\right]\\
    &= -2\sum_{x,x^{+}\in\Xaug} w_{x,x^{+}} f(x)^{\top} f(x^{+}) + \sum_{x,x^{-}\in\Xaug} w_{x} w_{x^{-}} \left(f(x)^{\top} f(x^{-})\right)^{2}\\
    &= \sum_{x,x'\in\Xaug} \left(-2 w_{x,x'} f(x)^{\top} f(x') + w_{x} w_{x^{'}} \left(f(x)^{\top} f(x')\right)^{2}\right)\\
    &= C + \sum_{x,x' \in \Xaug} \left(\frac{w_{x,x'}}{\sqrt{w_{x} w_{x'}}} - \left(\sqrt{w_{x}}f(x)\right)^{\top}\left(\sqrt{w_{x'}}f(x')\right)\right)^{2}
\end{align*}
where $C$ depends only on $w$ and thus only on $\gA$, but not $f$.
Thus $\Lspec(f)$ can be interpreted as a matrix factorization objective, with the matrix being $\An \in \R^{\Xaug \times \Xaug}$ such that $\An[x,x'] = \frac{w_{x,x'}}{\sqrt{w_{x} w_{x'}}}$ and scaled version of representation $u_{x} = \sqrt{w_{x}} f(x)$ is being used to factorize this.
Note that $\An$ only depends on $w$'s which in turn only depend on the distributions $\gA$, $\Dorig$ and $\Daug$.
We stack the representation $f$ into a matrix $\Fn \in \R^{\Xaug \times \dd}$, where the column corresponding to $x\in\Xaug$ is $\Fn[x] = \sqrt{w_{x}} f(x)$.
Then the objective can be written as 
\begin{align}
    \Lspec(f) = C + \|\An - \Fn\Fn^{\top}\|^{2}_{F}.
\end{align}
This helps characterize the optimal solution $f^{\star}$ of the contrastive objective, which corresponds to the matrix $\Fn^{\star}$ learning the top $\dd$ eigen-directions of the matrix $\An$.
Inspired by this analysis, we also show that the spectral loss with the function class $\gF_{\phi}$ is a matrix factorization problem, but for a different matrix that depends on both $\An$ and $\phi$.

\paragraph{2. $\epsilon$-optimal solution $f$}
While the above characterization tells us something about the optimal representation $f^{\star}$, in general we might have a representation that has sub-optimality of $\epsilon = \Lspec(f) - \Lspec(f^{\star})$.
In this case, it can be argued that such a representation captures significant mass of the first $\dd$ eigen-directions of $\An$ as long as $\epsilon$ is small and the eigen-gap is large.
More specifically, if $\gamma_{1}, \dots, \gamma_{\Xaug}$ denote the eigenvalues of $\An$, then the suboptimal $f$ will capture all except $\gO\left(\frac{\epsilon}{\left(\gamma_{\dd+1} - \gamma_{\dd}\right)^{2}}\right) = \gO\left(\frac{\Lspec(f) - \inf_{f^{\star}}\Lspec(f^{\star})}{\left(\gamma_{\dd+1} - \gamma_{\dd}\right)^{2}}\right)$ mass of the first $\dd$ eigen-directions of $\An$.
For our analysis, we will suffer a suboptimality only w.r.t. the function class $\gF_{\phi}$, i.e. the $\epsilon$ will be $\Lspec(f) - \inf\limits_{f^{\star}\in\gF_{\phi}}\Lspec(f^{\star})$ rather than $\Lspec(f) - \inf_{f^{\star}}\Lspec(f^{\star})$

\paragraph{3. Connecting to downstream.}
It remains to show why approximately learning the top $\dd$ directions of the augmentation matrix $\An$ can help with a downstream task $\gstar$.
This step uses two assumptions, (1) there is sufficient overlap in augmentation distributions overall, and (2) augmentations are approximately label invariant, i.e. there is not much overlap in augmentations of inputs from different classes.
These assumptions imply that the true label vector $\gstaraug\in\{\pm1\}^{\Xaug}$ has a high component on the first $\dd$ directions of $\An$.
We use similar properties but with less stringent conditions on the amount of overlap between augmentations. In addition to this we need a crucial assumption that the function class $\gF_{\phi}$ is expressive enough to solve the classification task on augmentations.

\subsection{Matrix notation}
\label{sec:apx_matrix_notation}

Given the backdrop of the results from \citet{haochen2021provable}, we now presentation the matrix notations for various functions that will be helpful to prove our main result.
All definitions and notations are summarized in \Cref{table:notation}.

\subsubsection{Distributions to matrices}

Let $w_{\bar{x}} = \Dorig(x)$ denote the marginal probabilities of input $\bar{x}\in\Xorig$ and $w_{\bar{x}, x} = \gA(x \mid \bar{x}) w_{\bar{x}}$ denote the joint probability of input and augmentation.
The marginal for augmentations can then be defined as $w_{x} = \Daug(x) = \sum_{\bar{x}} w_{\bar{x}, x}$.
To summarize
\begin{align}
    w_{\bar{x}} &= \Dorig(x)\label{eqn:w_xbar}\\
    w_{x\mid \bar{x}} &= \gA(x \mid \bar{x})\label{eqn:w_x_given_xbar}\\
    w_{x, \bar{x}} &= w_{\bar{x}, x} = \gA(x \mid \bar{x}) w_{\bar{x}} = w_{x \mid \bar{x}} w_{\bar{x}}\label{eqn:w_xbar_x}\\
    w_{x} &= \Daug(x)\label{eqn:w_x}
\end{align}

Let $\Dbar\in\R^{\Xorig \times \Xorig}$ denote a diagonal matrix of marginal probabilities, i.e. $\Dbar = \text{diag}((w_{\bar{x}})_{\bar{x}\in\Xorig})$.
Similarly $D\in\R^{\Xaug \times \Xaug}$ is the diagonal matrix of augmentation marginals.
Thus these diagonal matrices satisfy
\begin{align}
\label{eqn:D_Dbar}
    \Dbar[\bar{x}, \bar{x}] = w_{\bar{x}}, ~
    D[x, x] = w_{x}
\end{align}
We express the augmentation distributions $\gA(\cdot \mid \bar{x})_{\bar{x} \in \Xorig}$ as a matrix $\Abar \in \R^{\Xorig \times \Xaug}$, where $\Abar[\bar{x}, x] = w_{\bar{x}, x}$.
A normalized version of $\Abar$ is denoted by $\Abarn\in\R^{\Xorig \times \Xaug}$ and defined as $\Abarn[\bar{x}, x] = \frac{w_{\bar{x}, x}}{\sqrt{w_{\bar{x}} w_{x}}}$.
We summarize these definitions below, along with a matrix equation that follows easily from the definition
\begin{align}
\label{eqn:Abarn}
    \Abar[\bar{x}, x] = w_{\bar{x}, x}, ~
    \Abarn[\bar{x}, x] = \frac{w_{\bar{x}, x}}{\sqrt{w_{\bar{x}} w_{x}}}, ~
    \Abarn = \Dbar^{-\half} \Abar D^{-\half}
\end{align}

For the similarity distribution $\Dsim$ on pairs of augmentations, define the following
\begin{align}
\label{eqn:w_x_xp}
    w_{x,x'} = \Dsim(x,x') = \ex_{\bar{x}} [\gA(x \mid \bar{x}) \gA(x' \mid \bar{x})] = \sum_{\bar{x}} w_{\bar{x}} w_{x \mid \bar{x}} w_{x' \mid \bar{x}}.
\end{align}
$\Dsim$ is expressed as a matrix $\A\in\R^{\Xaug \times \Xaug}$, where $\A[x, x'] = w_{x,x'}$.
The normalized version of $\A$ is defined as $\An\in\R^{\Xaug \times \Xaug}$, where $\An[x, x'] = \frac{w_{x,x'}}{\sqrt{w_{x} w_{x'}}}$.
We summarize these definitions below, along with a matrix equation that follows easily from the definition
\begin{align}
\label{eqn:An}
    \A[x, x'] = w_{x,x'}, ~
    \An[x, x'] = \frac{w_{x,x'}}{\sqrt{w_{x} w_{x'}}}, ~
    \An = D^{-\half} \A D^{-\half}
\end{align}

The following lemma connects the $\Abarn$ and $\An$
\begin{lemma}
\label{lem:Abar_to_A}
    For $\Abarn$ and $\An$ defined in \Cref{table:notation}, we have the following
    \begin{align}
        \An = \Abarn^{\top} \Abarn
    \end{align}
\end{lemma}
\begin{proof}
    Firstly from \Cref{eqn:Abarn}, we get that $\Abarn^{\top} \Abarn = D^{-\half} \Abar^{\top} \Dbar^{-1} \Abar D^{-\half}$.
    Given that $\An = D^{-\half} \A D^{-\half}$ from \Cref{eqn:An}, it suffices to show that $\A = \Abar^{\top} \Dbar^{-1} \Abar$.
    The $(x, x')$ entry of the RHS is as follows
    \begin{align*}
        \left(\Abar^{\top} \Dbar^{-1} \Abar\right)[x, x']
        &= \sum_{\bar{x}} \frac{w_{\bar{x}, x} ~w_{\bar{x}, x'}}{w_{\bar{x}}}
        =^{(a)} \sum_{\bar{x}} \frac{w_{x \mid \bar{x}} w_{\bar{x}} ~w_{x' \mid \bar{x}} w_{\bar{x}}}{w_{\bar{x}}}
        = \sum_{\bar{x}} w_{\bar{x}} w_{x \mid \bar{x}} w_{x' \mid \bar{x}}
        =^{(b)} w_{x, x'} = \A[x, x']
    \end{align*}
    where $(a)$ follows from \Cref{eqn:w_xbar_x} and $(b)$ follows from \Cref{eqn:w_x_xp}.
    This completes the proof.
\end{proof}

\subsubsection{Representations to matrices}

The previous section described how to convert distributions to matrices.
We now do the same for representation functions.
For a feature map $\phi: \Xaug \rightarrow \R^{\DD}$, we denote $\Phi\in\R^{\Xaug \times \DD}$ to be the matrix of representations, with the rows being $\Phi[x] = \phi(x)$.
The distributionally normalized version of the representation $\phin(x) = \sqrt{\Daug(x)}\phi(x) = \sqrt{w_{x}} \phi(x)$ is denoted by $\Phin \in \R^{\Xaug \times \DD}$ with row for $x\in\Xaug$ being $\phin(x)$.
We similarly define the matrices for representation $f: \Xaug \rightarrow \R^{\dd}$ to be $F$, $\Fn$ for the distributionally normalized version.
It is easy to see the following relationship between $F$ and $\Fn$: $\Fn = D^{\half} F$.
For the function class of linear representations $\gF_{\phi} = \{W^{\top} \phi(\cdot) \mid W \in \R^{\DD \times \dd}\}$, the matrix version is defined as $\gF_{\Phi} = \{\Phi W \mid W \in \R^{\DD \times \dd}\}$.


\subsection{Connecting losses to matrix notations}


We first define various downstream evaluation metrics for representation.

\begin{definition}
\label{defn:clf_reg_def}
We define the classification and regression error for any augmentation representation function $h: \Xaug \rightarrow \R^{\dd}$.
For any ground-truth labeling $\gstar: \Xorig \rightarrow \{\pm 1\}$ on original inputs, we define the following
\begin{align}
    \Lclf(h;\gstar)
    &= \inf_{w\in\R^{\dd}}~ \ex_{\bar{x}\sim\Xorig} \left[\mathbbm{1}\left\{\sign\left(w^{\top}\aug{h}(\bar{x})\right) = \gstar(\bar{x})\right\}\right]\\
    \Lreg(h;\gstar)
    &= \inf_{w\in\R^{\dd}}~ \ex_{\bar{x}\sim\Xorig} \left[\left(w^{\top}\aug{h}(\bar{x}) - \gstar(\bar{x})\right)^{2}\right]
\end{align}
where $\aug{h}(\bar{x}) = \ex_{x\sim\gA(\cdot\mid\bar{x})}[h(x)]$ is the augmentation averaged representation (see \Cref{table:notation}).
For any labeling $g: \Xaug \rightarrow \{\pm 1\}$ on augmentations, we define the following
\begin{align}
    \Lreg(h; g)
    &= \inf_{w\in\R^{\dd}}~ \ex_{x\sim\Xaug} \left[\left(w^{\top} h(\bar{x}) - g(x)\right)^{2}\right]
\end{align}
\end{definition}


We now connect the downstream regression loss with matrix versions of feature map $\phi$.
\begin{lemma}
\label{lem:Lreg_norm}
    For an arbitrary predictor on augmentations $g\in\{\pm 1\}^{\Xaug}$ and its normalized version $\gn = D^{\half} g$, and an augmentation feature map $\phi: \Xaug \rightarrow \R^{\dd}$ and its normalized matrix $\Phin$,
    \begin{align}
        \Lreg(\phi; g) = \|P_{\Phin}^{\perp} \gn\|^{2}
    \end{align}
\end{lemma}
\begin{proof}
    Note that $\Phin = D^{\half} \Phi$, where $\Phi \in \R^{\Xaug \times \dd}$ is the matrix version of the augmentation feature map $\phi$ (refer \Cref{table:notation}).
    We prove the result by rewriting $\Lreg$ as follows
    \begin{align}
        \Lreg(\phi; g)
        &= \inf_{w\in\R^{\dd}} \ex_{x} \left(\phi(x)^{\top} w - g(x)\right)^{2}
        = \inf_{w\in\R^{\dd}} \sum_{x\in\Xaug} D(x) \left(\phi(x)^{\top} w - g(x)\right)^{2}\\
        &= \inf_{w\in\R^{\dd}} \sum_{x} \left(\sqrt{D(x)} \phi(x)^{\top} w - \sqrt{D(x)} g(x)\right)^{2}\\
        &= \inf_{w\in\R^{\dd}} \left\|D^{\half}\Phi w - D^{\half} g\right\|^{2}
        = \inf_{w\in\R^{\dd}} \left\|\Phin w - \gn\right\|^{2}\\
        &= \left\|P_{\Phin}^{\perp} \gn\right\|^{2}
    \end{align}
\end{proof}


We now express the spectral contrastive loss and upper bound the downstream classification error using matrix versions of distributions and representations. 
\begin{lemma}
\label{lem:matrix_form}
For any representation $f$ and its corresponding normalized matrix $\Fn \in \R^{\Xaug \times \dd}$, the spectral contrastive loss (\Cref{eqn:spectral_loss}) and classification loss (\Cref{eqn:clf_loss}) can be rewritten and upper bounded as
\begin{align}
    \Lspec(f)
    &= \Lspec(\Fn)
    = \left\|\An - \Fn\Fn^{\top}\right\|_{F}^{2} - \left\|\An\right\|_{F}^{2}
    = \left\|\Abarn^{\top}\Abarn - \Fn\Fn^{\top}\right\|_{F}^{2} - \left\|\Abarn^{\top}\Abarn\right\|_{F}^{2}\\
    \Lclf(f; \gstar)
    &\le \Lreg(f; \gstar)
    = \inf_{w\in\R^\dd} \left\|\Abarn \Fn w - \gstarn\right\|_{2}^{2}
    = \|P_{\Abarn \Fn}^{\perp} \gstarn\|_{2}^{2}
\end{align}
\end{lemma}
\begin{proof}
    We first prove the expression for $\Lspec(f)$.
    Note that $\An[x, x'] = w_{x,x'} = \Dsim(x, x')$ from \Cref{eqn:w_x_xp}.
    Furthermore $\Fn[x] = \sqrt{\Daug(x)} f(x)$ from \Cref{table:notation}.
    On expanding out the contrastive loss, we get
    \begin{align*}
        \Lspec(f)
        &= \ex_{\substack{(x,x^{+})\sim\Dsim}}\left[f(x)^{\top}f(x^{+})\right] + \ex_{\substack{x,x^{-}\sim\Dneg^{2}}}\left[ \left(f(x)^{\top}f(x^{-})\right)^{2}\right]\\
        &= -2\sum_{x,x^{+}\in\Xaug} w_{x,x^{+}} f(x)^{\top} f(x^{+}) + \sum_{x,x^{-}\in\Xaug} w_{x} w_{x^{-}} \left(f(x)^{\top} f(x^{-})\right)^{2}\\
        &= \sum_{x,x'\in\Xaug} \left(-2 w_{x,x'} f(x)^{\top} f(x') + w_{x} w_{x^{'}} \left(f(x)^{\top} f(x')\right)^{2}\right)\\
        &= -\sum_{x,x'\in\Xaug} \left(\frac{w_{x,x'}}{\sqrt{w_{x} w_{x'}}}\right)^{2} + \sum_{x,x' \in \Xaug} \left(\frac{w_{x,x'}}{\sqrt{w_{x} w_{x'}}} - \left(\sqrt{w_{x}}f(x)\right)^{\top}\left(\sqrt{w_{x'}}f(x')\right)\right)^{2}\\
        &= -\sum_{x,x'} \An[x, x']^{2} + \sum_{x,x'} \left(\An[x, x'] - \Fn[x]^{\top} \Fn[x']\right)^{2}\\
        &= -\|\An\|_{F}^{2} + \|\An - \Fn\Fn^{\top}\|_{F}^{2}
    \end{align*}
    Now we prove the upper bound of $\Lclf(f; \gstar)$.
    Firstly note that for any input representation $h: \Xorig \rightarrow \{\pm1\}$, we have that 
    \begin{align*}
        \Lclf(h; \gstar)
        &= \inf_{w\in\R^{\dd}} \E_{\bar{x}} \left[\mathbbm{1}\left\{\gstar(\bar{x}) \left({h(\bar{x})}^{\top} w\right) < 0\right\}\right]
        \le^{(a)} \inf_{w\in\R^{\dd}} \E_{\bar{x}} \left[\left(\gstar(\bar{x}) -  {h(\bar{x})}^{\top} w\right)^{2}\right]
        = \Lreg(h; \gstar)
    \end{align*}
    where $(a)$ from the fact that whenever $\gstar(\bar{x}) \left({h(\bar{x})}^{\top} w\right) < 0$, $h(\bar{x})^{\top} w$ has different sign compared to $\gstar \in \{\pm1\}$, and so $(h(\bar{x})^{\top} w - \gstar)^{2} \ge {{}\gstar}^{2} = 1$.
    Thus for an augmentation representation $f: \Xaug$, we have
    \begin{align*}
        \Lreg(f; \gstar)
        &= \Lreg(\aug{f}; \gstar)
        = \inf_{w\in\R^{\dd}} \E_{\bar{x}} \left[\left({\aug{f}(\bar{x})}^{\top} w - \gstar(\bar{x})\right)^{2}\right]
        = \inf_{w\in\R^{\dd}} \sum_{\bar{x}} \left[\left({\sqrt{w_{\bar{x}}}\aug{f}(\bar{x})}^{\top} w - \sqrt{w_{\bar{x}}}\gstar(\bar{x})\right)^{2}\right]\\
        &= \inf_{w\in\R^{\dd}} \sum_{\bar{x}} \left[\left({\sqrt{w_{\bar{x}}}\aug{f}(\bar{x})}^{\top} w - \gstarn(\bar{x})\right)^{2}\right]
    \end{align*}
    We first observe the following about $\aug{f}$:
    \begin{align*}
        \sqrt{w_{\bar{x}}}\aug{f}(\bar{x})
        &= \sqrt{w_{\bar{x}}}\ex_{x\sim\gA(\cdot \mid \bar{x})} [f(x)]
        = \sqrt{w_{\bar{x}}}\sum_{x\in\Xaug} \gA(x \mid \bar{x}) f(x)
        = \sum_{x\in\Xaug} \sqrt{w_{\bar{x}}}\frac{w_{\bar{x}, x}}{w_{\bar{x}}} f(x)
        = \sum_{x\in\Xaug} \frac{w_{\bar{x}, x}}{\sqrt{w_{\bar{x}} w_{x}}} \sqrt{w_{x}} f(x)\\
        &= \sum_{x} \An[\bar{x}, x] \Fn[x]
        = (\An \Fn)[\bar{x}]
    \end{align*}
    Plugging this back into the previous calculation, we get 
    \begin{align*}
        \Lreg(f; \gstar)
        &= \inf_{w\in\R^{\dd}} \sum_{\bar{x}} \left[\left((\An \Fn)[\bar{x}]^{\top} w - \gstarn[\bar{x}]\right)^{2}\right]
        = \inf_{w\in\R^{\dd}} \left\|\An\Fn w - \gstarn\right\|_{F}^{2}
    \end{align*}
    The final step follows from the standard expression for error of linear regression, which is the norm of the component of $\gstarn$ on the null space of $\An\Fn$, i.e. $\|P_{\An\Fn}^{\perp}\|_{F}^{2}$.
\end{proof}


We now show a more specialized form of matrix factorization objective that results from the representation belonging to a particular linear function class.
\begin{lemma}
\label{lem:matrix_form_phi}
For any representation $f\in\gF_{\phi}$ and its normalized matrix $\Fn\in\R^{\Xaug \times \dd}$, the spectral contrastive loss (\Cref{eqn:spectral_loss}) can be rewritten as
\begin{align}
    \Lspec(f)
    &= \Lspec(\Fn)
    = \left\|P_{\Phin}\An P_{\Phin} - \Fn\Fn^{\top}\right\|_{F}^{2} + C
    = \left\|P_{\Phin}\Abarn^{\top} \Abarn P_{\Phin} - \Fn\Fn^{\top}\right\|_{F}^{2} + C
\end{align}
where $C$ is a constant independent of $f$ but dependent on features $\phi$. Here $\Phin$ is the normalized matrix for the features $\phi$ and $P_{\Phin}\in\R^{\Xaug \times \Xaug}$ is the column projection matrix of $\Phin$.
\end{lemma}
\begin{proof}
    From \Cref{lem:matrix_form}, we know that $\Lspec(f)$ can be written as a matrix factorization objective as 
    \begin{align*}
        \Lspec(f)
        &= \left\|\An - \Fn\Fn^{\top}\right\|_{F}^{2} - \left\|\An\right\|_{F}^{2}
    \end{align*}
    Since $f$ is from the class $\gF_{\phi}$, the matrix form $F$ belongs to the class $\gF_{\Phi} = \left\{\Phi W \mid W \in \R^{\DD \times \dd}\right\}$ (refer to \Cref{table:notation}).
    Thus $\Fn = D^{\half} F$ can be written as $\Fn = D^{\half} F = D^{\half} \Phi W = \Phin W$ for some $W \in \R^{\DD \times \dd}$.
    We can conclude that $P_{\Phin} \Fn = \Fn$ and $P^{\perp}_{\Phin} \Fn = 0$ and further simplify the contrastive loss as
    \begin{align*}
        \Lspec(f)
        &=^{(a)} \left\|P_{\Phin}\An P_{\Phin} + P_{\Phin}\An P^{\perp}_{\Phin} + P^{\perp}_{\Phin}\An P_{\Phin} + P^{\perp}_{\Phin} \An P^{\perp}_{\Phin} - \Fn\Fn^{\top}\right\|_{F}^{2} - \left\|\An\right\|_{F}^{2}\\
        &=^{(b)} \left\|P_{\Phin}\An P_{\Phin} - \Fn\Fn^{\top}\right\|_{F}^{2} + \left\|P_{\Phin}\An P^{\perp}_{\Phin} + P^{\perp}_{\Phin}\An P_{\Phin}\right\|_{F}^{2} + \left\|P^{\perp}_{\Phin}\An P^{\perp}_{\Phin}\right\|_{F}^{2} - \left\|\An\right\|_{F}^{2}\\
        &=^{(c)} \left\|P_{\Phin}\An P_{\Phin} - \Fn\Fn^{\top}\right\|_{F}^{2} + C
    \end{align*}
    where $(a)$ follows by decomposing $\An = (P_{\Phin} + P^{\perp}_{\Phin}) \An (P_{\Phin} + P^{\perp}_{\Phin})$, $(b)$ follows because cross terms cancel through $P_{\Phin} P^{\perp}_{\Phin}$ multiplications, and $(c)$ because all other terms are independent of $\Fn$ (and so $f$).
    This completes the proof.
\end{proof}


We now restate the definition of Inconsistency from \Cref{sec:upper_bound} and then relate it to some matrix form.
\begin{definition}[Inconsistency]
\label{defn:apx_inconsitent_pred}
    We define inconsistency of a labeling function $g \in \{\pm 1\}^{\Xaug}$ on augmentations w.r.t. some ground truth labeling $\gstar\in \{\pm 1\}^{\Xorig}$ on original inputs, as followed:
    \begin{align}
        \Lcons(g, \gstar) = \ex_{\bar{x}} \left[ \ex_{x\sim \gA(\cdot \mid \bar{x})} \left[\mathbbm{1}\{g(x) \neq \gstar(\bar{x})\}\right]\right]
    \end{align}
\end{definition}

\begin{lemma}
\label{lem:gstar_aligns}
    For the normalized matrix $\Abarn \in \R^{\Xorig \times \Xaug}$ corresponding to augmentation distribution $\gA$ (refer \Cref{table:notation}), ground-truth labeling $\gstar\in\{\pm 1\}^{\Xorig}$ on original inputs and its normalized version $\gstarn = D^{\half} \gstar$, and an arbitrary predictor $g\in\{\pm 1\}^{\Xaug}$ on augmentations and its normalized version $\gn = D^{\half} g$, we have
    \begin{align}
        {{}\gstarn}^\top \Abarn \gn = 1 - 2\Lcons(g, \gstar)
    \end{align}
\end{lemma}
\begin{proof}
    Since $\gstarn = \Dbar^{\half} \gstar$, $\Abarn = \Dbar^{-\half} \Abar D^{-\half}$ and $\gn = D^{\half} g$, the left hand size is equivalent to ${{}\gstar}^{\top} \Abar g$.
    Expanding this further we get
    \begin{align}
        {{}\gstar}^{\top} \Abar g
        &= \sum_{\bar{x}, x} \Abar[\bar{x}, x] \gstar(\bar{x}) g(x)
        =^{(a)} \sum_{\bar{x}, x} w_{\bar{x}, x} \gstar(\bar{x}) g(x)
        =^{(b)} \sum_{\bar{x}, x} w_{\bar{x}, x} \left(1 - 2\mathbbm{1}\{\gstar(\bar{x}) \neq g(x)\}\right)\\
        &=^{(c)} \sum_{\bar{x}\in\Xorig} w_{\bar{x}} \sum_{x\in\Xaug} w_{x \mid \bar{x}} \left(1 - 2\mathbbm{1}\{\gstar(\bar{x}) \neq g(x)\}\right)\\
        &= 1 - 2\ex_{\bar{x}} \ex_{x\sim\gA(\cdot \mid \bar{x})} \left[\mathbbm{1}\{\gstar(\bar{x}) \neq g(x)\}\right]
        = 1 - 2\Lcons(g, \gstar)
    \end{align}
    where $(a)$ follows from \Cref{eqn:An}, $(b)$ follows from \Cref{eqn:w_xbar_x}, and $(c)$ follows from $\gstar(\bar{x}), g(x) \in \{\pm 1\}$.
\end{proof}

\subsection{Proof of main result}
\label{sec:apx_main_results}

We first state the key lemmas that will used to prove the main result.

The following lemma says that if there is a predictor on augmentations that is consistent with $\gstar$ and also expressible enough by fixed features $\phi$, then most of $\gstar$ is retained by multiplication by $P_{\Phin} \Abarn$.
\begin{lemma}
\label{lem:gstar_aligns_Phi}
    For the normalized matrix $\Abarn \in \R^{\Xorig \times \Xaug}$ corresponding to augmentation distribution $\gA$, an augmentation feature map $\phi$ and corresponding normalized matrix $\Phin$ (refer \Cref{table:notation}), ground-truth labeling $\gstar\in\{\pm 1\}^{\Xorig}$ on original inputs and its normalized version $\gstarn = D^{\half} \gstar$, and an arbitrary predictor $g\in\{\pm 1\}^{\Xaug}$ on augmentations and its normalized version $\gn = D^{\half} g$, we have
    \begin{align}
        \left\|P_{\Phin} \Abarn^{\top} \gstarn\right\| \ge 1 - 2\Lcons(g, \gstar) - \sqrt{\Lreg(\phi; g)}
    \end{align}
\end{lemma}
\begin{proof}
    We will use \Cref{lem:gstar_aligns} to prove this result.
    Note that $\|\gstarn\| = \|\gn\| = 1$.
    First we lower bound $\|P_{\Phin} \Abarn^{\top} \gstarn\|$ by computing ${{}\gstarn}^{\top} \Abarn P_{\Phin} \gn$
    \begin{align*}
        \|P_{\Phin} \Abarn^{\top} \gstarn\|
        &\ge^{(a)} \frac{{{}\gstarn}^{\top} \Abarn P_{\Phin} \gn}{\|\gn\|}
        = {{}\gstarn}^{\top} \Abarn P_{\Phin} \gn\\
        &= {{}\gstarn}^{\top} \Abarn \gn - {{}\gstarn}^{\top} \Abarn P_{\Phin}^{\perp} \gn\\
        &=^{(b)} 1 - 2\Lcons(g,\gstar) - {{}\gstarn}^{\top} \Abarn P_{\Phin}^{\perp} \gn\\
        &\ge^{(c)} 1 - 2\Lcons(g,\gstar) - \|\gstarn\| \|\Abarn\|_{2} \|P_{\Phin}^{\perp} \gn\|\\
        &\ge^{(d)} 1 - 2\Lcons(g,\gstar) -  \|P_{\Phin}^{\perp} \gn\|\\
        &\ge^{(e)} 1 - 2\Lcons(g,\gstar) -  \sqrt{\Lreg(\phi;g)}
    \end{align*}
    where $(a)$ and $(c)$ follow from Cauchy-Schwarz inequality, $(b)$ follows from \Cref{lem:gstar_aligns}, $(d)$ follows from the fact that $\|\Abarn\|_{2} = 1$ and $(e)$ follows from \Cref{lem:Lreg_norm}
\end{proof}


The next lemma quantifies how much of the top singular directions of $P_{\Phin}\Abarn$ are be captured by an $\epsilon$-optimal representation $f$ (or its matrix version $\Fn$). This is related to Lemma D.10 from \citet{haochen2021provable}, however it differs in the fact that we are decomposing $P_{\Phin} \Abarn$ instead of $\Abarn$, and we have a better dependence on $\dd$ on the right hand side.
Furthermore, the sub-optimality term is w.r.t. the best representation in the class $\gF_{\phi}$ rather than the unconstrained optimizer of $\Lspec$.
\begin{lemma}
\label{lem:eps_opt_cont}
    Let $f \in \gF_{\phi} \in \gF_{\phi}$ be an augmentation representation function. Suppose $\Fn\in\R^{\Xaug\times \dd}$ is the normalized representation matrix corresponding to $f$, $\Abarn$ is the normalized matrix corresponding to augmentation distribution $\gA$ and $\Phin$ is normalized version of $\Phi$ (refer \Cref{table:notation}).
    Let $P_{\Phin} \Abarn^{\top} = U S V^{\top}$ be the singular value decomposition, with $\sqrt{\gamma_{1}}, \dots, \sqrt{\gamma_{\DD}}$ being the singular values in decreasing order.
    Then for $\dd' \le \dd$,
    \begin{align}
        \left\|P^{\perp}_{\Fn} U_{:\dd'}\right\|_{F}^{2} \le \frac{\Lspec(f) - \inf\limits_{f^{\star}\in\gF_{\phi}} \Lspec(f^{\star})}{\gamma_{\dd'}^{2} - \gamma_{\dd+1}^{2}} \le \frac{\Lspec(f) - \inf\limits_{f^{\star}\in\gF_{\phi}} \Lspec(f^{\star})}{\left(\gamma_{\dd'} - \gamma_{\dd+1}\right)^{2}}
    \end{align}
    where $U_{:\dd'} \in \R^{\Xaug \times \dd'}$ corresponds to the first $\dd'$ columns (and thus singular vectors) of $P_{\Phin} \Abarn$.
\end{lemma}
\begin{proof}
    We first note, using \Cref{lem:matrix_form_phi}, that the contrastive loss can be written as the following matrix factorization objective
    \begin{align*}
        \Lspec(f)
        = \|P_{\Phin}\An P_{\Phin} - \Fn\Fn^{\top}\|_{F}^{2}
        = \|P_{\Phin}\Abarn^{\top} \Abarn P_{\Phin} - \Fn\Fn^{\top}\|_{F}^{2}
        = \|US^{2}V^{\top} - \Fn\Fn^{\top}\|_{F}^{2}
    \end{align*}
    It is is easy to see that $\gamma_{i}\le 1$ for every $i$, since $\max_{i} \gamma_{i} = \|P_{\Phin} \Abarn\|_{2}^{2} \le \|\Abarn\|_{2}^{2} \le 1$.
    Thus we can invoke Lemma D.10 from \citet{haochen2021provable}, but for matrix $P_{\Phin}\An P_{\Phin}$ instead of $\An$, to argue that
    \begin{align*}
        \|P^{\perp}_{\Fn} u_{i}\|_{F}^{2} \le \frac{\epsilon}{(\gamma_{i}^{2} - \gamma_{\dd+1}^{2})} \le \frac{\epsilon}{(\gamma_{i} - \gamma_{\dd+1})^{2}}
    \end{align*}
    where $\epsilon$ is the suboptimality $\Lspec(f) - \inf\limits_{f^{\star}} \Lspec(f^{\star}) = \Lspec(f) - \inf\limits_{f^{\star}\in\gF_{\phi}} \Lspec(f^{\star})$, since the optimal decomposition for $P_{\Phin} \An P_{\Phin}$ must lie in the span of $\Phin$, and thus $f^{*} \in \gF_{\phi}$.
    Adding these for $i\in[\dd']$ we get
    \begin{align*}
        \|P^{\perp}_{\Fn} U_{:\dd'}\|_{F}^{2} \le \sum_{i=1}^{\dd'} \frac{\epsilon}{(\gamma_{i} - \gamma_{\dd+1})^{2}} \le \frac{\epsilon \dd'}{(\gamma_{\dd'} - \gamma_{\dd+1})^{2}}
    \end{align*}
    This completes the proof.
\end{proof}


We now show conditions under which the top $\dd'$ directions of the matrix being factorized captures significant mass of the ground-truth labels.
\begin{lemma}
\label{lem:top_k_downstream}
    Let $\Abarn$ be the normalized matrix corresponding to augmentation distribution $\gA$ and $\Phin$ be the normalized version of $\Phi$ (refer \Cref{table:notation}).
    Let $P_{\Phin} \Abarn^{\top} = U S V^{\top}$ be the singular value decomposition, with $\sqrt{\gamma_{1}}, \dots, \sqrt{\gamma_{\DD}}$ being the singular values in decreasing order.
    Then we have,
    \begin{align}
        \left\|V_{\dd':}^{\top} ~\gstarn\right\|^{2} \le \frac{1 - \left\|P_{\Phin} \Abarn^{\top} \gstarn\right\|^{2}}{1 - \gamma_{\dd'+1}}
    \end{align}
    where $V_{\dd':} \in \R^{\Xaug \times (|\Xaug|-\dd')}$ corresponds to the last $|\Xaug| - \dd'$ columns (and thus singular vectors) of $P_{\Phin} \Abarn$.
\end{lemma}
\begin{proof}
    We expand out the term $1 - \left\|P_{\Phin} \Abarn^{\top} \gstarn\right\|^{2}$ as follows
    \begin{align*}
        1 - \left\|P_{\Phin} \Abarn^{\top} \gstarn\right\|^{2}
        &= 1 - \left\|USV^{\top} \gstarn\right\|^{2}
        = 1 - \left\|SV^{\top} \gstarn\right\|^{2}
        = 1 - \sum_{i=1}^{\DD} \gamma_{i} (v_{i}^{\top} \gstarn)^{2}\\
        &=^{(a)} \|\gstarn\|_{2}^{2} - \sum_{i=1}^{\DD} \gamma_{i} (v_{i}^{\top} \gstarn)^{2}\\
        &\ge^{(b)} \sum_{i=1}^{\DD}(v_{i}^{\top}\gstarn)^{2} - \sum_{i=1}^{\DD} \gamma_{i} (v_{i}^{\top} \gstarn)^{2}
        = \sum_{i=1}^{\DD} (1-\gamma_{i}) (v_{i}^{\top} \gstarn)^{2}\\
        &\ge^{(c)} (1-\gamma_{\dd'+1}) \sum_{i=\dd'+1}^{\DD} (v_{i}^{\top} \gstarn)^{2}
        = (1-\gamma_{\dd'+1}) \|V_{\dd':}\gstarn\|^{2}
    \end{align*}
    where $(a)$ follows because $\|\gstarn\|^{2} = \sum_{\bar{x}} \Dorig(\bar{x}) \gstar(\bar{x}) = 1$, $(b)$ follows because $\{v_{i}\}_{i=1}^{\DD}$ form a partial orthonormal basis and $(c)$ is true because $\gamma_{i} \le 1$ for every $i$ and since $\gamma_{i}$'s are in decreasing order.
    Rearranging terms completes the proof
\end{proof}


The following lemma connects the eigenvalues of augmentation averaged features that shows up in the final bound, to the eigenvalues of the matrix being decomposed in the spectral contrastive loss.
\begin{lemma}
\label{lem:features_to_matrix}
    Let $\Sigma(\cdot)$ be the covariance operator for features and $\aug{\phi}$ denote augmentation averaged representation obtained from $\phi$ (see \Cref{table:notation}).
    Let $\lambda_{1},\cdots,\lambda_{\DD}$ be the eigenvalues of $I_{\DD} - \Sigma(\phi)^{-\half}\Sigma(\aug{\phi})\Sigma(\phi)^{-\half}$ in increasing order.
    Let $\Abarn$ be the normalized matrix corresponding to augmentation distribution $\gA$ and $\Phin$ be the normalized version of $\phi$ (refer \Cref{table:notation}).
    Let $P_{\Phin} \Abarn^{\top} = U S V^{\top}$ be the singular value decomposition, with $\sqrt{\gamma_{1}}, \dots, \sqrt{\gamma_{\DD}}$ being the singular values in decreasing order.
    Then we have,
    \begin{align}
        \lambda_{i} = 1 - \gamma_{i}, \forall~i\in[\DD]
    \end{align}
\end{lemma}
\begin{proof}
    Let $w_{x}, w_{\bar{x}}, w_{x,\bar{x}}, w_{x\mid\bar{x}}$ be as defined in \Cref{eqn:w_x,eqn:w_xbar,eqn:w_xbar_x,eqn:w_x_given_xbar}.
    Using $\Phin = D^{\half} \Phi$ and that $D$ is diagonal with $D[x,x] = w_{x}$, we first simplify $\Sigma(\phi)$ as follows
    \begin{align*}
        \Sigma(\phi)
        &= \ex_{x \sim \Daug} \left[\phi(x)\phi(x)^{\top}\right]
        = \sum_{x \in \Xaug} \left[w_{x} (x)\phi(x)\phi(x)^{\top}\right]
        = \Phi^{\top} D \Phi = \Phin^{\top} \Phin
    \end{align*}
    Next we find the matrix version of $\aug{\phi}$ using $\Abar[\bar{x}, x] = w_{\bar{x}, x}$ and the following sequence of equalities.
    \begin{align*}
        \aug{\phi}(\bar{x})
        = \ex_{x\sim \gA(\cdot \mid \bar{x})} [\phi(x)]
        = \sum_{x \in \Xaug} w_{x \mid \bar{x}} \phi(x)
        = \frac{1}{w_{\bar{x}}}\sum_{x \in \Xaug} w_{x, \bar{x}} \phi(x)
        = \frac{1}{w_{\bar{x}}}\sum_{x}\Abar[\bar{x}, x] \Phi[x]
    \end{align*}
    Thus the matrix form of $\aug{\phi}$ is $\aug{\Phi} = \Dbar^{-1} \Abar \Phi$.
    Similar to the argument for $\Sigma(\phi)$, we can then write $\Sigma(\aug{\phi})$ as follows
    \begin{align*}
        \Sigma(\aug{\phi})
        &= \aug{\Phi}^{\top} \Dbar \aug{\Phi}
        = (\Dbar^{-1} \Abar \Phi)^{\top} \Dbar (\Dbar^{-1} \Abar \Phi)
        = \Phi^{\top} \Abar^{\top} \Dbar^{-1} \Abar \Phi\\
        &= \Phi^{\top} D^{\half} \left(D^{-\half} \Abar^{\top} \Dbar^{-\half}\right) \left(\Dbar^{-\half} \Abar D^{-\half}\right) D^{\half} \Phi\\
        &=^{(a)} \Phin^{\top} \Abarn^{\top} \Abarn \Phin
    \end{align*}
    where $(a)$ follows from \Cref{eqn:Abarn}.
    Let $MNR^{\top} = \Phin$ be the SVD, so  $P_{\Phin} = M M^{\top}$ and $\Phin^{\top} \Phin = R^{\top} N^{2} R$.
    Using this, we simplify the matrix $\Sigma(\phi)^{-\half}\Sigma(\aug{\phi})\Sigma(\phi)^{-\half}$.
    \begin{align*}
        \Sigma(\phi)^{-\half}\Sigma(\aug{\phi})\Sigma(\phi)^{-\half}
        &= \left(\Phin^{\top}\Phin\right)^{-\half} \Phin^{\top} \Abarn^{\top} \Abarn \Phin \left(\Phin^{\top}\Phin\right)\\
        &= (R N^{2} R^{\top})^{-\half} \left(R N M^{\top}\right) \Abarn^{\top} \Abarn \left(M N R^{\top}\right) (R N^{2} R^{\top})^{-\half}\\
        &= (R N^{-1} R^{\top}) \left(R N M^{\top}\right) \Abarn^{\top} \Abarn \left(M N R^{\top}\right) (R N^{-1} R^{\top})\\
        &= R M^{\top} \Abarn^{\top} \Abarn M R^{\top}
    \end{align*}
    Since $\{\lambda_{i}\}_{i=1}^{\DD}$ are the eigenvalues of $I_{\DD} - \Sigma(\phi)^{-\half}\Sigma(\aug{\phi})\Sigma(\phi)^{-\half}$, $\{1-\lambda_{i}\}_{i=1}^{\DD}$ are the eigenvalues of $\Sigma(\phi)^{-\half}\Sigma(\aug{\phi})\Sigma(\phi)^{-\half} = R M^{\top} \Abarn^{\top} \Abarn M R^{\top}$.
    Thus $\{\sqrt{1-\lambda_{i}}\}_{i=1}^{\DD}$ are the singular values of $R M^{\top} \Abarn^{\top}$ and thus $M^{\top} \Abarn^{\top}$ and thus $M M^{\top} \Abarn^{\top} = P_{\Phin} \Abarn^{\top}$.
    The previous statements are true because multiplication by an orthogonal matrix does not change the singular values.
    Thus $\sqrt{\gamma_{i}} = \sqrt{1 - \lambda_{i}}$, finishing the proof.
\end{proof}


\begin{lemma}
\label{lem:matrix_to_laplacian}
    Let $\Sigma(\cdot)$ be the covariance operator for features and $\aug{\phi}$ denote augmentation averaged representation obtained from $\phi$ (see \Cref{table:notation}).
    Let $\Ln = I - \An$ be the Laplacian of the augmentation graph.
    If the features $\phi$ are full rank, then the eigenvalues of $I_{\DD} - \Sigma(\phi)^{-\half}\Sigma(\aug{\phi})\Sigma(\phi)^{-\half}$ are the same as the eigenvalues of $\Ln$.
\end{lemma}
\begin{proof}
    Note from \Cref{lem:Abar_to_A} that the normalized adjacency matrix can be rewritten as $\An = \Abarn^{\top} \Abarn$.
    Also from \Cref{lem:features_to_matrix}, we can imply that the eigenvalues of $I - P_{\Phin} \Abarn^{\top} \Abarn P_{\Phin}$ are the same as the eigenvalues of $I_{\DD} - \Sigma(\phi)^{-\half}\Sigma(\aug{\phi})\Sigma(\phi)^{-\half}$.
    Since $\phi$ is full rank, so is the matrix $\Phin$, thus $P_{\Phin} = I$.
    So $I - P_{\Phin} \Abarn^{\top} \Abarn P_{\Phin} = I - \Abarn^{\top} \Abarn = I - \An = \Ln$; this completes the proof.
\end{proof}


We are now ready to present our main result.

\begin{reptheorem}{thm:low_rank_result}
    Let $\Sigma(\cdot)$ be the covariance operator for features and $\aug{\phi}$ denote augmentation averaged representation obtained from $\phi$ (see \Cref{table:notation}).
    Let $\lambda_{1},\cdots,\lambda_{d}$ be the eigenvalues of $I_{d} - \Sigma(\phi)^{-\half}\Sigma(\aug{\phi})\Sigma(\phi)^{-\half}$ in increasing order, then for $\dd \le \DD$, any representation $f\in\gF_{\phi}$, will satisfy
    \begin{align}
        \Lclf(f; \gstar) \le \min\limits_{\substack{1 \le \dd' \le \dd}} \left\{\frac{\min\limits_{g\in\{\pm 1\}^{\Xaug}} 4\left(2\Lcons(g, \gstar) + \sqrt{\Lreg(\phi; g)}\right)}{\lambda_{\dd'+1}} +  \frac{2\dd'\left(\Lspec(f) - \inf\limits_{f^{\star}\in\gF_{\phi}}\Lspec(f^{\star})\right)}{(1-\lambda_{\dd'}) (\lambda_{\dd+1}-\lambda_{\dd'})^{2}}\right\}
    \end{align}
    where $\Lcons$ is defined in \Cref{defn:inconsitent_pred} and $\Lreg$ in \Cref{defn:clf_reg_def}.
\end{reptheorem}

\begin{proof}
    We first sketch an outline of the proof and how different lemmas will be used to prove the final result.
    We will use matrix versions of distributions and functions from \Cref{table:notation} throughout the proof.
    The following are the main steps:
    \begin{enumerate}[leftmargin=*]
        \item (Matrix factorization) 
        The spectral contrastive loss is shown to be equivalent to a matrix factorization objective as in \citet{haochen2021provable}.
        For representations in the class $\gF_{\phi}$, \Cref{lem:matrix_form_phi} shows that the problem of contrastive learning is reduced to matrix factorization of a projected adjacency matrix $P_{\Phin} \An P_{\Phin}$ through the objective $\Lspec(F) = \left\|P_{\Phin}\Abarn^{\top}\Abarn P_{\Phin} - \Fn \Fn^{\top}\right\|_{F}^{2} + C$, where $\An$ is the normalized matrix corresponding to augmentation distribution $\gA$ and $\Fn$ is the normalized matrix for representation function $f$ (refer \Cref{table:notation}).
        Thus the spectral contrastive loss $\Lspec$ is attempting to find a rank $\dd$ approximation for $P_{\Phin} \Abarn^{\top}$.

        \item ($\epsilon$-optimal solutions) If $P_{\Phin} \Abarn^{\top} = U S V^{\top}$ is the singular value decomposition, then any $\epsilon$-optimal representation $f$ (and corresponding $\Fn$), i.e. $\Lspec(f) - \inf\limits_{f^{\star}\in\gF_{\phi}}\Lspec(f^{\star})\le\epsilon$, can be shown (\Cref{lem:eps_opt_cont}) to capture most of the signal for the top $\dd'$ directions of $P_{\Phin} \Abarn^{\top}$, i.e. $\left\|P_{\Fn}^{\perp} U_{:\dd'}\right\|_{F} = \gO(\epsilon)$ is small.

        \item (Connecting to downstream) Top $\dd'$ directions of $P_{\Phin} \Abarn^{\top}$ can be shown to capture a lot of the mass of $\gstar$ if the features $\phi$ and augmentation distribution $\gA$ are ``nice'' enough, as quantified by \Cref{lem:top_k_downstream} that upper bounds $\left\|V_{\dd':} ~\gstarn\right\|^{2}$, in conjunction with \Cref{lem:gstar_aligns_Phi} which quantifies these nice properties of $\gA$ and features $\phi$.

        \item (Wrapping up) Both of the above steps will have upper bounds that depend on the singular values of $P_{\Phin} \Abarn$.
        Relating the singular values of $P_{\Phin} \Abarn$ with the eigenvalues of $\left(I_{d} - \Sigma(\phi)^{-\half}\Sigma(\aug{\phi})\Sigma(\phi)^{-\half}\right)$ through \Cref{lem:features_to_matrix} completes the proof.
    \end{enumerate}
    
    \textbf{Step 1:}
    We can rewrite the contrastive loss in matrix forms using \Cref{lem:matrix_form}.
    \begin{align}
        \Lspec(f)
        = \Lspec(\Fn)
        &= \left\|\An - \Fn\Fn^{\top}\right\|_{F}^{2} - \left\|\An\right\|_{F}^{2}
        = \left\|\Abarn^{\top}\Abarn - \Fn\Fn^{\top}\right\|_{F}^{2} - \left\|\Abarn^{\top}\Abarn\right\|_{F}^{2}
    \end{align}
    For representation $F\in\gF_{\phi}$, we can write it as $F = \Phi W$, thus giving $\Fn = D^{\half} F = D^{\half} \Phi W = \Phin W$.
    Note that $P_{\Phin} = \Phin\Phin^{\dagger}$ is the projection matrix of column space of $\Phin$; then we have $\Fn = P_{\Phin} \Fn$.
    From \Cref{lem:matrix_form_phi}, we know that $\Lspec(F) = \left\|P_{\Phin}\Abarn^{\top}\Abarn P_{\Phin} - \Fn \Fn^{\top}\right\|_{F}^{2} + C$.

    Thus from this we see that the contrastive learning is aiming to learn a good rank $\dd$ decomposition of the matrix $P_{\Phin} \An P_{\Phin} = P_{\Phin} \Abarn^{\top} \Abarn P_{\Phin}$.
    This is a similar to the formulation in \citet{haochen2021provable} where the matrix $\An$ is being factorized instead.
    The classical result on low-rank approximation of matrices tells us that the minimizer $\Fn$ will span the top $\dd$ singular of $P_{\Phin} \Abarn$.

    \textbf{Step 2:}
    Let $P_{\Phin} \Abarn = U S V^{\top}$ be the singular value decomposition, with $S = \text{diag}(\sqrt{\gamma_{1}}, \dots, \sqrt{\gamma_{d}})$ being the singular values in decreasing order.
    Then we know that the optimal solution $\Fn^{\star}$ be will $U_{:\dd} S_{:\dd} R$ for any orthogonal matrix $R$.
    Note that $\Fn^{\star} \in \gF_{\Phin}$ since the matrix $P_{\Phin} \Abarn$ being decomposed is in the span of $\Phin$.
    This argument can be extended to $\epsilon$-optimal representation $f$ (or matrix $F$) by invoking \Cref{lem:eps_opt_cont}, which gives us that
    \begin{align}
        \left\|P^{\perp}_{\Fn} U_{:\dd'}\right\|_{F}^{2} \le \frac{\epsilon \dd'}{\gamma_{\dd'}^{2} - \gamma_{\dd+1}^{2}} \le \frac{\epsilon  \dd'}{\left(\gamma_{\dd'} - \gamma_{\dd+1}\right)^{2}}
    \end{align}
    This tells us that being close to optimality ensures that the representation captures most of the top $\dd'$ singular directions of $U$ and thus $P_{\Phin} \Abarn$, whenever $\dd' \le \dd$.
    Note that the gap $\gamma_{\dd'} - \gamma_{\dd+1}$ in singular values determines how the suboptimality affects the magnitude of ``signal'' captured.

    \textbf{Step 3:}
    Given that $\epsilon$-optimal solutions can capture the top directions of $P_{\Phin} \Abarn$ (or $U$), we now focus our attention on what this means for downstream performance.
    We invoke \Cref{lem:matrix_form} again to upper bound the downstream classification error $\Lclf$ (refer \Cref{defn:clf_reg_def}) as
    \begin{align}
        \Lclf(f; \gstar)
        &\le \Lreg(f; \gstar)
        \le \Lreg(\Fn)
        = \inf_{w\in\R^{\dd}} \left\|\Abarn \Fn w - \gstarn\right\|_{2}^{2}\\
        &\le^{(a)} \inf_{w\in\R^{\dd}} \left\|\Abarn P_{\Phin} \Fn w - \gstarn\right\|_{2}^{2}
        \le^{(b)} \inf_{w\in\R^{\dd}} \left\|\Abarn P_{\Phin} P_{\Fn} w - \gstarn\right\|_{2}^{2}\\
        &\le^{(c)} \inf_{w\in\R^{\dd}} \left\|V S U^{\top} P_{\Fn} w - \gstarn\right\|_{2}^{2}
    \end{align}
    where $(a)$ follows from the fact that $\Fn \in \gF_{\Phin}$ and thus $P_{\Phin} \Fn = \Fn$, $(b)$ is true since for any $w\in\R^{\dd}$, there exists $w'\in\R^{\dd}$ such that $\Fn w = P_{\Fn} w'$, and $(c)$ uses the singular value decomposition of $P_{\Phin} \Abarn$.
    Thus the downstream error is upper bounded by a quantity that depends on how much of $\gstarn$ is not captured by the columns of $\Abarn P_{\Phin} P_{\Fn} = V S U^{\top} P_{\Fn}$.
    We show this quantity is small, by arguing that the top $\dd'$ directions of $V$ captures enough component of $\gstarn$ \Cref{lem:gstar_aligns_Phi}, and that an $\epsilon$-optimal representation will capture a large enough portion of the top $\dd'$ directions.
    Note that for any matrix $B\in\R^{n\times n}$, $B_{:m}\in\R^{m\times n}$ denotes the first $m$ columns of $B$ and $B_{m:}\in\R^{n-m\times n}$ denotes that last $m$ columns of $B$.
    The calculation is as follows
    \begin{align*}
        \Lclf(f; \gstar)
        &\le \inf_{w\in\R^{\dd}} \left\|V S U^{\top} P_{\Fn} w - \gstarn\right\|_{2}^{2}\\
        &= \inf_{w\in\R^{\dd}} \left\|V S U^{\top} P_{\Fn} w - V_{:\dd'} V_{:\dd'}^{\top}\gstarn + V_{\dd':} V_{\dd':}^{\top}\gstarn\right\|_{2}^{2}\\
        &\le^{(a)} 2 \left(\inf_{w\in\R^{\dd}} \left\|V S U^{\top} P_{\Fn} w - V_{:\dd'} V_{:\dd'}^{\top}\gstarn\right\|^{2} + \left\|V_{\dd':} V_{\dd':}^{\top}\gstarn\right\|_{2}^{2}\right)\\
        &= 2 \inf_{w\in\R^{\dd}} \left\|V_{:\dd'} S_{:\dd'} U_{:\dd'}^{\top} P_{\Fn} w - V_{:\dd'} V_{:\dd'}^{\top}\gstarn\right\|^{2} + 2\left\|V_{\dd':} V_{\dd':}^{\top}\gstarn\right\|_{2}^{2}\\
        &= 2 \inf_{w\in\R^{\dd}} \left\|S_{:\dd'} U_{:\dd'}^{\top} P_{\Fn} w - V_{:\dd'}^{\top}\gstarn\right\|^{2} + 2\left\|V_{\dd':}^{\top}\gstarn\right\|_{2}^{2}\\
        &\le^{(b)} 2 \left\|S_{:\dd'} U_{:\dd'}^{\top} P_{\Fn} U_{:\dd'} S_{:\dd'}^{-1} V_{:\dd'}^{\top} \gstarn - V_{:\dd'}^{\top}\gstarn\right\|^{2} + 2\left\|V_{\dd':}^{\top}\gstarn\right\|_{2}^{2}\\
        &= 2 \left\|S_{:\dd'} U_{:\dd'}^{\top} P_{\Fn}^{\perp} U_{:\dd'} S_{:\dd'}^{-1} V_{:\dd'}^{\top} \gstarn\|^{2} + 2\|V_{\dd':}^{\top}\gstarn\right\|_{2}^{2}\\
        &\le^{(c)} 2 \left\|S_{:\dd'}\right\|_{2}^{2} \left\|P_{\Fn}^{\perp} U_{:\dd'}\right\|_{2}^{2} \left\|S_{:\dd'}^{-1}\right\|_{2}^{2} \left\|\gstarn\right\|^{2} + 2\left\|V_{\dd':}^{\top}\gstarn\right\|_{2}^{2}\\
        &\le^{(d)} \frac{2\left\|P_{\Fn}^{\perp} U_{:\dd'}\right\|_{F}^{2}}{\gamma_{\dd'}} + 2\left\|V_{\dd':}^{\top}\gstarn\right\|_{2}^{2}
    \end{align*}
    where $(a)$ follows from the inequality $\|a + b\|^{2} \le 2(\|a\|^{2} + \|b\|^{2})$, $(b)$ follows by a picking a specific value $w = U_{:\dd'}S^{-1}_{:\dd'}V_{:\dd'}\gstarn$, $(c)$ follows from multiple applications of Cauchy-Schwarz inequality and that $\|V_{:\dd'}^{\top} \gstarn\| \le \|\gstarn\|$ and $(d)$ follows from $\|S_{:\dd'}\|_{2} \le 1$ and $\|S_{:\dd'}^{-1}\|_{2} \le \gamma_{\dd'}^{-1}$.

    The first term is upper bounded in step 2 already by the sub-optimality of $f$, while the second term is upper bounded using \Cref{lem:top_k_downstream}.
    Plugging these in, we get
    \begin{align*}
        \Lclf(f; \gstar)
        \le \frac{2\epsilon \dd'}{\gamma_{\dd'} (\gamma_{\dd'} - \gamma_{\dd+1})^{2}} + \frac{2(1 - \left\|P_{\Phin} \Abarn^{\top} \gstarn\right\|^{2})}{1 - \gamma_{\dd'+1}}
        \le \frac{2\epsilon \dd'}{\gamma_{\dd'} (\gamma_{\dd'} - \gamma_{\dd+1})^{2}} + \frac{4(1 - \left\|P_{\Phin} \Abarn^{\top} \gstarn\right\|)}{1 - \gamma_{\dd'+1}}
    \end{align*}
    where for the last inequality we use that $1-x^{2} = (1-x)(1+x) \le 2(1-x)$ for $x\in[0,1]$.
    This is further simplified using \Cref{lem:gstar_aligns_Phi} to
    \begin{align*}
        \Lclf(f; \gstar) \le \frac{2\epsilon \dd'}{\gamma_{\dd'} (\gamma_{\dd'} - \gamma_{\dd+1})^{2}} + \frac{4\left(2\Lcons(g, \gstar) + \sqrt{\Lreg(\phi; g)}\right)}{1 - \gamma_{\dd'+1}}
    \end{align*}
    
    \textbf{Step 4:}
    Finally the singular values $\gamma_{i}$ are linked the eigenvalues $\lambda_{i}$ in the theorem statement through \Cref{lem:features_to_matrix}.
    Specifically, we have $\gamma_{i} = 1-\lambda_{i}$, giving us the final result
    \begin{align*}
        \Lclf(f; \gstar) \le \frac{4\left(2\Lcons(g, \gstar) + \sqrt{\Lreg(\phi; g)}\right)}{\lambda_{\dd'+1}} +  \frac{2\epsilon \dd'}{(1-\lambda_{\dd'}) (\lambda_{\dd+1}-\lambda_{\dd'})^{2}}
    \end{align*}
    where $\epsilon$ is the suboptimality $\Lspec(f) - \inf\limits_{f^{\star}\in\gF_{\phi}}\Lspec(f^{\star})$.
    The above inequality holds for every $g\in\{\pm1\}^{\Xaug}$ and for every $\dd'\in[\dd]$.
    Taking a $\min$ over both completes the proof.
    This completes the proof.

\end{proof}

\subsection{Discussion of upper bound}
\label{sec:apx_discussion}

We dissect our result from \Cref{thm:low_rank_result}, and compare it to the result from \citet{haochen2021provable}, presented in \Cref{thm:haochen_result}.
For the representation $f\in\gF_{\phi}$, downstream performance is good if
\begin{itemize}[leftmargin=*]
    \item \textit{$\Lspec(f) - \inf_{f^{\star}\in\gF_{\phi}}\Lspec(f^{\star})$ is small}: The contrastive loss of $f$ is close to the optimal loss in $\gF_{\phi}$, even if best in class is far from the absolute minimizer.
    The equivalent term in \Cref{thm:haochen_result} was the global sub-optimality of $f$, i.e. $\Lspec(f) - \inf_{f^{\star}} \Lspec(f^{\star})$.
    \item \textit{$2\Lcons(g, \gstar) + \sqrt{\Lreg(\phi; g)}$ is small}: This happens if there exists a predictor $g\in\{\pm1\}^{\Xaug}$ on augmentations that is expressible by the features $\phi$ and is sufficiently consistent with the ground-truth labels $\ystar$ on inputs.
    Note that if augmentation distributions overlap across classes, then $\Lcons(g,\ystar)$ cannot be made small.
    In fact, $\Lcons(g,\ystar)$ is of the same order as $\alpha$ from \Cref{thm:haochen_result}.
    The extra condition we need here is that $\sqrt{\Lreg(\phi; g)}$ is small, i.e. despite $\phi$ not being full rank, it can roughly express a function that is consistent with ground-truth labels.
    \item \textit{Eigenvalues $\lambda_{\dd'}$ and eigen-gaps $\lambda_{\dd+1} - \lambda_{\dd'}$ are not too small}: This is very similar to \Cref{thm:haochen_result}, except there the eigenvalues were of the normalized Laplacian (that only depended on distributions), while here the eigenvalues also depend on $\phi$ and thus the function class.
    Intuitively these values are large if the augmentation graph is dense {\em in the view of the features $\phi$}.
\end{itemize}


\section{Lower bounds for (approximately) disjoint augmentations}
\label{sec:lower_bound_proof}

Here we prove that the global minimizer of the contrastive objective can achieve trivial downstream performance when the augmentation distributions do not overlap. 

\begin{theorem}
Let $N \in \mathbb{N}$ be given and let $d \in \mathbb{N}$ satisfy $3 \leq d \leq c N/\log_2(N)$ for a universal constant $c>0$. Let $\bar{\mathcal{X}}$ be a set of $|\bar{\mathcal{X}}|=N$ instances, $\ystar \in \{\pm 1\}^N$ be any labeling function with $\sum_i y_i^\star = 0$, and let $\mathcal{D}$ be the uniform distribution over $\bar{\mathcal{X}}$. Suppose that the augmentation distribution $\mathcal{A}(\cdot \mid \bar{x})$ is such that $\forall \bar{x}, \bar{x}' \in \bar{\mathcal{X}}: \mathrm{supp}(\mathcal{A}(\cdot \mid \bar{x})) \cap \mathrm{supp}(\mathcal{A}(\cdot \mid \bar{x}')) = \emptyset$. Additionally assume either
\begin{itemize}
    \item \textbf{Unnormalized case:} representations are unconstrained; or
    \item \textbf{Normalized case:} representations are constrained (to any set) and there is a fixed source of randomness $W \in \Delta(\mathcal{W})$ and mapping $T: \bar{\mathcal{X}}\times\mathcal{W} \to \mathcal{X}$ that is invertible in $w$ for any $\bar{x}$ such that $x \sim \mathcal{A}(\cdot \mid \bar{x}) \equiv w \sim W, x = T(\bar{x},w)$.
\end{itemize}
Then for any representation $f^\star: \mathcal{X} \to \RR^d$ there exists a representation $\hat{f}: \mathcal{X} \to \mathbb{R}^d$ such that:
\begin{align*}
    \Lcont(\hat{f}) \leq \Lcont(f^\star), \quad \textrm{ and } \quad \Lclf(\hat{f}) = \min_{w \in \mathbb{R}^d} \frac{1}{N} \sum_{i=1}^n \mathbf{1}\{ \mathrm{sign}(w^\top \hat{f}(\bar{x}_i)) \ne y_i^\star\} \geq \frac{1}{2} - O(\sqrt{d \log(N)/N})
\end{align*}
\end{theorem}
\begin{proof}
    Let us start by considering the unnormalized case. Let $f^\star: \mathcal{X} \to \mathbb{R}^d$ be any representation function. The proof consists of three steps:
    \begin{enumerate}
        \item Show that every instance $\bar{x} \in \bar{\mathcal{X}}$ has an embedding $v_{\bar{x}}$, such that if we embed $\bar{x}$ and all of its augmentations to $v_{\bar{x}}$ then we obtain a new embedding function $\hat{f}$ for which $\Lcont(\hat{f})$ is no worse than $\Lcont(f^\star)$. 
        \item Let $\mathcal{V} := \{ v_{\bar{x}} : \bar{x} \in \bar{\mathcal{X}}\}$. Show that for any bijection $\pi: \mathcal{V} \to \mathcal{V}$, we have $\Lcont(\pi\circ \hat{f}) = \Lcont(\hat{f})$. In other words, if we apply a permutation to the embeddings of $\hat{f}$ we do not change the contrastive loss. 
        \item Show that there exists some permutation $\pi$ such that $\pi\circ \hat{f}$ has very high downstream error rate.  
    \end{enumerate}
    
\paragraph{Part 1.}
Let us first show that embeds all augmentations of an instance $\bar{x}$ identically only lowers the contrastive loss. By convexity, we have that
\begin{align*}
    \Lcont(f) &:= \E_{\bar{x}_1, \bar{x}_2 \sim \mathcal{D}} \E_{(x,x^+) \sim \mathcal{A}(\cdot \mid \bar{x}_1), x^- \sim \mathcal{A}(\cdot \mid \bar{x}_2)} \left[ -\log \left(\frac{ \exp(f(x)^\top f(x^+)) }{\exp(f(x)^\top f(x^+)) + \exp(f(x)^\top f(x^-))} \right)\right]\\
    & \geq \E_{\bar{x}_1, \bar{x}_2 \sim \mathcal{D}} \left[ -\log \left(\frac{ \exp(g(\bar{x}_1)^\top g(\bar{x}_1)) }{\exp(g(\bar{x}_1)^\top g(\bar{x}_1)) + \exp(g(\bar{x}_1)^\top g(\bar{x}_2))} \right)\right]
\end{align*}
where $g(\bar{x}) = \mathbb{E}_{x \sim \mathcal{A}(\cdot \mid \bar{x})}[f(x)]$ is the mean embedding for $f$. Note that this inequality is strict if $f$ does not embed all augmentations of an instance identically. If $f^\star$ did not embed augmentations identically, we could replace $f^\star$ with the mean embedding $\hat{f}$ and reduce the contrastive loss. Thus there exists a set $\mathcal{V} := \{ v_{\bar{x}} : \bar{x} \in \bar{\mathcal{X}}\}$ such that for each $\bar{x}$, we can assume that $\hat{f}$ embeds $\bar{x}$ and all of its augmentations as $v_{\bar{x}}$.

The same argument also hold for the spectral contrastive loss defined in \Cref{eqn:spectral_loss}, since it is also convex in the inner products.

\paragraph{Part 2.}
Let us rename the embedding vectors $\{v_{\bar{x}}\}_{\bar{x} \in \bar{\mathcal{X}}}$ to $\mathcal{V} := \{v_i\}_{i=1}^N$. Then we can rewrite the objective as
\begin{align*}
    \Lcont(\hat{f}) = \frac{1}{N^2} \sum_{i,j} - \log\left( \frac{\exp(v_i^\top v_i)}{\exp(v_i^\top v_i) + \exp(v_i^\top v_j) }\right)
\end{align*}
Any bijection from $\mathcal{V}$ to $\mathcal{V}$ can be equivalently viewed as a permutation $\pi: [N] \to [N]$. We claim that the above objective is invariant to permuting the indices. This is easy to see, since all pairs $(i,j)$ appear with equal weighting in the above expression. 
Thus we see that $\Lcont(\pi\circ \hat{f}) = \Lcont(\hat{f})$.

\paragraph{Part 3.}
In the last step of the proof, we use a combinatorial argument to show that there exists some permutation $\pi$ with high error rate. First, note that the embedding function $\pi \circ \hat{f}$ embeds $\bar{x}_i$ and all of its augmentations to $v_{\pi(i)}$.
Thus, the downstream loss when using linear function $w$ is
\begin{align*}
    \frac{1}{N} \sum_{i=1}^n \mathbf{1}\{ \mathrm{sign}(w^\top v_{\pi(i)}) \ne y_i^\star\} = \frac{1}{N} \sum_{i=1}^n \mathbf{1}\{ \mathrm{sign}(w^\top v_{i}) \ne y_{\pi^{-1}(i)}^\star\}
\end{align*}
So instead of permuting the embeddings $\{v_i\}$, we can equivalently permute the labels $\{y_i^\star\}$. 
Define:
\begin{align*}
    \mathcal{Y} &:= \{ ( y_{\pi(i))}^\star )_{i=1}^N : \pi \textrm{ is a permutation } \} \\
    \mathcal{W} &:= \left\{ (\mathrm{sign}(w^\top v_i))_{i=1}^N : w \in \mathbb{R}^d \right\}\\
    \mathcal{Z}_{\tau} &:= \{x \in \{\pm 1\}^N : \exists b\in \mathcal{W} \textrm{ s.t. } \frac{1}{N}\sum_{i=1}^N \mathbf{1}\{x_i\ne b_i\} \leq \tau  \}.
\end{align*}
Here $\mathcal{Y}$ are the possible labellings we can generate by permuting the indices (which as we discussed is equivalent to permuting the embedding vectors). $\mathcal{W}$ is the labels we can generate via a linear function of the embeddings. Finally $\mathcal{Z}_{\tau}$ is the set of labellings that are $\tau$ close to the ones that our embeddings can generate. The statement of the theorem is equivalent to $\mathcal{Y} \setminus \mathcal{Z}_{\tau} \ne \emptyset$ for some large $\tau$, which means that there is some permutation of the labels that is far from all linear functions of our embeddings. 

We prove this via a combinatorial argument. First, since $\sum_i y_i^\star = 0$ (meaning that the classes are balanced), we have $|\mathcal{Y}| = {N \choose N/2} \geq 2^{N/2}$. On the other hand, by Sauer's lemma,
\begin{align}
    \left|\mathcal{W}\right| \leq \sum_{i=0}^d {N \choose i}, \quad \textrm{hence} \quad 
    |\mathcal{Z}_{\tau}| \leq \sum_{i=0}^d {N \choose i} \cdot \sum_{i=0}^{N\tau} {N \choose i}
\end{align}
Let $H(p) := p \log_2(1/p) + (1-p)\log_2(1/(1-p))$ be the binary entropy function, for $p \in [0,1]$. Standard bounds on the volume of Hamming cubes \cite{cover1999elements} gives that
\begin{align*}
    |\mathcal{Z}_{\tau}| \leq 2^{H(d/N)\cdot N} 2^{H(\tau)\cdot N}.
\end{align*}
We also have
\begin{align*}
    {N \choose N/2} \geq 2^{H(1/2)\cdot N}\cdot \frac{2}{eN} \geq 2^{N - \log_2(eN/2)}
\end{align*}
Therefore, a sufficient condition is
\begin{align*}
    H(d/N) + H(\tau) <= 1 - \frac{\log_2(eN/2)}{N}
\end{align*}
To proceed, we upper bound the entropy functional on the left hand side using the taylor expansion. For the $H(d/N)$ term we use a first order expansions around $p=1/N$, which, by concavity, yields an upper bound. 
\begin{align*}
    H(d/N) &\leq H(1/N) + \left.\frac{\partial H(x)}{\partial x}\right|_{x=1/N} \left( d/N - 1/N \right)\\
    & = 1/N \log_2(N) + (1-1/N) \log_2(N/(N-1)) - \log_2\left( \frac{1/N}{1-1/N}\right)\cdot\left(d/N - 1/N\right)\\
    & = \log_2(N/(N-1)) + d/N \cdot \log_2(N-1)\\
    & \leq 2/N + d\log_2(N)/N.
\end{align*}
The last inequality holds for $N \geq 2$. For $H(\tau)$ we have the upper bound
\begin{align*}
    H(\tau) = H(1/2) - \frac{4}{\ln(2)} \cdot \frac{1}{2} (1/2 - \tau)^2 + \frac{H^{(3)}(\xi)}{6} (\tau-1/2)^3 \leq 1 - \frac{2}{\ln(2)} (1/2 - \tau)^2,
\end{align*}
Here the first equality is Taylor's remainder theorem where $\xi \in [\tau,1/2]$ and the second holds because the third derivative is non-negative on the interval $[0,1/2]$ and we will take $\tau \leq 1/2$. Putting these together, a sufficient condition is
\begin{align*}
    & \frac{2}{N} + \frac{ d\log_2(N)}{N} + 1 - \frac{2}{\ln(2)} (1/2 - \tau)^2 \leq 1 - \frac{\log_2(eN/2)}{N}\\
    & \Leftarrow \tau < \frac{1}{2} - \sqrt{ \frac{\ln(2)}{2} \left(\frac{d \log_2(N)}{N} + \frac{2}{N} + \frac{ \log_2(eN/2)}{N}\right)}
\end{align*}
So the error rate is $1/2 - O(\sqrt{d \log(N)/N})$. 

For the normalized case, the proof is structurally very similar, except that we cannot rely on the argument in part 1 to show that $f^\star$ embeds $\bar{x}$ and all of its augmentations to the same vector $v_{\bar{x}}$. However, we only use the mean vector $v_{\bar{x}} := \E_{x \sim \mathcal{A}(\cdot \mid \bar{x})}[f^\star(x)]$ in subsequent steps of the proof and we will see that we can remap embeddings $f^\star(x)$ so that we (a) preserve the NCE loss of $f^\star$ and (b) permute all of the mean vectors $v_{\bar{x}}$. 

Let us number the original inputs $\bar{x}_1,\ldots,\bar{x}_N$ and let $\pi: [N] \to [N]$ be any permutation. Let $\mathcal{W}$ be the choices for the random seed and for input $\bar{x}_i$ let $x_{i,w} = T(\bar{x}_i,w)$ be the augmentation obtained when using seed $w$ on input $\bar{x}$. The invertibility of $\mathcal{A}(\bar{x}_i,\cdot)$ implies that $x_{i,w} \ne x_{i,w'}$. This means that we can define a new predictor $f_\pi$ as
\begin{align*}
    f_\pi: x_{i,w} \mapsto f^\star(x_{\pi(i),w}).
\end{align*}
Since the examples are sampled uniformly at random and since the random seed is independent of the example, we can show that $\Lcont(f_\pi) = \Lcont(f^\star)$ using a similar argument to the one we used to show permutation invariance in the unnormalized case.
At the same time, we have changed the mean embeddings so that $\bar{x}$ is now embedded as $v_{\bar{x}_{\pi(i)}}$. So now we can continue with part 3 to obtain the result. 

\end{proof}

\subsection{Approximately disjoint augmentations}
\label{sec:apx_apx_disjoint_lemma}


\begin{definition}
\label{defn:bayes_error}
    For an augmentation distribution $\gA$, we define $\BE(\gA)$ as the Bayes error of augmentation classification as the minimum error achievable in the input identification task, i.e. predicting the input that could have generated an augmentation.
    Formally we define it as follows:
    \begin{align}
    \label{eqn:bayes_error}
        \BE(\gA)
        = \inf_{g: \Xaug \rightarrow \Xorig} ~\ex_{\bar{x}} \left[\ex_{x\sim\gA(\cdot \mid \bar{x})} \left[\mathbbm{1}\left\{g(x) \neq \bar{x}\right\}\right]\right]
    \end{align}
\end{definition}

\begin{lemma}
\label{lem:bayes_error}
    For an augmentation distribution $\gA$, the Bayes error from \Cref{defn:bayes_error} has the following expression
    \begin{align}
        \BE(\gA)
        &= 1 - \ex_{x \sim \gD} \left[\|\gA(\cdot \mid x)\|_{\infty}\right]
    \end{align}
    where $\gA(\cdot \mid x)$ is the posterior distribution over original inputs given an augmentation $x$.
\end{lemma}
\begin{proof}
    In the above definition of Bayes error, we pick the optimal predictor $g$ to be $g(x) = \argmax_{\bar{x}} \gA(\bar{x} \mid x)$, which will give us the expression for Bayes error.
\end{proof}

\begin{lemma}
\label{lem:bayes_error_eigen_gap}
    Consider the augmentation distribution $\gA$ and its normalized adjacency matrix $\An$, and let $\lambda_{1}, \dots, \lambda_{|\Xaug|}$ be the eigenvalues of the normalized Laplacian $I_{|\Xaug|} - \An$ in increasing order.
    The eigen-gap $\lambda_{\dd+1} - \lambda_{\dd}$ can be upper bounded as follows:
    \begin{align}
        \lambda_{\dd+1} - \lambda_{\dd} \le \lambda_{\dd+1} \le \frac{2\bar{\rho} ~\BE(\gA)}{1 - \nicefrac{\dd}{|\Xorig|}}
    \end{align}
    where $\bar{\rho} = \frac{\gDbar_{\max}}{\gDbar_{\min}}$ is the ratio of max and min probabilities over inputs.
\end{lemma}
\begin{proof}
        Let $\gamma_{\dd}$ be the $\dd^{th}$ largest eigenvalue of the normalized adjacency matrix $\An = D^{-\half} \A D^{-\half} \in \R^{\Xaug \times \Xaug}$, where $A[x,x']$ is the joint probability of augmentations $x$ and $x'$ appearing as two augmentations of the same input.
        Then we know that the $i^{th}$ smallest eigenvalue of $I_{|\Xaug|} - \An$ is $\lambda_{i} = 1 - \gamma_{i}$.
        Furthermore we note that $\An \in \R^{\Xaug \times \Xaug}$ has rank at most $|\Xorig|$, since from \Cref{table:notation} we know that $\An = \Abarn^{\top} \Abarn$, where $\Abarn \in \R^{\Xorig \times \Xaug}$ is the normalized of the input-augmentation distribution (refer \Cref{table:notation}) that has entries $\Abarn[\bar{x}, x] = \frac{\Dsim(x,  \bar{x})}{\sqrt{\Daug(x)}\sqrt{\Dorig(\bar{x})}}$.
        Thus we can conclude that $\gamma_{i} = 0$ for $|\Xorig| < i \le |\Xaug|$.
        First we prove the statement $\An = \Abarn^{\top} \Abarn$ below
        \begin{align*}
            (\Abarn^{\top} \Abarn)[x, x']
            &= \sum_{\bar{x}} \Abarn[\bar{x}, x] \Abarn[\bar{x}, x']
            = \sum_{\bar{x}} \frac{\gA(x,  \bar{x})}{\sqrt{\Daug(x)}\sqrt{\Dorig(\bar{x})}} \frac{\gA(x',  \bar{x})}{\sqrt{\Daug(x')}\sqrt{\Dorig(\bar{x})}}\\
            &= \frac{1}{\sqrt{\Daug(x) \Daug(x')}} \sum_{\bar{x}} \frac{\gA(x, \bar{x}) \gA(x', \bar{x})}{\Dorig(\bar{x})}
            = \frac{1}{\sqrt{\Daug(x) \Daug(x')}} \sum_{\bar{x}} \Dorig(\bar{x}) \gA(x \mid \bar{x}) \gA(x' \mid \bar{x})\\
            &= \frac{\A[x, x']}{\sqrt{\Daug(x) \Daug(x')}}
            = \An[x, x']
        \end{align*}
        We now connect $\BE(\gA)$ to the normalized augmentation matrix $\An$ by using \Cref{lem:bayes_error}.
        \begin{align*}
            \left(1 - \BE(\gA)\right)^{2}
            &=^{(a)} \left(\ex_{x \sim \gD} \left[\|\gA(\cdot \mid x)\|_{\infty}\right]\right)^{2}
            \le^{(b)} \left(\ex_{x \sim \gD} \left[\|\gA(\cdot \mid x)\|_{2}\right]\right)^{2}\\
            &\le^{(c)} \ex_{x \sim \gD} \left[\|\gA(\cdot \mid x)\|^{2}_{2}\right]
            = \sum_{x \in \Xaug} \Daug(x) \sum_{\bar{x} \in \Xorig} \gA(\bar{x} \mid x)^{2}\\
            &= \sum_{\bar{x} \in \Xorig} \Dorig(\bar{x}) \sum_{x \in \Xaug} \gA(x \mid \bar{x}) \gA(\bar{x} \mid x)\\
            &= \sum_{\bar{x}} \Dorig(\bar{x}) \sum_{x} \frac{\gA(x, \bar{x})}{\Dorig(\bar{x})} \frac{\gA(x, \bar{x})}{\Daug(x)}\\
            &= \sum_{\bar{x}} \Dorig(\bar{x}) \sum_{x} \frac{\gA(x, \bar{x})}{\sqrt{\Dorig(\bar{x}) \Daug(x)}} \frac{\gA(x, \bar{x})}{\sqrt{\Dorig(\bar{x}) \Daug(x)}}\\
            &= \sum_{\bar{x}} \Dorig(\bar{x}) \sum_{x} \Abarn[\bar{x}, x] \Abarn[\bar{x}, x]
            = \sum_{\bar{x}} \Dorig(\bar{x}) (\Abarn \Abarn^{\top})[\bar{x}, \bar{x}]\\
            &= \tr\left(\Dbar \Abarn \Abarn^{\top}\right)
        \end{align*}
        where $(a)$ follows from \Cref{lem:bayes_error}, $(b)$ follows from $\|\cdot\|_{\infty} \le \|\cdot\|_{2}$, $(c)$ follows from Jensen's inequality since $h(x) = x^{2}$ is convex.
        This upper bound can be used to lower bound the Bayes error as follows:
        \begin{align}
            2~\BE(\gA)
            &\ge 1 - \left(1 - \BE(\gA)\right)^{2}\nonumber\\
            &\ge^{(a)} 1 - \tr\left(\Dbar \Abarn \Abarn^{\top}\right)
            =^{(b)} \tr\left(\Dbar\right) - \tr\left(\Dbar \Abarn \Abarn^{\top}\right)
            =^{(c)} \tr\left(\Dbar (I_{|\Xorig|} - \Abarn \Abarn^{\top})\right)\nonumber\\
            &\ge^{(d)} \|\Dbar^{-1}\|_{2}^{-1} ~\tr\left(I_{|\Xorig|} - \Abarn \Abarn^{\top}\right)
            = \gDbar_{\min} ~\tr\left(I_{|\Xorig|} - \Abarn \Abarn^{\top}\right)\label{eqn:bayes_error_first}
        \end{align}
        where $\gDbar_{\min} = \min_{\bar{x} \in \Xorig} \Dorig(x)$.
        In the above sequence, $(a)$ follows the preceeding calculation, $(b)$ follows from $\tr(\Dbar) = \sum_{\bar{x}} \Dbar(\bar{x}) = 1$ and $(c)$ follows from linearity of the trace operator.
        The penultimate step $(d)$ follows from the fact that $\tr(X Y) \le \|X\|_{2} ~\tr(Y)$ for symmetric psd matrices $X, Y$; a proof for this can be found in Lemma 18 from \citet{jin2017escape}.
        We now connect this quantity to the eigenvalues of $\An$ as follows:
        \begin{align}
            \tr\left(I_{|\Xorig|} - \Abarn \Abarn^{\top}\right)
            &= |\Xorig| - \tr\left(\Abarn \Abarn^{\top}\right)
            =^{(a)} |\Xorig| - \tr\left(\Abarn^{\top} \Abarn\right)
            = |\Xorig| - \tr\left(\An\right)\nonumber\\
            &=^{(b)} |\Xorig| - \sum_{i=1}^{|\Xaug|} \gamma_{i}
            =^{(c)} |\Xorig| - \sum_{i=1}^{|\Xorig|} \gamma_{i}
            = |\Xorig| - \sum_{i=1}^{|\Xorig|} (1 - \lambda_{i})
            =\sum_{i=1}^{|\Xorig|} \lambda_{i}\nonumber\\
            &\ge \sum_{i=\dd+1}^{|\Xorig|} \lambda_{i}
            \ge (|\Xorig| - \dd) \lambda_{\dd+1}\label{eqn:bayes_error_second}
        \end{align}
        where $(a)$ follows from $\tr(PQ) = \tr(QP)$, $(b)$ is true because $\gamma_{i}$'s are the eigenvalues of $\An$ and because trace of a symmetric matrix is the sum of its eigenvalues, $(c)$ follows because $\An$ is rank $|\Xorig|$ and so $\gamma_{i}=0$ for $i>|\Xorig|$.
        Combining \Cref{eqn:bayes_error_first,eqn:bayes_error_second}, we get $\BE(\gA) \ge \nicefrac{1}{2} \gDbar_{\min} (|\Xorig| - \dd) \lambda_{\dd+1}$.
        Note that
        \begin{align}
            \gDbar_{\min} = \frac{\gDbar_{\min}}{\gDbar_{\max}} \gDbar_{\max}
            \ge \frac{\gDbar_{\min}}{\gDbar_{\max}} |\Xorig|^{-1}
            = \bar{\rho}^{-1} |\Xorig|^{-1}\label{eqn:bayes_error_third}
        \end{align}
        Plugging this into the bound gives $\BE(\gA) \ge \frac{1}{2\bar{\rho}} \left(1 - \frac{\dd}{|\Xorig|}\right) \lambda_{\dd+1}$, giving us
        \begin{align*}
            \lambda_{\dd+1} \le \frac{2\bar{\rho}~\BE(\gA)}{1 - \nicefrac{\dd}{|\Xorig|}}
        \end{align*}
    \end{proof}


\section{Experiment details}
\label{sec:apx_exps}

In this section, we provide additional notes, tables, and figures on the experiments.

\subsection{Synthetic experiments: hypercube example}
\label{sec:apx_synthetic_exps}

Figure~\ref{fig:synth-large} shows the results from Section~\ref{sec:hypercube} in greater detail. This section completes the details omitted in the main paper.

\paragraph{Data and augmentations.} As outlined in Section~\ref{sec:hypercube}, the data are drawn uniformly from the hypercube in dimension $D = 50$. The downstream labels are determined by a randomly drawn linear classifier $w$, whose first $k = 10$ coefficients are drawn from $\mathcal{N}(0,1)$; the rest are 0. The training set (under which $\Lcont$ is minimized) is of size $50000$; the downstream accuracies under a linear classifier are evaluated on a holdout validation set of $12500$. The augmentations are selected by i.i.d. random scaling factors $\tau \sim \mathrm{Unif}([0,1])$ and scaling down the last 40 coordinates.

\paragraph{Training and evaluation.} The two-layer MLP models used a hidden layer width of $2D = 100$, and an output (i.e. representation) dimension of $20$. Adam was run with a learning rate of $10^{-3}$, and default parameters $\beta_1=0.9, \beta_2=0.99$. The weight decay parameter for Adam was 0.004, selected from $\{0.001, 0.002, \ldots, 0.007\}$ based on best transfer performance. SGD was run with learning rate $0.01$. Quantitative results are shown in Table~\ref{table:hypercube}; means and 95\% confidence intervals are computed from 10 random seeds. 500 epochs of pre-training were run, with batch size 512.

\begin{figure}
    \centering
    \includegraphics[width=\linewidth]{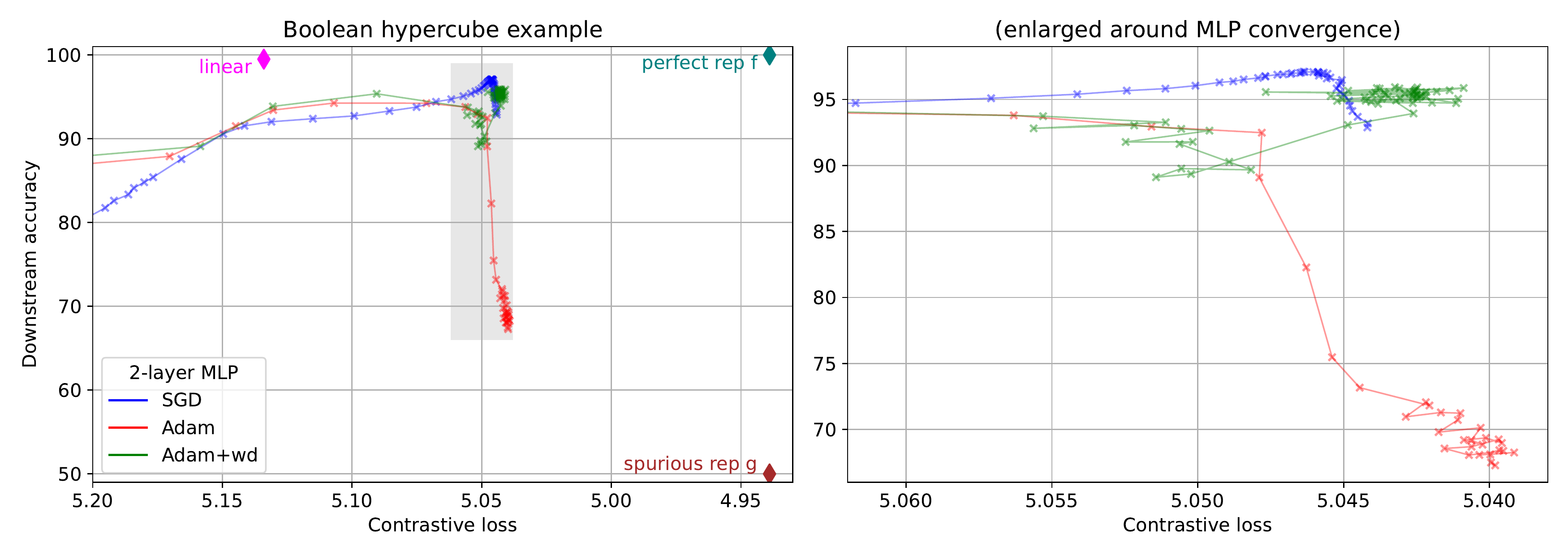}
    \caption{Full plots for the synthetic experiments, with all contrastive loss minimizers shown from various function classes (left) and enlarged plot near convergence of trajectories of solutions found by training 2-layed MLPs with various configurations of first-order optimizers (right).}
    \label{fig:synth-large}
\end{figure}

\subsection{CIFAR-10 + SimCLR experiments}
\label{sec:apx_vision_exps}


For all {\resnet} experiments, we use the {\resnete} architecture from PyTorch, with the standard modification for CIFAR-10 of replacing the first $7 \times 7$ convolution layer with a $3\times 3$ convolution and removing the maxpool layer.
We use the {\vit} implementation from \url{https://github.com/lucidrains/vit-pytorch} with patch size: 4, hidden dimension: 256, depth: 6 and number of heads: 8.
For {\mlpmixer} we use the implementation from \url{https://github.com/lucidrains/mlp-mixer-pytorch} with patch size: 4, hidden dimension: 256 and number of heads: 8.
In each model, the representation for contrastive learning is computed by adding an extra MLP (projection layer) on top of the base model, as proposed in \citet{chen2020simple}.
The projection layer has 1 hidden layer with 2048 dimensions, followed by a batch norm layer and ReLU non-linearity, and output dimensionality of 1024.

\begin{figure}
    \centering
    \includegraphics[width=\textwidth]{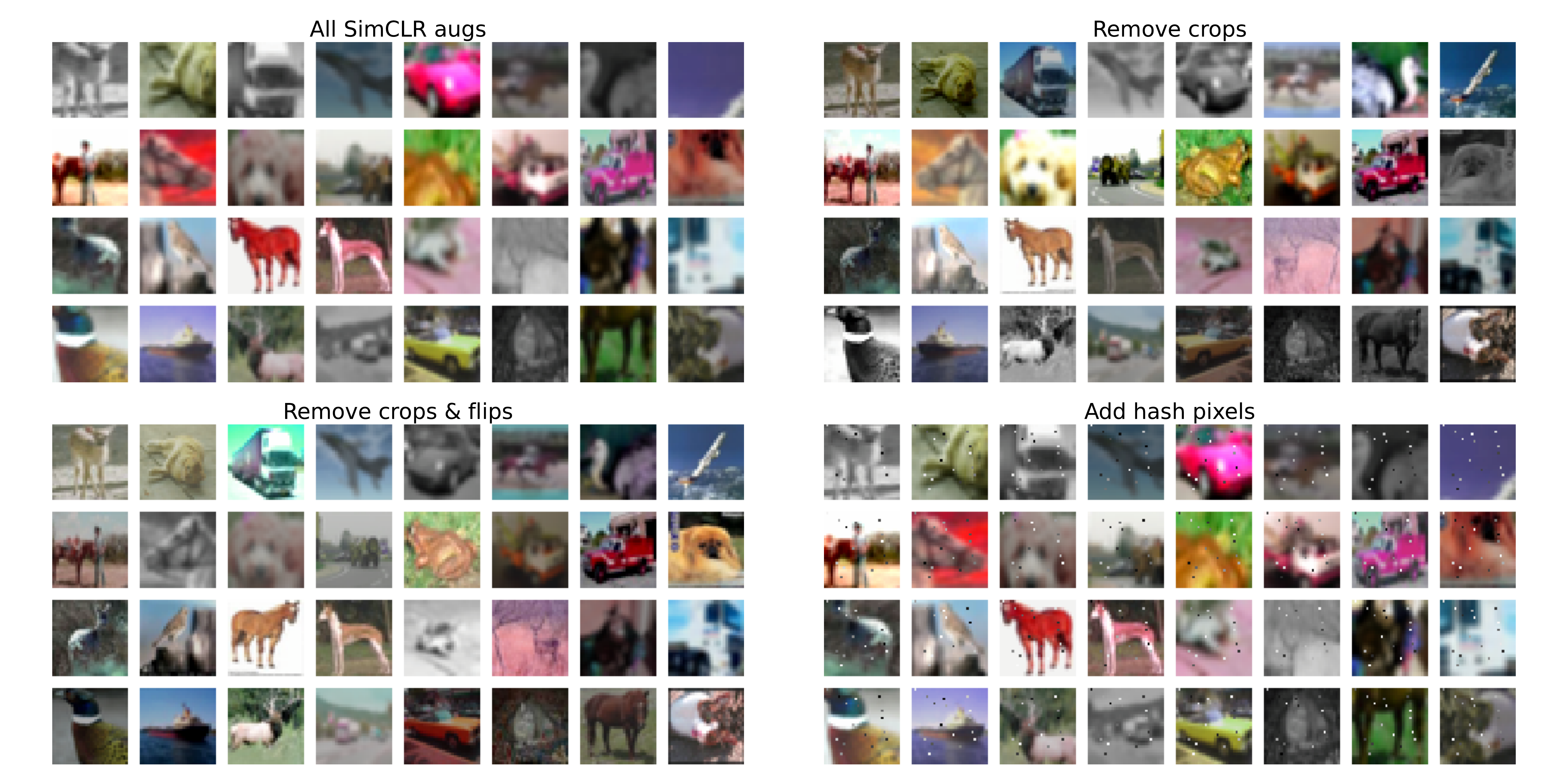}
    \caption{Examples of augmented images from CIFAR-10 used in the SimCLR experiments. {\bf TL:} Full pipeline of augmentations from SimCLR \citep{chen2020simple}. {\bf TR:} Remove random cropping. {\bf BL:} Remove random cropping and horizontal flip. {\bf BR:} Add ``hash pixels'' to each image, which uniquely identify the particular example.}
    \label{fig:aug-examples}
\end{figure}

\textbf{Augmentations.} The following augmentations are used, inspired by \citep{chen2020simple}:

{\tt
    transforms.Compose([
    \begin{quote}
    RandomResizedCrop(32, scale=(0.3, 1.0)),\\
    RandomHorizontalFlip(p=0.5),\\
    transforms.RandomApply([transforms.ColorJitter(0.4, 0.4, 0.4, 0.1)], p=0.8),\\
    RandomGrayscale(p=0.2),\\
    GaussianBlur(kernel\_size=3)
    \end{quote}
    ])
}

For experiments in \Cref{fig:vision_scatter_plot} we use the full pipeline of augmentations (top left) and sequentially remove random cropping (top right) and horizontal flipping (bottom left). Examples of augmented CIFAR-10 images are shown in Figure~\Cref{fig:aug-examples}

\paragraph{Contrastive training.} We train the model for 1000 epochs, by performing a pass over this training dataset and minimize the SimCLR contrastive learning loss.
We normalize the representations $f$ to unit norm when computing the SimCLR loss, as is common in many works.
Formally, given a batch $\{(x_i, x'_i)\}_{i=1}^B$ of pairs of augmentations we perform a single update of Adam to minimize the following loss:
\begin{equation}
\label{eqn:simclr_objective_vision}
    L(f) = -\frac{1}{2B} \sum_{i=1}^n \frac{w(x_i, x'_i)}{\sum_{j=1}^n w(x_i, x'_j) + \sum_{j=1, j\ne i}^n w(x_i, x_j)} -\frac{1}{2B} \sum_{i=1}^n \frac{w(x_i, x'_i)}{\sum_{j=1}^n w(x'_i, x_j) + \sum_{j=1, j\ne i}^n w(x'_i, x'_j)},
\end{equation}
where $w(x, x') = \exp\left(\frac{f(x)^\top f(x')}{\tau\|f(x)\|_2\|f(x')\|_2}\right)$ and we pick the temperature parameter as $\tau = 0.5$.
Training hyperparameters are presented in \Cref{table:hyperparameters_vision}.

\paragraph{Downstream evaluation.} The downstream evaluation is linear classification accuracy of the learned representation $f$ to predict the class for an image.
The linear classifier is trained for 1000 epochs using Adam; hyperparameter details are presented in \Cref{table:hyperparameters_vision}.

For the plots in \Cref{fig:vision_scatter_plot} we evaluate every the contrastive loss and downstream accuracy every 5 epochs of training and stop when the average test contrastive loss (window size of 5) is minimized.

\begin{table*}[!t]
    \centering
    \caption{Hyperparameter values for experiments on CIFAR-10 trained using {\resnete}.}
    \begin{tabular}{c|c|c}
         \hline
         \textbf{Hyperparameters} & \textbf{Values} \\
         \hline 
         \multirow{5}{*}{Contrastive training} & Max epoch & 1000\\
         & Learning rate & 0.001\\
         & Optimizer & Adam + weight decay (0.0005)\\
         & Batch size & 512\\
         & Representation dimension & 1024\\
         \hline
         \multirow{5}{*}{Downstream training} & Epochs & 1000\\
         & Learning rate (start) & 0.01\\
         & Optimizer & Adam + weight decay (0.000005)\\
         & Scheduler & ExponentialLR (gamma: $10^{0.004}$)\\
         & Batch size & 1000\\
         \hline
    \end{tabular}
    \label{table:hyperparameters_vision}
\end{table*}

\subsubsection{Hash experiment.}
\label{sec:hash_exp}

To enforce the disjoint augmentation regime, we select a set of 16 pixels, and modify an augmentation by replacing those 16 pixels (8 bits each) with an 128-bit MD5 hash of the image that generated the augmentation. 
This way the original image hash (and thus its identity) can be recovered from any of its augmentations.
The result of training on this small variation of the standard pipeline is presented in \Cref{fig:vision_scatter_plot} (bottom right). Some examples of this augmentation is shown in \Cref{fig:aug-examples} (bottom right).

\subsubsection{Label-orthogonal training.}
\label{sec:label_orth}

We largely follow the same procedure as standard training, but modify the representation $f(x)$ for an augmentation before passing it to the contrastive loss.
In particular for an augmentation and label pair $(x,y)$, compute the representation $f(x)$ as usual. Before passing it into the contrastive loss, convert apply the transformation $f'(x) = f(x) - \mu_{y}$, where $\mu_{y}$ is the mean representation for augmentations from class $y$.
$\mu_{y}$ is computed at every step, using augmentations from a memory bank of $10240$ pairs of (x,y) collected over training. 
Then $f'(x)$ is passed into the SimCLR loss in \Cref{eqn:simclr_objective_vision} instead of $f(x)$ and everything else remains the same.
The subtraction of the mean $\mu_{y}$ from the representation make the representation orthogonal to the labels, thus declining its ability to linearly classify images.
The result of training with this procedure is presented in \Cref{fig:vision_scatter_plot} (top left).
Note that the implicit assumption in the calculation of the contrastive loss is that different classes do not share any augmentations, i.e. the labels are almost invariant to standard augmentations.



\subsection{Experiments on Text Domain}
\label{sec:apx_text_exps}

\paragraph{Experimental Setup.} We evaluate on the AG News classification dataset~\cite{zhang2015character}.
This dataset contains 4 classes (``World'', ``Sports'', ``Business'', ``Sci/Tech'') and each class contains news articles from that topic.
We use the tokenizer from torchtext library.
If a token sequence is of length more than 60, we then trim it to its first 60 tokens, leading to a vocabulary size of 11970.

We perform contrastive learning similar to SimCLR. We train the model in epochs and in each epoch we sample pairs of augmentation for 50,000 randomly chosen pieces of text in the training dataset.
We then perform a single pass over this dataset and minimize the SimCLR contrastive learning loss from \Cref{eqn:simclr_objective_vision}, with temperature $\tau=1$.
The downstream evaluation task is to simply predict the class given the text.

At the start of contrastive learning, we create a held-out validation set of pairs of augmentation sampled for 10,000 randomly chosen examples from the original validation set. At the end of each epoch of contrastive learning, we evaluate the model on this held-out validation set by computing the SimCLR loss.
We also train a linear classifier on top of fixed model representations, to evaluate the model on the downstream classification task.
During the downstream training, we evaluate the model at the end of epoch on the validation set and report the linear classifier with the best validation loss.
We stop training if the best validation loss does not improve for $\kappa$ consecutive epochs where $\kappa$ is the patience hyperparameter, or if we hit a maximum number of epochs. 
Hyperparameter values are listed in Table~\ref{table:hyperparameters_nlp}.

\begin{table*}[!t]
    \centering
    \caption{Hyperparameter values for experiments on AG News. Unless specified, the same hyperparameter value is used for both contrastive learning and the downstream classification task.}
    \label{table:hyperparameters_nlp}
    \begin{tabular}{c|c}
         \hline
         \textbf{Hyperparameters} & \textbf{Values} \\
         \hline 
         Max epoch & 100\\
         Learning rate for contrastive learning & 0.01 for {\bow}, 0.001 otherwise\\
        Learning rate for downstream linear classification & 0.01\\
         Patience & 10 \\
         Batch size & 128\\
         Representation dimension & 768\\
         Gradient clipping norm & 2.5\\
         \hline
    \end{tabular}
\end{table*}

\ifthenelse{\boolean{arxiv}}{
\begin{figure}[ht!]
    \centering
        \includegraphics[width=.65\linewidth]{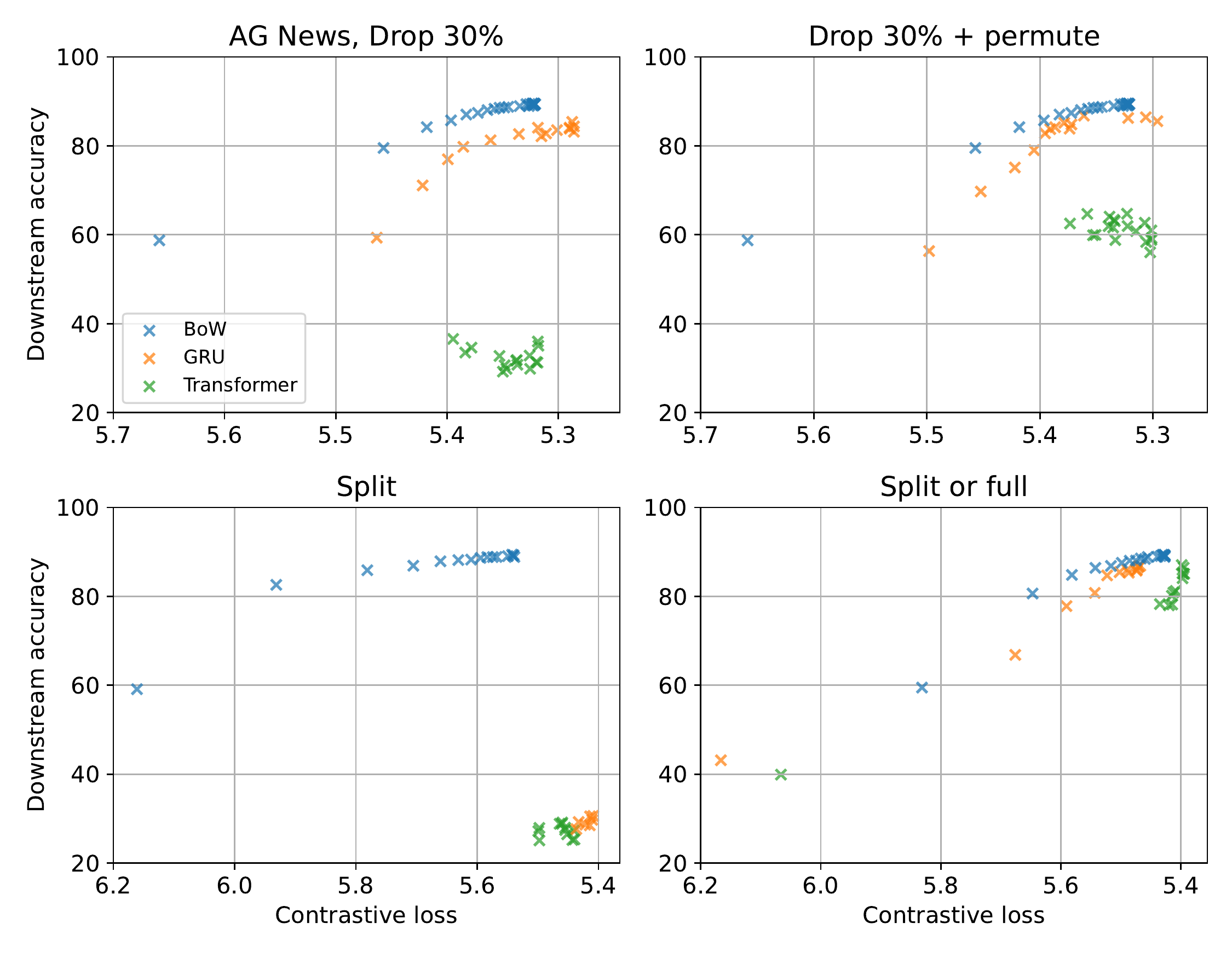}
        \vspace{-0.1in}
    \caption{Contrastive loss $\rightarrow$ accuracy transfer plots for AG News with bag-of-words ({\bow}), {\gru} and {\transformers} architectures with representation dimensionality $\dd=128$. These plots use the average representation of augmentations $\aug{f}$ for downstream evaluation rather than the representation $f$ directly. Augmentations in each case are as follows: {\bf TL:} Drop random 30\% of tokens. {\bf TR:} Drop random 30\% of tokens and randomly permute the rest. {\bf BL:} Either the first half or second half of the input. {\bf BR:} Either the first half, second half or the full input. The plots here are almost identical to the plots from \Cref{fig:agnews_scatter_plot}, suggesting that the distribution shift from augmentations to unaugmented inputs from contrastive learning to downstream evaluation does not play a big role.
    }
    \label{fig:agnews_scatter_robust_plot}
\end{figure}
}{
}

\paragraph{Model Details.} We evaluate three models on the AG News task. All models encode a given text to a $\dd$-dimensional representation.
The first model is a bag of word ($\bow$) that trains a word embedding matrix and simply returns the average word embedding of tokens in the text.
The second model is Gated Recurrent Unit ($\gru$), which is a recurrent neural network~\cite{chung2014empirical}.
The {\gru} is uni-directional, uses a 300 dimensional input word embedding, dropout of 0.3, hidden dimension of 768, has 4 layers, and linearly maps the hidden state representation of the final token from the layer to $\dd$-dimensions.
The final model is a $\transformers$~\citep{vaswani2017attention}, which is the base model for many state-of-the-art neural networks in NLP.
The {\transformers} is uni-directional, hidden dimension of 128, has 4 layers and 4 attention heads, and linearly maps the hidden state representation of the final token from the layer to $\dd$-dimensions.

\subsubsection{Robust evaluation.}
\label{sec:robust_eval}
Standard practice is to train a representation $f$ on augmentations $x$, and use the same function to compute representations for unaugmented inputs $\bar{x}$.
This is the strategy we employ for the plots in \Cref{fig:agnews_scatter_plot}.
However, as discussed in \Cref{sec:preliminaries}, this causes an obvious distribution shift, since the representations have been trained to output something meaningful for unaugmented inputs.
This could be a potential reason for the brittle transfer performance of {\gru} and {\transformers}.
However we verify that this distribution shift is not the reason, by instead evaluating downstream performance using the augmentation-averaged representation $\aug{f}$, as defined in \Cref{eqn:aug_mean_rep}.
These robust evaluation transfer plots are presented in \Cref{fig:agnews_scatter_robust_plot}, which look almost identical to those in \Cref{fig:agnews_scatter_plot}.

\subsubsection{Visualizing 2-dimensional representations.}
\label{sec:visualize_reps}

We train contrastive learning models with output dimensionality $\dd=2$ and visualize the contrastive loss $\rightarrow$ accuracy in \Cref{fig:agnews_scatter_plot_d2}.
Firstly we note that the trends are not exactly the same as in \Cref{fig:agnews_scatter_plot} that plot the same for $\dd=128$.
Most interestingly, for the split augmentation, {\gru} does not perform well on downstream accuracy for $\dd=128$, but it does almost as well as {\bow} at $\dd=2$.
This kind of non-monotonic behavior w.r.t. representation dimensionality $\dd$ is also unexplained by existing theory.

Next we visualize the learned representations for augmentations from different classes (normalized to unit norm) in  \Cref{fig:agnews_visualize_reps_drop}, for the drop augmentation.
We sample 100 inputs per class and 4 augmentations per input, and encode them with the trained {\bow}, {gru} and {\transformers} models.
For clear visualization, we plot the 4 augmentations per image with the same color, {\bf with each of them plotted at different radii} (1.0, 1.133, 1.267, 1.4).
We observe that the {\bow} representations look roughly linearly separable since different classes tend to roughly occupy different quadrants of the circle, corroborating its good downstream performance from \Cref{fig:agnews_scatter_plot_d2}.
It does so by roughly bringing augmentations of the same input (points with the same color) closer to each, although the representations not perfectly augmentation invariant.
The {\gru} representations in every class, on the other hand, are spread out and brings augmentations very closer to each other than {\bow} representations, reminiscent of the uniformity and alignment properties from \citet{wang2020understanding}.
However these representations are not linearly separable.
The {\transformers} representations are intriguing since they are not uniformly spread out, but almost perfectly augmentation invariant.
Furthermore the representation distributions for different classes are identical to each other, justifying its bad downstream performance from \Cref{fig:agnews_scatter_plot_d2}.

\ifthenelse{\boolean{arxiv}}{
\begin{figure}[t!]
    \centering
        \includegraphics[width=.65\linewidth]{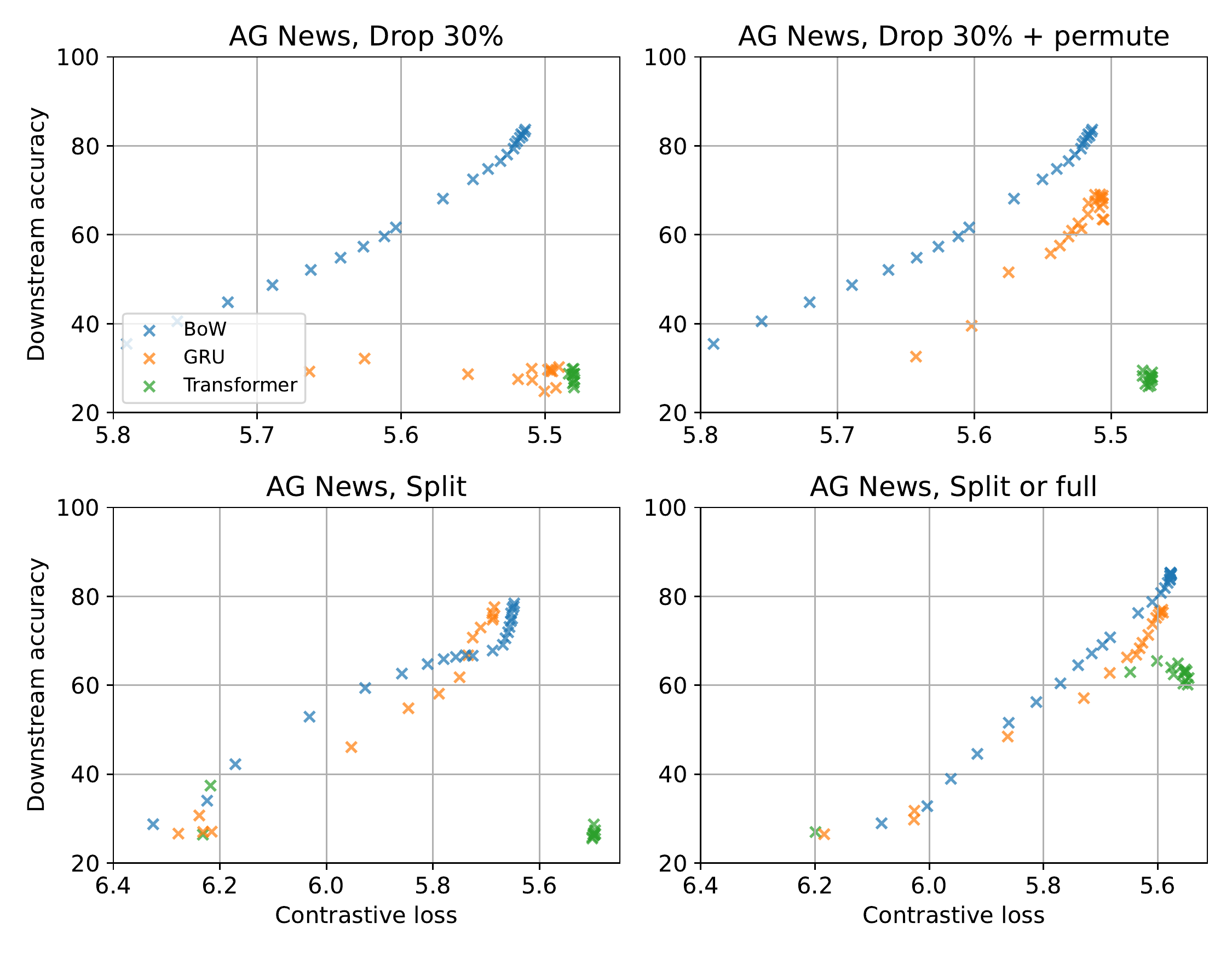}
        \vspace{-0.1in}
    \caption{Contrastive loss $\rightarrow$ accuracy transfer plots for AG News with bag-of-words ({\bow}), {\gru} and {\transformers} architectures with representation dimensionality $\dd=2$. Augmentations in each case are as follows: {\bf TL:} Drop random 30\% of tokens. {\bf TR:} Drop random 30\% of tokens and randomly permute the rest. {\bf BL:} Either the first half or second half of the input. {\bf BR:} Either the first half, second half or the full input. In all cases {\bow} representation does quite well downstream ($\sim 80$\%), but either {\transformers} or both {\gru} and {\transformers} demonstrate brittleness of transfer for different augmentations.
    }
    \label{fig:agnews_scatter_plot_d2}
\end{figure}
}{
}

\ifthenelse{\boolean{arxiv}}{
\begin{figure}[h!]
    \centering
        \includegraphics[width=.75\linewidth]{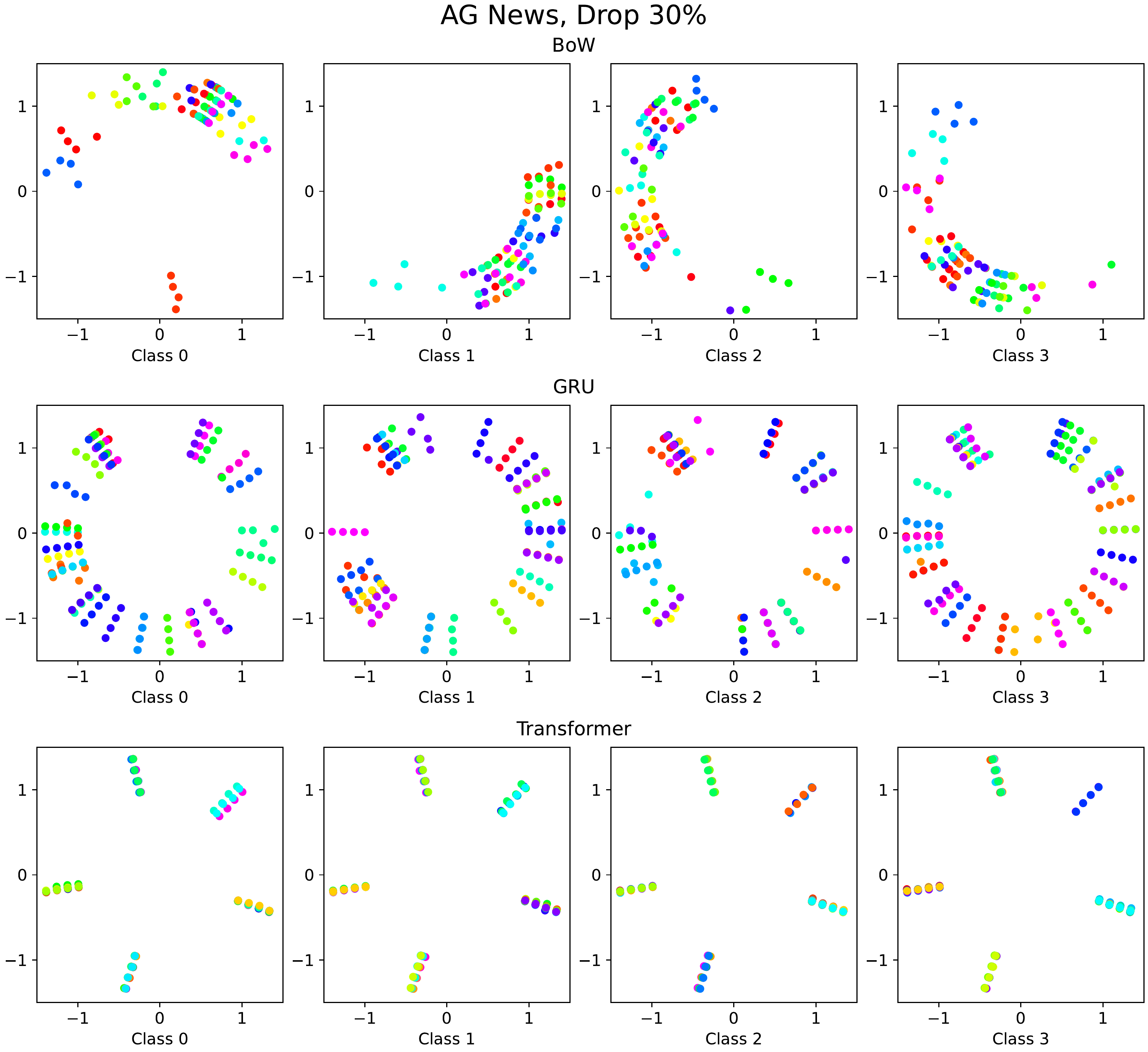}
        \vspace{-0.1in}
    \caption{
    We plot representations of augmentations from different classes, for {\bow}, {\gru} and {\transformers} respectively. The 30\% drop augmentation is used for these plots. While all representations are supposed to be normalized to unit norm, for clear visualization, we plot the 4 augmentations per image with the same color, {\bf with each of them plotted at different radii} (1.0, 1.133, 1.267, 1.4). We observe that {\gru} and {\transformers} are quite augmentation invariant, but are not linearly separable. See \Cref{sec:visualize_reps} for more discussion about this.
    }
    \label{fig:agnews_visualize_reps_drop}
\end{figure}
}{
}

\ifthenelse{\boolean{arxiv}}{
\begin{figure}[h!]
    \centering
        \includegraphics[width=.75\linewidth]{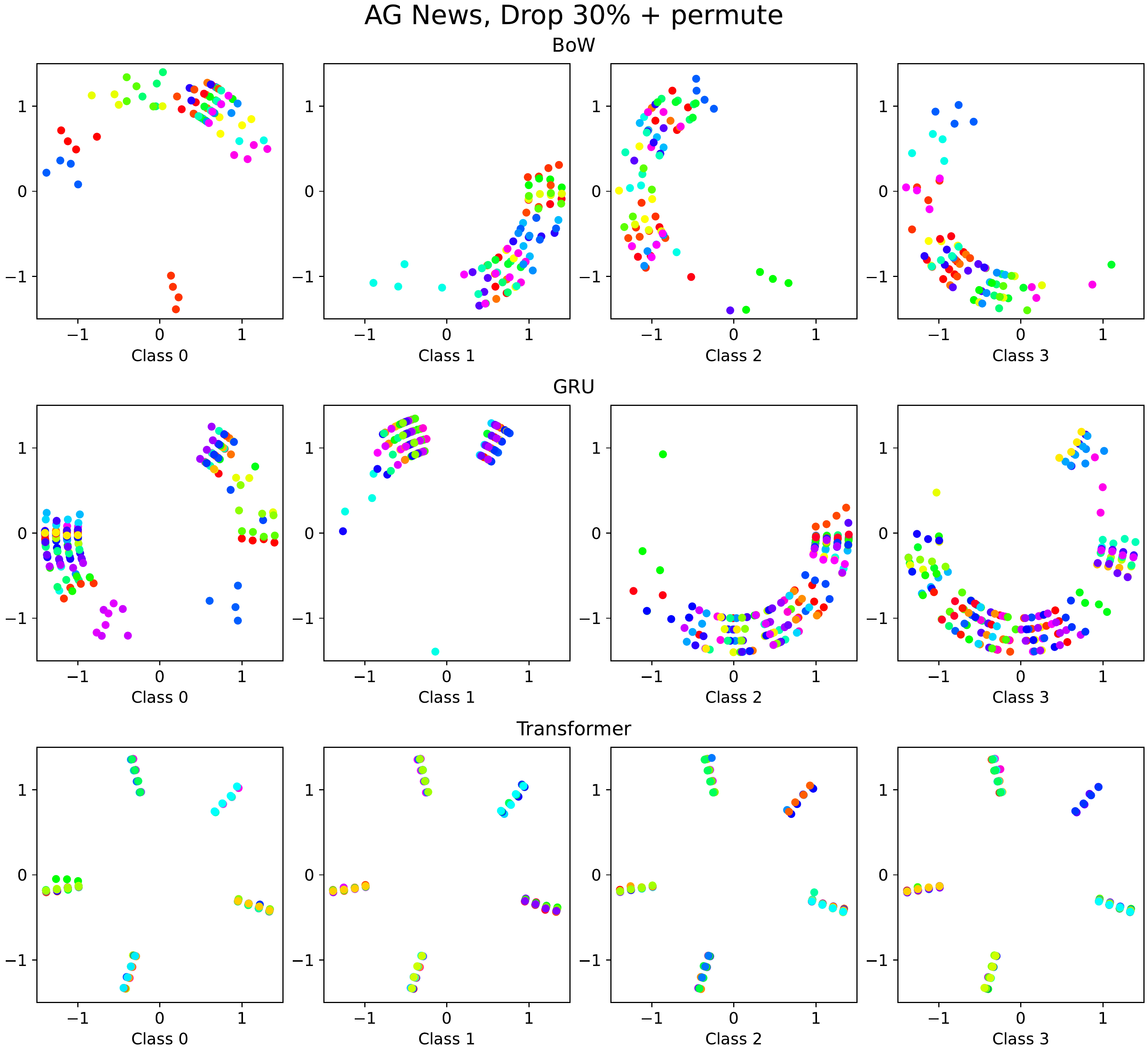}
        \vspace{-0.1in}
    \caption{
    We plot representations of augmentations from different classes, for {\bow}, {\gru} and {\transformers} respectively. The 30\% drop + permute augmentation is used for these plots. While all representations are supposed to be normalized to unit norm, for clear visualization, we plot the 4 augmentations per image with the same color, {\bf with each of them plotted at different radii} (1.0, 1.133, 1.267, 1.4). We observe that {\bow} representations are roughly linearly classifiable, {\gru} representations are somewhat classifiable while {\transformers} are quite augmentation invariant, but not linearly separable.
    }
    \label{fig:agnews_visualize_reps_drop_permute}
\end{figure}
}{
}

\end{document}